\documentclass[accepted]{uai2025} 
                        
\usepackage[american]{babel}

\usepackage{hyperref}
\usepackage{url}
\usepackage{microtype}
\usepackage{graphicx}
\usepackage{booktabs} 
\usepackage{siunitx} 
\usepackage{caption} 
\usepackage{multirow}

\usepackage{enumitem}
\usepackage{bm}
\usepackage{bbm}
\usepackage{comment}
\usepackage{amsmath}
\usepackage{amssymb}
\usepackage{mathtools}
\usepackage{amsthm}
\usepackage[capitalize,noabbrev]{cleveref}

\usepackage[utf8]{inputenc} 
\usepackage[T1]{fontenc}    
\usepackage{hyperref}       
\usepackage{url}            
\usepackage{booktabs}       
\usepackage{amsfonts}       
\usepackage{nicefrac}       
\usepackage{microtype}      
\usepackage{xcolor}         
\usepackage{natbib}

\usepackage{subcaption}
\usepackage{caption}

\usepackage{booktabs} 
\usepackage{siunitx} 
\usepackage{multirow}

\usepackage{arydshln}

\usepackage{enumitem}
\usepackage{bm}
\usepackage{bbm}
\usepackage{comment}
\usepackage{amsmath}
\usepackage{amssymb}
\usepackage{mathtools}
\usepackage{amsthm}
\usepackage[capitalize,noabbrev]{cleveref}

\usepackage{algorithm}
\usepackage{algorithmic}
\usepackage{amsmath}
\usepackage{amssymb}
\usepackage{mathtools}
\usepackage{amsthm}
\allowdisplaybreaks[4]

\usepackage{bbm} 
\usepackage{tabularx} 
\usepackage{multirow} 
\usepackage{bbding}
\usepackage{wrapfig}

\theoremstyle{plain}
\newtheorem{theorem}{Theorem}

\newtheorem{lemma}[theorem]{Lemma}

\newtheorem{assumption}{Assumption}
\newtheorem{definition}{Definition}

\theoremstyle{definition}
\newtheorem{remark}{Remark}

\usepackage{natbib} 
\usepackage{mathtools} 
\usepackage{booktabs} 
\usepackage{tikz} 

\title{Nearly Optimal Differentially Private ReLU Regression}

\author[2]{Meng Ding}
\author[2]{Mingxi Lei}
\author[3]{Shaowei Wang}
\author[4]{Tianhang Zheng}
\author[5]{Di Wang\thanks{Correspondence to: Di Wang <di.wang@kaust.edu.sa>, Jinhui Xu <jhxu@ustc.edu.cn>}}
\author[1]{Jinhui Xu$^*$}
\affil[1]{%
    School of Information Science and Technology, University of Science and Technology of China
  }
\affil[2]{%
    Department of Computer Science and Engineering, State University of New York at Buffalo
}
\affil[3]{%
    Institute of Artificial Intelligence and Blockchain, Guangzhou University
}
\affil[4]{%
    The State Key Laboratory of Blockchain and Data Security, Zhejiang University
  }
\affil[5]{%
    Division of CEMSE, King Abdullah University of Science and Technology 
  }

\begin{document}
\maketitle

\begin{abstract}
In this paper, we investigate one of the most fundamental non-convex learning problems—ReLU regression—in the Differential Privacy (DP) model. Previous studies on private ReLU regression heavily rely on stringent assumptions, such as constant-bounded norms for feature vectors and labels. We relax these assumptions to a more standard setting, where data can be i.i.d. sampled from $O(1)$-sub-Gaussian distributions. We first show that when $\varepsilon = \tilde{O}(\sqrt{\frac{1}{N}})$ and there is some public data, it is possible to achieve an upper bound of  $\Tilde{O}(\frac{d^2}{N^2 \varepsilon^2})$ for the excess population risk in $(\varepsilon, \delta)$-DP, where $d$ is the dimension and $N$ is the number of data samples. Moreover, we relax the requirement of $\varepsilon$ and public data by proposing and analyzing a one-pass mini-batch Generalized Linear Model Perceptron algorithm (DP-MBGLMtron). Additionally, using the tracing attack argument technique, we demonstrate that the minimax rate of the estimation error for $(\varepsilon, \delta)$-DP algorithms is lower bounded by $\Omega(\frac{d^2}{N^2 \varepsilon^2})$. This shows that DP-MBGLMtron achieves the optimal utility bound up to logarithmic factors. Experiments further support our theoretical results. 
\end{abstract}

\section{Introduction}
Privacy preservation has become a critical consideration, posing a significant challenge for machine learning models that process sensitive data. To address this issue, Differential Privacy (DP) \citep{dwork2006calibrating} has emerged as a widely used approach, providing verifiable protection against identification and resistance to any auxiliary information that attackers might have.

Stochastic Optimization (SO) and its empirical counterpart, Empirical Risk Minimization (ERM), represent some of the most fundamental challenges in machine learning and statistics, which are especially susceptible to privacy leaks when involved with sensitive data. Therefore, significant efforts have been made to develop differentially private algorithms tailored to these challenges, specifically referred to as DP-SO and DP-ERM. Although there is an extensive body of research on DP-SO and DP-ERM \citep{bassily2014private, wang2017differentially, wang2023generalized, feldman2020private, song2020characterizing, su2021faster, asi2021private, bassily2021non, kulkarni2021private,hu2022high,zhang2025improved,su2024faster,su2023differentially}, the majority of existing studies primarily focus on convex loss functions. This focus inadvertently neglects the crucial role of nonconvex optimization, which is essential for the development of advanced machine learning models. Recent progress has introduced algorithms for DP nonconvex optimization \citep{zhang2017efficient,wang2017differentially,wang2019differentially1,wang2019differentially,zhang2021private, bassily2021differentially,wang2023efficient,wang2024gradient}. However, unlike the convex loss function, DP-SO with non-convex loss is still far from well-understood due to its intrinsic difficulties (see Section \ref{sec:pre} for details). 

ReLU regression, a fundamental non-convex model, is widely recognized for its effectiveness in deep learning applications and serves as a foundational step toward understanding multi-layer neural networks~\citep{du2018gradient}. Despite the extensive studies in the non-private setting that have been conducted, the theoretical exploration of ReLU regression in the DP model remains relatively limited. Particularly, in DP ReLU regression, we have an $N$-size dataset $D = \{ (\mathbf{x}_i,y_i ) \}_{i=0}^{N-1}$, where each data point consisting of a feature vector $\mathbf{x}_i \in \mathcal{X}\subseteq \mathbb{R}^d$ and a response variable $y_i \in \mathcal{Y}$ is i.i.d. sampled from a ReLU regression model. Specifically, each pair of $(\mathbf{x}_i, y_i)$ is a realization of the ReLU regression model 
\begin{equation}\label{eq:model}
    y  = \operatorname{ReLU} (\mathbf{x}^{\top} \mathbf{w}_* ) + z,
\end{equation}
where $\operatorname{ReLU}(\cdot):=$ $\max \{\cdot, 0\}$; $z$ is a zero mean randomized noise; $\mathbf{w}_*\in \mathbb{R}^d$ is the optimal model parameter. The objective of the problem is to develop a DP model $\mathbf{w}_{\text {priv }}$ that minimizes the excess population risk, defined as $\mathcal{L} (\mathbf{w}_{\text{priv}} )-\mathcal{L} (\mathbf{w}_* )$, where the risk function $\mathcal{L} (\mathbf{w} )$ is given by: 

\begin{equation}\label{eq:relu_risk}
    \mathcal{L}(\mathbf{w})=\frac{1}{2} \mathbb{E}_{(\mathbf{x}, y) \sim \mathcal{D}} [ (\operatorname{ReLU} (\mathbf{x}^{\top} \mathbf{w} )-y )^2 ].
\end{equation}

Recently, \cite{shen2023differentially} explored DP ReLU regression in both well-specified and misspecified settings, yet the problem remains largely unexplored, with numerous challenges yet to be addressed. Specifically, their methods rely on stringent assumptions, including bounded norms for feature vectors and labels, with $\|\mathbf{x}\|_2 \leq O(1)$ and $\|y\| \leq O(1)$—assumptions that do not hold even for typical Gaussian distributions. Even when $\|\mathbf{x}\|_2 \leq O(\sqrt{d})$ such as Bernoulli or uniform distributions, the bound in \cite{shen2023differentially} is only sub-optimal (see Remark \ref{remark:4} and Theorem~\ref{thm:lower_bound} for details).  Moreover, their proposed differentially private projected gradient descent (DP-PGD) requires at least $O(N^2)$ gradient computations, rendering it inefficient. 

In this paper, we revisit the problem of DP ReLU regression and offer (nearly) optimal guarantees for excess population risk under more standard assumptions where the data can be i.i.d. sampled from $O(1)$-sub-Gaussian distributions. 
Our contributions can be summarized as follows:

\noindent {\bf 1)} We provide the analysis on the Differentially Private Generalized Linear Model Perceptron algorithm (DP-GLMtron), which utilizes a one-pass training strategy where data points are permuted and sampled without replacement. To make the gradient norm bounded, instead of using a fixed clipping threshold, we incorporate adaptive clipping by estimating from additional public data points. This allows the noise to be set adaptively based on the excess error in each iteration. We demonstrate that our $(\varepsilon, \delta)$-DP method can achieve an excess population risk upper bound of $\Tilde{O}(\frac{d^2}{N^2 \varepsilon^2})$.

\noindent {\bf 2)} Key concerns with the analysis of DP-GLMtron include that its upper bound only holds with a small privacy budget $\varepsilon = {O}(\sqrt{\frac{\log (N / \delta)}{N}})$ and its reliance on additional public data for the adaptive clipping mechanism. To address these limitations, we modify DP-GLMtron to introduce a new method—DP-MBGLMtron (DP-Mini-Batch Generalized Linear Model Perceptron)—which divides the data into mini-batches and performs one pass of the mini-batch GLMtron. We show that DP-MBGLMtron can achieve the same excess population risk upper bound as DP-GLMtron, even with larger privacy budgets and without available public data. 

\noindent {\bf 3)} To illustrate the tightness of our analysis, we derive a lower bound of the estimation error for any $(\varepsilon, \delta)$-DP algorithms. Specifically, our analysis uses a tracing attack argument, illustrating that estimators with overly precise estimates would compromise privacy guarantees. According to this property, we can establish that any such algorithm must incur an excess population risk of  $\Omega(\frac{ d^2}{N^2\varepsilon^2})$, indicating that the upper bound is optimal up to logarithmic factors.

\section{Related Work}\label{sec:pre}
{\bf Private Convex Optimization.} Differentially private convex optimization has been extensively studied over the past decade \cite{chaudhuri2011differentially,jain2012differentially,kifer2012private,bassily2014private,jain2014near,wang2017differentially,feldman2020private}. Existing approaches in this field can broadly be categorized into three main categories: output perturbation, objective perturbation, and gradient perturbation. Output perturbation ensures differential privacy by adding calibrated noise to the final model parameters \cite{dwork2006calibrating,chaudhuri2011differentially,kifer2012private,zhang2017efficient,wu2017bolt}; Objective perturbation modifies the optimization objective itself by injecting noise into the loss function before solving the problem, thereby inherently privatizing the optimization process \cite{chaudhuri2011differentially,kifer2012private,talwar2014private,iyengar2019towards}; Gradient perturbation privatizes iterative optimization algorithms (e.g., stochastic gradient descent) by perturbing the gradient updates at each iteration \cite{bassily2014private,wang2017differentially,jayaraman2018distributed,wang2019differentially,bassily2019private}. All of these approaches have been demonstrated to achieve the asymptotically optimal bound $\tilde{O}(\frac{\sqrt{d}}{\varepsilon N})$ for smooth convex loss.

{\bf Private Nonconvex Optimization.}
In the domain of DP-SO and DP-ERM with convex loss functions, excess population risk has traditionally been the main metric for utility evaluation. However, in non-convex settings, utility assessment methods generally fall into three categories: first-order stationarity-based, second-order stationarity-based, and direct measurement of excess population risk.
First-order stationarity-based methods \citep{wang2019differentially, zhou2020bypassing, song2021evading, bassily2021differentially, zhang2021private, xiao2023theory,wang2023finite,tao2025second} evaluate utility by analyzing the $\ell_2$-norm of the gradient of the population risk function. While widely adopted, these methods face notable challenges. For example, \citet{agarwal2017finding} showed that as the sample size increases indefinitely, the gradient norm approaches zero. However, a vanishing gradient does not necessarily indicate that a differentially private estimator converges to, or is near, a local minimum.
Second-order stationarity-based methods \citep{wang2019differentially, wang2021escaping} assess both the gradient norm and the minimal eigenvalue of the Hessian matrix of the population risk function. These approaches work well in specific settings where any second-order stationary point is a local minimum, and all local minima are global minima, such as in problems like matrix completion and dictionary learning.
The third category directly uses excess population risk to evaluate utility \citep{shen2023differentially, wang2019differentially}, which aligns with the focus of our work.

One of the concurrent works, \cite{dingrevisiting}, addresses the same private ReLU regression problem with similar assumptions, but our work differs significantly in several key aspects, including threshold estimation, privacy amplification techniques, theoretical bounds, and data assumptions. Specifically, \cite{dingrevisiting} uses a threshold estimation method based on \cite{liu2023near} and a tree aggregation mechanism for privacy amplification, whereas we leverage statistical properties and minibatch sampling. Additionally, the theoretical results in \cite{dingrevisiting} include an upper bound with a term $\Gamma$ dependent on unknown intermediate parameters $\mathbf{w}_t$, while our results depend only on the problem parameters $d, n$, and $\varepsilon$, making them more natural. Furthermore, the lower bound in \cite{dingrevisiting} is algorithm-specific and relies on intermediate models, whereas our lower bound is general, depending solely on $d, n$, and $\varepsilon$, and achieves nearly optimal rates. Finally, \cite{dingrevisiting} assumes the eigenvalue decomposition of the data covariance matrix is well-defined, which our approach does not require.

\section{Preliminaries}\label{sec:preliminaries}
\textbf{Notations}: We use boldface lower letters such as $\mathbf{x}, \mathbf{w}$ for vectors and boldface capital letters (e.g., $\mathbf{A}, \mathbf{H}$) for matrices.
Let $\|\mathbf{A}\|_2$ denote the spectral norm of $\mathbf{A}$.
For two matrices $\mathbf{A}$ and $\mathbf{B}$ of appropriate dimension, their inner product is defined as $\langle\mathbf{A}, \mathbf{B}\rangle:=\operatorname{tr} (\mathbf{A}^{\top} \mathbf{B} )$. For a positive semi-definite (PSD) matrix $\mathbf{A}$ and a vector $\mathbf{v}$ of appropriate dimension, we write $\|\mathbf{v}\|_{\mathbf{A}}^2:=\mathbf{v}^{\top} \mathbf{A v}$. The outer product is denoted by $\otimes$. 

In this paper, we will employ the definition of classical DP \cite{dwork2006calibrating} for privacy guarantees. 
\begin{definition}[Differential Privacy
\citet{dwork2006calibrating}]\label{def:dp}
	A randomized algorithm $\mathcal{A}$ is considered $(\varepsilon, \delta)$-differentially private (abbreviated as $(\varepsilon, \delta)$-DP) if, for any two datasets $D$ and $D^{\prime}$ that differ by a single element, and for any event $S$ in the output space of $\mathcal{A}$, the following condition holds: $\mathbb{P}[\mathcal{A}(D) \in S] \leq e^{\varepsilon} \cdot \mathbb{P}[\mathcal{A}(D^{\prime}) \in S]+\delta$
\end{definition} 

In the following, we will introduce some definitions related to the model. We first consider the ReLU regression model to satisfy the following condition, which is commonly referred to in the literature as the "noisy teacher" setting \cite{frei2020agnostic} or the well-structured noise model \citep{goel2019learning}, has been extensively studied in prior research \citep{zou2021benign,varshney2022nearly,shen2023differentially}.

\begin{definition}[Well-specified Condition]
Assume that there exists a parameter $\mathbf{w}_* \in \mathbb{R}^d$ such that
$\mathbb{E}[y \mid \mathbf{x}]=\operatorname{ReLU}(\mathbf{x}^{\top} \mathbf{w}_*),$
and the variance of the model noise can be denoted by
$\sigma^2:=\mathbb{E}\left[(y-\operatorname{ReLU}(\mathbf{x}^{\top} \mathbf{w}_*))^2\right].$
\end{definition}
Moreover, we give some assumptions on the data to ensure the analysis of algorithms. 
\begin{assumption}[Data Covariance]
    Define $\mathbf{H}:=$ $\mathbb{E} [\mathbf{x} \mathbf{x}^{\top} ]$ as the expected data covariance matrix and assume that each entry and the trace of $\mathbf{H}$ are finite. 
\end{assumption}


\begin{assumption}[Fourth Moment Conditions]\label{asm:fourth-moment}
Assume that the fourth moment of $\mathbf{x}$ is finite and there exists a constant $\alpha>0$ such that for any Positive Semi-Definite (PSD) matrix $\mathbf{A}$, the following holds: 
        $$
        \mathbb{E} [\mathbf{x} \mathbf{x}^{\top} \mathbf{A} \mathbf{x} \mathbf{x}^{\top} ] \preceq \alpha \cdot \operatorname{tr}(\mathbf{H A}) \cdot \mathbf{H} .
        $$
\end{assumption}

\begin{remark}
For normal Gaussian distribution, it can be verified that \cref{asm:fourth-moment} holds with ${\alpha}=3$ \cite{zou2021benign}.  Moreover, when the data follows a sub-Gaussian distribution—more precisely, when $\mathbf{x}=\mathbf{H}^{-\frac{1}{2}} \mathbf{z}$, where $\mathbf{z}$ is a sub-Gaussian random vector with variance $\sigma_{\mathbf{z}}^2$--\cref{asm:fourth-moment} remains valid with $\alpha = 16 \sigma_{\mathbf{z}}^2$ \citep{wang2019sparse,varshney2022nearly,zhu2023improved,zhu2024truthful,liu2023near,ding2024understanding}.
\end{remark}

\begin{definition}[$ (\mathbf{H}, C_2, a, b )$-Tail]\label{def:tail}
     A random vector $\mathbf{x}$ satisfies $ (\mathbf{H}, C_2, a, b_{\mathbf{x}} )$-Tail if the following holds:
     \begin{itemize}
        \item $\exists a>0$ s.t. with probability $\geq 1-b_{\mathbf{x}}$,
        \begin{equation}\label{eq:def_tail_1}
            \|\mathbf{x}\|_2^2 \leq \mathbb{E} [\|\mathbf{x}\|_2^2 ] \cdot \log ^{2 a}(1 / b_{\mathbf{x}}), 
        \end{equation}
        \item We have,
        $$
        \max _{\mathbf{v},\|\mathbf{v}\|=1} \mathbb{E} [\exp  ( (\frac{|\langle\mathbf{x}, \mathbf{v}\rangle|^2}{C_2^2\|\mathbf{H}\|_2} )^{1 / 2 a} ) ] \leq 1,
        $$        
        That is, for any fixed $\mathbf{v}$, with probability $\geq 1-b_{\mathbf{x}}$:
        $$
        (\langle\mathbf{x}, \mathbf{v}\rangle)^2 \leq C_2^2\|\mathbf{H}\|_2\|\mathbf{v}\|^2 \log ^{2 a}(1 / b_{\mathbf{x}}).
        $$
     \end{itemize}
\end{definition}
Definition \ref{def:tail} has been extensively employed in recent studies on differential privacy analysis for sub-Gaussian data, as seen in \citep{varshney2022nearly,liu2022dp,liu2023label}. In this work, we assume that each sample $\mathbf{x}$ satisfies the $(\mathbf{H}, C_2, a, b_{\mathbf{x}})$-Tail condition, while the inherent noise $z$ satisfies the $(\sigma^2, C_2, a, b_{\mathbf{x}})$-Tail condition. Furthermore, based on \cref{asm:fourth-moment}, it directly follows that $\|\mathbf{x}\|_2^2 \leq \alpha \operatorname{tr}(\mathbf{H}) \cdot \log^{2a}(1/b_{\mathbf{x}})$ with probability at least $1-b_{\mathbf{x}}$, as shown in \cref{eq:def_tail_1}.

\begin{assumption}[Symmetricity conditions]\label{asm:symmetric}
    Assume that for every $\mathbf{u}, \mathbf{v} \in \mathbb{R}^d$, it holds that:
    \begin{align*}
     \mathbb{E} [\mathbf{x} \mathbf{x}^{\top} &\cdot \mathbbm{1} [\mathbf{x}^{\top} \mathbf{u}>0, \mathbf{x}^{\top} \mathbf{v}>0 ] ] \\
     =&\mathbb{E} [\mathbf{x} \mathbf{x}^{\top} \cdot \mathbbm{1} [\mathbf{x}^{\top} \mathbf{u}<0, \mathbf{x}^{\top} \mathbf{v}<0 ] ], \\
     \mathbb{E} [(\mathbf{x}^{\top} &\mathbf{v})^2  \mathbf{x} \mathbf{x}^{\top} \cdot \mathbbm{1} [\mathbf{x}^{\top} \mathbf{u}>0, \mathbf{x}^{\top} \mathbf{v}>0 ] ] \\
     =&\mathbb{E} [(\mathbf{x}^{\top} \mathbf{v})^2 \mathbf{x} \mathbf{x}^{\top} \cdot \mathbbm{1} [\mathbf{x}^{\top} \mathbf{u}<0, \mathbf{x}^{\top} \mathbf{v}<0 ] ] .
    \end{align*}
\end{assumption}

\begin{remark}\label{rem:sym}
Here, we impose the assumptions that both the second and fourth moments of $\mathbf{x}$ exhibit symmetry. Assumption \ref{asm:symmetric} is satisfied when $\mathbf{x}$ and $-\mathbf{x}$ follow the same distribution. This condition naturally holds for symmetric sub-Gaussian distributions, including symmetric Bernoulli and Gaussian distributions.

\end{remark}

\section{DP-GLMtron Algorithm}
Before presenting our analysis on DP-GLMtron, we first recall the proposed DP-PGD algorithm in \citep{shen2023differentially}. 
The central principle of DP-PGD in ensuring privacy protection involves adding noise to the gradient and executing a projection operation post-model update. This process ensures that the model parameter $\mathbf{w}$ remains bounded, thereby keeping the gradient within manageable limits as well. However, this method leaves several unresolved issues. Primarily, their algorithm assumes that the data are bounded with $\|\mathbf{x}\|_2 \leq 1$, which enables the control of the gradient $\nabla \mathcal{L} (\mathbf{w}) = (\operatorname{ReLU} (\mathbf{x}^{\top} \mathbf{w} ) - y ) \mathbf{x} \cdot \mathbbm{1}[\mathbf{x}^{\top} \mathbf{w}>0]$ via the model $\mathbf{w}$ its subsequent projection. If the data exhibit $O(1)$-sub-Gaussian properties, then we can see $\|\nabla \mathcal{L} (\mathbf{w})\|\leq O(d)$ (with high probability), which means the Gaussian noise added in each iteration has a scale of $\Omega(d^2)$, making a large estimation error (see Remark \ref{remark:4} for a detailed comparison). 
Additionally, it is noticed that at each iteration, DP-PGD requires computing a full gradient. This process is highly costly and inefficient, particularly in settings involving large datasets or high-dimensional data. 

To address the above-mentioned challenges, we consider the DP-GLMtron method built upon the Generalized Linear Model Perceptron (GLMtron) algorithm of \citep{kakade2011efficient} with a one-pass strategy. The fundamental distinction between SGD and GLMtron lies in their respective update rules. Specifically, it takes the following rules:
\begin{equation*}\label{eq:dp_glm}
\begin{aligned}
    \text{SGD:} \quad \mathbf{w}_t &= \mathbf{w}_{t-1} - \eta \cdot  \bm{l}_t\\
    \text{where} \quad \bm{l}_t=(&\operatorname{ReLU} (\mathbf{x}_t^{\top} \mathbf{w}_{t-1} ) - y_t ) \mathbf{x}_t \cdot \mathbbm{1}[\mathbf{x}_t^{\top} \mathbf{w}_{t-1}>0]\\
    \text{GLMtron:} \quad \mathbf{w}_t &= \mathbf{w}_{t-1} - \eta \cdot  (\operatorname{ReLU} (\mathbf{x}_t^{\top} \mathbf{w}_{t-1} ) - y_t ) \mathbf{x}_t.
\end{aligned}
\end{equation*}
The algorithm begins from an initial point $\mathbf{w}_0$ and iterates from $t=0$ to $t=N-1$ with a step size $\eta$. In contrast to the typical update rule of SGD, GLMtron diverges by modifying the derivative of the ReLU function in its update mechanism. The exclusion of this derivative in GLMtron's framework not only simplifies the computational process but also enhances efficiency. Furthermore, \citep{kakade2011efficient} demonstrates that this specific omission significantly contributes to GLMtron's ability to efficiently identify a predictor that closely approximates the optimal solution.
\begin{algorithm}
\caption{DP-GLMtron}
\begin{algorithmic}[1]\label{alg:dp_glm}
\STATE \textbf{Input}: Samples: $\{(\mathbf{x}_i, \mathbf{y}_i)\}_{i=0}^{N-1}$, Clipping Norm: $\xi$, DP Noise Multiplier: ${f}$, Learning Rate: $\eta$, Public Data ${D'=\{\mathbf{x}}_i^{\prime}, {y}_i^{\prime} ) \}_{i=1}^m$, Parameters $\Upsilon, \Delta$
\STATE Randomly permute $\{(\mathbf{x}_i, \mathbf{y}_i)\}_{i=0}^{N-1}$
\STATE Initialize $\mathbf{w}_0 \leftarrow 0$
\FOR{$t = 0, \dots, N-1$}
    \STATE $s_t  \leftarrow \textbf{DP-Threshold} ( \{ ({\mathbf{x}}_i^{\prime}, {y}_i^{\prime} ) \}_{i=1}^m, \mathbf{w}_t, \Upsilon, \Delta)$
    \STATE Sample $\mathbf{g}_t \sim \mathcal{N}(0, \mathbf{I}_{d \times d})$
    \STATE $\mathbf{w}_{t+1} \leftarrow \mathbf{w}_t - \eta ( \text{clip}_{s_t} (\mathbf{x}_t^\top (\operatorname{ReLU} (\mathbf{x}_t^{\top} \mathbf{w}_t) - \mathbf{y}_t)) +  2 {f}s_t \mathbf{g}_t )$
\ENDFOR
\STATE \textbf{return} $\mathbf{w} \leftarrow \frac{1}{N} \sum_{t=0}^{N-1} \mathbf{w}_t$
\end{algorithmic}
\end{algorithm}
Building upon these foundations, we now present the detailed implementation of the proposed DP-GLMtron. The process starts with a random permutation of the dataset to amplify privacy via shuffling~\citep{feldman2022hiding}. In contrast to \cite{shen2023differentially}, our method adopts the one-pass DP strategy without data replacement, ensuring that the time complexity is linear in $N$ and each iterate of model $\mathbf{w}_t$ is independent of data $\mathbf{x}_t$.  See Algorithm \ref{alg:dp_glm} for details. 

A critical step in our approach involves determining the clipping threshold prior to the iterative updates for $\mathbf{w}_t$. An excessively low clipping threshold can result in the loss of important gradient information, leading to high bias \cite{mcmahan2017learning,amin2019bounding}. Therefore, we employ an adaptive clipping by estimating additional public data points~\cite{andrew2021differentially,varshney2022nearly}. Specifically, \cref{alg:dp_threshold} sets the initial threshold $s_0$, which seems to be a threshold that will be iteratively refined to find the approximate maximum. The loop runs for $\lceil\log _2(\Upsilon / \Delta)\rceil$ iterations, covering a range of possible maximum values scaled by the parameter $\Upsilon$ and the discretization width $\Delta$. In each iteration, the \cref{alg:dp_threshold} counts the number of samples for which the value $| \operatorname{ReLU} (\mathbf{x}_t^{\top} \mathbf{w}_t) - \mathbf{y}_t)|$ is less than or equal to the current threshold $s_t$. If the private count is less than the sample size of public data $m$, the threshold is updated for the next iteration to double of its current value. If the count meets $m$, the \cref{alg:dp_threshold} exits the loop. 
When determining the clipping threshold, the model updates via the classical Gaussian mechanism. Finally, \cref{alg:dp_glm} returns to the average of the iterates.

\begin{algorithm}
\caption{DP-Threshold}
\begin{algorithmic}[1]\label{alg:dp_threshold}
\STATE \textbf{Input}: Estimating Samples: $\{(x_i^{\prime}, y_i^{\prime})\}_{i=1}^{m}$, Current Model: $w$, DP Noise Multiplier: $f$, Domain Size: $\Upsilon$, Discretization Width: $\Delta$
\STATE $s_0  \leftarrow \Delta$
\FOR{$i \in \{0, \ldots, \lceil\log _2(\Upsilon / \Delta) \rceil \} $}
    \STATE $u  \leftarrow | \{ | \text{ReLU}(\mathbf{x}_j^{\top}\mathbf{w})-y_j | \leq s_i: j \in\{0, \ldots, m\} \} |$
    \IF{Estimating samples are public} 
        \STATE  $u_{\text {priv }}  \leftarrow u$
    \ELSE
        \STATE $u_{\text {priv }}  \leftarrow u+\mathcal{N} (0, \lceil\log _2(\Upsilon / \Delta) \rceil f^2 )$
    \ENDIF
    \IF{$u_{\text {priv }} \textless m$}
        \STATE $s_{i+1}  \leftarrow 2*s_i$
    \ELSE
        \STATE \textbf{break}
    \ENDIF
\ENDFOR
\STATE \textbf{return}  $s_{\text{priv}}  \leftarrow s_i$
\end{algorithmic}
\end{algorithm}

\begin{theorem}[Privacy Guarantee] \label{thm:glm_privacy}
    DP-GLMtron satisfies $(\varepsilon, \delta)$-DP with a noise multiplier set to ${f} = \Omega(\frac{\log (N / \delta)}{\varepsilon \sqrt{N}})$ if  $\varepsilon = O(\sqrt{\frac{\log (N / \delta)}{N}})$ and $0<\delta<1$.
\end{theorem}

\begin{remark}\label{rem:dp_glm_dp}
Note that the privacy budget is limited to $\varepsilon = {O}(\sqrt{\frac{\log (N / \delta)}{N}})$ because of privacy amplification via shuffling in~\cite{feldman2022hiding}.  If there is no shuffling, plainly using the Gaussian mechanism will make ${f} = \Omega(\frac{\log (N / \delta)}{\varepsilon })$. Thus, privacy amplification can improve a factor of $\tilde{O}(\sqrt{N})$. However, this highlights a key limitation in DP-GLMtron: as the dataset size $N$ increases, the algorithm is constrained by a smaller privacy budget $\varepsilon$.
\end{remark}

\begin{theorem}[Utility Guarantee]\label{thm:dp_glm_utility}
    Let $D=\{(\mathbf{x}_i, y_i)\}_{i=0}^{N-1}$ be sampled i.i.d. with $\mathbf{x}_i \sim \mathcal{D}$ satisfying $(\mathbf{H}, C_2, a, b_{\mathbf{x}})$-Tail, and the distribution of the inherent noise $z$ satisfies $(\sigma^2, C_2, a, b_{\mathbf{x}})$-Tail with $b_{\mathbf{x}}=\frac{1}{\text{Poly}(N)}$.  Let $\kappa$ be the condition number of the covariance matrix $\mathbf{H}$ and denote $R_x^2 = \alpha \operatorname{tr}(\mathbf{H}) \cdot \log^{2a}(1/b_{\mathbf{x}})$.

    Initialize parameters in DP-GLMtron as follows: stepsize $\eta=\min \{\frac{1}{2 R_x^2}, \frac{c_1}{\log ^{4 a} N} \cdot \frac{1}{C_2^2 R_x^2 \kappa^2} \cdot \frac{1}{d f^2}\}$, where $c_1, c_2>0$ are global constants, noise multiplier $f={\Omega}(\frac{\log (N / \delta)}{\varepsilon \sqrt{N}})$, domain size $\Upsilon=C_2 R_x(\|\mathbf{w}^*\|_{\mathbf{H}}+\sigma) \log ^{2 a} N$, granularity $\Delta = \frac{\|\mathbf{w}^*\|_{\mathbf{H}}+\sigma}{ \text{Poly}(N)}$, public datasize $m = \Omega(\frac{\log (N / \delta)\sqrt{\log (N \log N)} }{\varepsilon \sqrt{N}} ).$  Then, the output $\overline{\mathbf{w}}$ of DP-GLMtron achieves the following excess risk w.p. $\geq 1-1 / \text{Poly}(N)$ over randomness in data and algorithm:
    $$
    \begin{aligned}
    \mathcal{L}(\overline{\mathbf{w}})-\mathcal{L}(\mathbf{w}^*) &\lesssim \frac{\|\mathbf{w}_*\|_{\mathbf{H}}^2}{\text{Poly}(N)} + \frac{\sigma^2 d}{N} \\
    &+ \frac{d^2 \log N  \log (1/\delta)}{N^2 \varepsilon^2} \cdot C_2^2 \kappa^2(\sigma^2+\|\mathbf{w}_*\|_{\mathbf{H}}^2).
    \end{aligned}
    $$
\end{theorem}
\begin{remark}\label{remark:4}
Theorem \ref{thm:dp_glm_utility} provides a utility guarantee for the DP-GLMtron algorithm, balancing privacy and performance. The excess risk is composed of three key components: The first component, dependent on $\|\mathbf{w}_*\|_{\mathbf{H}}^2$, diminishes polynomially in $N$. The second component corresponds to the inherent model noise, achieving the optimal rate (up to a constant factor) for non-private ReLU regression as established by \cite{wu2023finite}. The third component is of the order $\Tilde{O}(\frac{d^2}{N^2 \varepsilon^2})$. For $N = \Omega(d)$, the bound implies nearly optimal sample complexity, further supported by the lower bound derived in \cref{sec:lower_bound}, disregarding constant factors.

Compared to \cite{shen2023differentially}, our analysis here relax the data assumption $\|\mathbf{x}\|_2 \leq 1$. If we assume that $\|\mathbf{x}\|_2\leq O(\sqrt{d})$, via the same analysis as in \cite{shen2023differentially},  we can show the utility bound will be $O(\min\{\frac{d\sqrt{d}}{N\varepsilon}, (\frac{d}{N\varepsilon})^\frac{2}{3}\})$, which is worse than the one in Theorem~\ref{thm:dp_glm_utility}. 
\end{remark} 

\section{Advanced DP-mini-batch-GLMtron}

A key concern with the DP-GLMtron algorithm is its limited practicality where the privacy budget $\varepsilon$ is small, potentially restricting its utility in real-world applications (see \cref{thm:glm_privacy} for more details). Furthermore, \cref{alg:dp_glm} may require additional public data to estimate the threshold.  To overcome these challenges, we introduce DP-Mini-batch-GLMtron (\cref{alg:dp_mini_glm}) in this section.

Specifically, the algorithm first operates by randomly partitioning the training samples $\{(\mathbf{x}_i, y_i)\}_{i=0}^{N-1}$,  and setting the number of iterations $T=N /(b+m)$, where $b$ and $m$ are batch sizes and estimating sample size for determining the threshold. 
It is worth noting that, in this approach, a separate public dataset is not required to estimate the clipping threshold. Instead, we divide each batch of data and use a portion of it as the estimation data for the threshold. Therefore, in each iteration, the algorithm processes a mini-batch of data with size $m$ and computes the DP-Threshold $\gamma_t$ using estimating samples $\{(\mathbf{x}_i^{\prime}, y_i^{\prime})\}_{i=1}^m$, and updates the clipping parameter $s_t$. In contrast to DP-GLMtron, we need to protect the counting numbers during the estimation process as we are using private data. Noise $g_t$ is sampled from a Gaussian distribution and added to the gradient step for privacy preservation. The model weights are updated using step \ref{alg:step_w}, where $\bm{l}_{t+1}$ is the averaged clipped gradient. After iterating $T$ times, the final weight estimate $\mathbf{w}$ is returned as the average of all weight updates.

\begin{theorem}\label{thm:dpmini}
    Algorithm DP-mini-batch-GLMtron with noise multiplier ${f} \geq \frac{2 \sqrt{\log (1 / \delta)+\varepsilon}}{\varepsilon}$ satisfies $(\varepsilon, \delta)$-DP. 
    Furthermore, if $\varepsilon \leq \log (1 / \delta)$, then ${f} \geq \frac{\sqrt{8 \log (1 / \delta)}}{\varepsilon}$ suffices to ensure $(\varepsilon, \delta)$-DP.
\end{theorem}

Theorem \ref{thm:dpmini} addresses the limitations of the DP-GLMtron algorithm, particularly requiring a small privacy budget $\varepsilon$, which can severely limit its utility in practical scenarios. By processing a subset of data in each iteration, the algorithm effectively reduces the sensitivity of the overall computation. This reduction allows for less noise to be added while maintaining the same level of privacy, thus improving the utility of the model. 
\begin{algorithm}
\caption{DP-MBGLMtron}
\begin{algorithmic}[1]\label{alg:dp_mini_glm}
\STATE \textbf{Input}: Training Samples: $\{(x_i, y_i)\}_{i=0}^{N-1}$, Learning Rate: $\eta$, DP Noise Multiplier: $f$, Expected $x$ Norm: ${\sqrt{\alpha \operatorname{tr}(\mathbf{H})}}$,  Parameters $\Upsilon, \Delta$
\STATE Initialize $\mathbf{w}_0  \leftarrow \mathbf{0}$ and $s_0  \leftarrow \Delta$
\STATE Set $T  \leftarrow N / (b + m)$
\FOR{$t = 0 \ldots T-1$}
\STATE Set $\tau(t) \leftarrow (b + m)t$
\STATE $\gamma_t  \leftarrow \textbf{DP-Threshold} ( D_t, \mathbf{w}_t, f, \Upsilon, \Delta)$, where $D_t=\{ (\mathbf{x}_{\tau(t)+j}, y_{\tau(t)+j} ) \}_{j=0}^{m-1}$
\STATE $s_t = {\sqrt{2\alpha \operatorname{tr}(\mathbf{H})}} C_2  \log^{2a}N \cdot \gamma_t$ and add $s_t$ to the list $\bm{s}$
\STATE Sample $\mathbf{g}_t \sim \mathcal{N}(0, \mathbf{I}_{d \times d})$
\STATE \label{alg:step_w}$ \mathbf{w}_{t+1}  \leftarrow \mathbf{w}_t - \eta \bm{l}_{t+1}- \frac{2 {f} s_t \eta}{b} \mathbf{g}_t$, 
where $\bm{l}_{t+1} := \frac{1}{b} \sum_{i=0}^{b-1} \operatorname{clip}_{{s}_t}(\mathbf{x}_{\tau(t)+m+i}(\operatorname{ReLU}(\mathbf{x}_{\tau(t)+m+i}^{\top} \mathbf{w}_t)-y_{\tau(t)+m+i}))$ 
\ENDFOR
\STATE \textbf{return}  $\mathbf{w} := \frac{1}{T} \sum_{t=0}^{T-1} \mathbf{w}_t$
\end{algorithmic}
\end{algorithm}

\begin{theorem}\label{thm:dp_miniglm_utility}
    Let $D=\{(\mathbf{x}_i, y_i)\}_{i=0}^{N-1}$ be sampled i.i.d. with $\mathbf{x}_i \sim \mathcal{D}$ satisfying $(\mathbf{H}, C_2, a, b_{\mathbf{x}})$-Tail, and the distribution of the inherent noise $z$ satisfies $(\sigma^2, C_2, a, b)$-Tail with $b_{\mathbf{x}}=\frac{1}{\text{Poly}(N)}$. 

    Initialize parameters in DP-MBGLMtron as follows: batch size $b=\frac{N}{T}-m$, estimating sample size $m=\frac{b}{10}$, appropriate stepsize $\eta=O(\frac{1}{R_x^2})$, number of iterations $T=O(\kappa \log (N))$, domain size $\Upsilon=C_2 R_x(\|\mathbf{w}^*\|_{\mathbf{H}}+\sigma) \log ^{2 a} N$, granularity $\Delta = \frac{\|\mathbf{w}^*\|_{\mathbf{H}}+\sigma}{ \text{Poly}(N)}$ and noise multiplier $f=\frac{\sqrt{8 \log (1 / \delta)}}{\varepsilon}$. Then, the output $\overline{\mathbf{w}}$ achieves the following excess risk with probability $\geq 1-1 / \text{Poly}(N)$ over the randomness in data and algorithm:
    $$
    \begin{aligned}
     \mathcal{L}(\overline{\mathbf{w}})-\mathcal{L}(\mathbf{w}^*) &\lesssim \frac{\|\mathbf{w}_*\|_{\mathbf{H}}^2}{\text{Poly}(N)} + \frac{\sigma^2 d}{N} \\
     &+ \frac{d^2 \log N  \log (1/\delta)}{N^2 \varepsilon^2} \cdot C_2^2 \kappa^2(\sigma^2+\|\mathbf{w}_*\|_{\mathbf{H}}^2).
    \end{aligned}
    $$
\end{theorem}
To prove the utility, we have the following utility for the DP-Threshold algorithm. 
\begin{theorem}[DP-Threshold]
    Suppose that DP-Threshold is applied to $m$ estimated data with certain parameters $\{ \mathbf{w}_t, \Upsilon, f, \Delta \} $, \cref{alg:dp_threshold} satisfies $(\varepsilon/2,\delta/2)$-DP with ${f} \geq \frac{2 \sqrt{\log (1 / \delta)+\varepsilon}}{\varepsilon}$. Given $\Lambda = f \sqrt{2 \log (\Upsilon / \Delta) \log (\log (\Upsilon / \Delta) / b_{\mathbf{x}})}$, then with probability at least $1-b_{\mathbf{x}}$, \cref{alg:dp_threshold} outputs a private threshold $s_{\text {priv }}$ such that
    \begin{itemize}
        \item $|\{|\operatorname{ReLU} (\mathbf{x}_i^{\top} \mathbf{w} )-y_i| \leq s_{\text {priv }}: i \in\{0, \ldots, m\}\}| \geq m-\Lambda$,
        \item $|\{|\operatorname{ReLU} (\mathbf{x}_i^{\top} \mathbf{w} )-y_i| \leq \max \{\frac{s_{\text {priv }}}{2}, \Delta\}: i \in\{0, \ldots, m\}\}|<m-\Lambda$.
    \end{itemize}
\end{theorem}
\begin{remark}\label{rem:threshold}
We provide further details regarding the threshold here. Suppose $\Lambda = \Omega(f \log N)$. With probability at least $1 - 1/\text{Poly}(N)$, at least $m - \Lambda$ data points satisfy the condition$|\operatorname{ReLU}(\mathbf{x}_i^{\top} \mathbf{w}) - y_i| \leq s_{\text{priv}}.$
According to \cref{def:tail}, and considering that $\mathbf{w}_t$ is independent of $\mathbf{x}_{\tau(t)+j}$, we have the following with probability $\geq 1 - 1/\text{Poly}(N)$:
$$
\|\mathbf{x}_{\tau(t)+j}( \operatorname{ReLU}((\mathbf{x}_{\tau(t)+j})^{\top} \mathbf{w}_t) - \mathbf{x}_{\tau(t)+j})\| \leq s_t,
$$
by setting $s_t = O(R_x \gamma_t \log^a N)$ for all iterations. Moreover, recalling that $\Delta$ is the granularity of the search for the approximate maximum threshold and \cref{alg:dp_threshold} will terminate once most of the samples fit under the current guess for $\gamma_{t}$, meaning $\gamma_{t}$ will not exceed the current value plus the granularity $\Delta$. That is,  $\gamma_t \leq C_2 \log^a N(\sqrt{\kappa}\|\mathbf{w}_t-\mathbf{w}^*\|_{\mathbf{H}} + \sigma+\Delta)$ and $\Delta=\frac{\|\mathbf{w}^*\|_\mathbf{H}+\sigma}{\text{Poly}(N)}$. The choice of $\Delta$ reflects the granularity needed as the current weight approaches the optimal one. Therefore, with probability $\geq 1 - 1/\text{Poly}(N)$, both events hold: 1) the threshold is not required for any data point in its batch, and 2) the above condition on $\gamma_t$ is satisfied in each iteration. 
\end{remark}

\section{Lower Bound} \label{sec:lower_bound}
In this section, we demonstrate that the minimax rate of the excess population risk for $(\varepsilon, \delta)$-DP algorithms is lower bounded by $\Omega(\frac{d^2}{N^2\varepsilon^2})$, indicating that the bound mentioned above is optimal up to logarithmic factors. To show this, we consider the following class of distributions for $(\mathbf{x}, y)$: 
\begin{equation}
    \begin{aligned}
    &\mathcal{P}(\sigma, d, \mathcal{W})=\{P(\mathbf{x}, y)| \mathbf{w} \in \mathcal{W}, \mathbf{x}\sim \text{Uni}([-1, 1]^d), \\
    &f_{\mathbf{w}}(y|\mathbf{x})=\frac{1}{\sqrt{2\pi}\sigma}\exp(-\frac{(y-\text{ReLU}(\mathbf{w}^\top \mathbf{x}))^2 }{2\sigma^2}\}, 
\end{aligned}
\end{equation}
where $\mathcal{W}=\{\mathbf{w} \in \mathbb{R}^d| \|\mathbf{w} \|_2\leq 1\}$, and $f_{\mathbf{w} }(y|\mathbf{x})$ is the density function of $y$ given $\mathbf{w} $ and $\mathbf{x}$. Thus, for any $(\mathbf{x}, y)\sim P\in  \mathcal{P}(\sigma, d, \mathcal{W})$ we have $y=\text{ReLU}(\mathbf{w}^\top \mathbf{x})+z$, where $z\sim \mathcal{N}(0, \sigma^2)$, and the covariate $\mathbf{x}$ satisfies Assumption \ref{asm:fourth-moment} with $\alpha, \beta=O(1)$ and Assumption \ref{asm:symmetric}. It also satisfies $ (\mathbf{I}_d, C_2, a, b )$-Tail with some $a, b=O(1)$. 

Our lower bounds will be in the form of private minimax risk. Let $\mathcal{P}$ be a class of distributions over a data universe $\mathcal{X}$. For each distribution $p\in \mathcal{P}$, there is a deterministic function $\mathbf{w}(p)\in \mathcal{W}$, where $\mathcal{W}$ is the parameter space. Let $\rho: \mathcal{W} \times \mathcal{W} :\mapsto \mathbb{R}_+ $ be  a semi-metric function on the space $\mathcal{W}$ and $\Phi: \mathbb{R}_+\mapsto \mathbb{R}_+$ be a non-decreasing function with $\Phi(0)=0$.\footnote{In this paper, we assume that  $\rho(\mathbf{w},\mathbf{w}^{\prime})=\|{\mathbf{w}}-{\mathbf{w}}^{\prime}\|_{\Sigma_\mathbf{x}}$ and $\Phi(\mathbf{x})=\mathbf{x}^2$ unless specified otherwise, where $\Sigma_\mathbf{x}=I_d$ is the covariance matrix of $x$. Here, we do not omit $\Sigma_\mathbf{x}$ to make our results consistent with previous results.} We further assume that  $D=\{X_i\}_{i=1}^{n}$ are  $n$ i.i.d observations drawn according to some distribution $p\in \mathcal{P}$, and   $\hat{\mathbf{w}}:\mathcal{X}^N\mapsto \mathcal{W}$ be some estimator. In the $(\varepsilon, \delta)$-DP model, the estimator $\hat{\mathbf{w}}$ is obtained via some $(\varepsilon, \delta)$-DP mechanism $Q$.  The $(\varepsilon, \delta)$-private minimax risk is defined as:   
\begin{equation*}
\mathcal{M}_n(\theta(\mathcal{P}), \Phi\circ \rho):=\inf_{Q\in \mathcal{Q}_{\varepsilon,\delta}}
\sup_{p\in \mathcal{P}}\mathbb{E}_{p, Q}[\Phi(\rho(Q(D), \mathbf{w}(p))],
\end{equation*}
where $\mathcal{Q}_{\varepsilon, \delta}$ is the set of all the $(\varepsilon, \delta)$-DP mechanisms. 

To prove the lower bound, we aim to use the tracing attack argument in \citep{cai2021cost}. Specifically,  a tracing attacker attempts to construct an attack to detect the absence/presence of a sample $\mathbf{x}$ in a target dataset $D$ by looking at the (private) estimator $M(D)$ for the dataset. If one can construct a tracing attack that is powerful, given an accurate
estimator, an argument by contradiction leads to a lower bound: suppose
a differentially private estimator computed from the target data set is sufficiently accurate, the tracing adversary will be able to determine whether
a given sample belongs to the dataset or not, thereby contradicting with
the differential privacy guarantee. The privacy guarantee and the tracing
adversary together ensure that a differentially private estimator cannot be
"too accurate". In detail, for a dataset $D$ and a target sample $(\mathbf{x}, y)$, we consider the following tracing attack: 
\begin{equation}\label{eq:attack}
    \begin{aligned}
        \mathcal{T}_\mathbf{w}((\mathbf{x},y), M(D))&=\langle M(D)-\mathbf{w}, (y-\text{ReLU}(\mathbf{w}^\top \\
        & \mathbf{x}))\mathbf{x} \cdot \mathbbm{1}(\mathbf{w}^\top \mathbf{x}>0)\rangle. 
    \end{aligned}
\end{equation}
We will first show that if $(\mathbf{x},y)\in D$, then the attack value is small; otherwise, it will be large. 
\begin{lemma}\label{thm:low_1}
    Consider $D=(Y, X)=\{(\mathbf{x}_i, y_i)\}_{i=1}^N$ be i.i.d. sampled from $P\in \mathcal{P}(\sigma, d, \mathcal{W})$ with the underlying $\mathbf{w}$. For every $(\varepsilon,\delta)$-DP algorithm $M$ satisfying $\mathbb{E}_{Y, X|\mathbf{w}} 
 \|M(D)-\mathbf{w}\|_2^2 =o(1)$ for all $\mathbf{w}\in \mathcal{W} $, then we have the following: 
\begin{enumerate}
    \item For each $i\in [n]$, denote $D_i'$ as the dataset obtained by replacing $(\mathbf{x}_i, y_i)$ in $D$ with an independent copy, then we have 
    $$\left\{
    \begin{aligned}
        &\mathbb{E}\mathcal{T}_{\mathbf{w}}((\mathbf{x}_i,y_i), M(D_i')) = 0,  \\ 
        &\mathbb{E}|\mathcal{T}_{\mathbf{w}} ((\mathbf{x}_i,y_i), M(D_i'))| \leq  \sigma \sqrt{\mathbb{E}\|M(D)-\mathbf{w}\|^2_{\Sigma_\mathbf{x}}}, 
    \end{aligned}
    \right.$$ 
    \item There exists a prior distribution of $\pi$ for $\mathbf{w}$ supported on $\mathcal{W}$ such that 
    $$ \sum_{i\in [n]}\mathbb{E}_\pi  \mathbb{E}_{Y, X|\mathbf{w}}[\mathcal{T}_\mathbf{w}((\mathbf{x}_i,y_i), M(D))]\geq \Omega(\sigma^2 d).$$
\end{enumerate}
\end{lemma}
\begin{remark}
\cref{thm:low_1} establishes a connection between the accuracy of a DP estimator and the potential for privacy breaches via tracing attacks. Specifically, when $(x_i, y_i)$ is independent on $D_i'$, we can control the variance of $\mathcal{T}_{\mathbf{w}}((\mathbf{x}_i,y_i), M(D_i'))$, which is upper bounded by $\sigma\sqrt{\mathbb{E}\|M(D)-\mathbf{w}\|^2_{\Sigma_\mathbf{x}}}$. Moreover, if they are dependent, then from part 2 we can see there exists $w$ such that $\mathcal{T}_\mathbf{w}((\mathbf{x}_i,y_i), M(D))\geq \Omega(\frac{\sigma^2 d}{n})$. These results show that when $\|M(D)-\mathbf{w}\|^2_{\Sigma_\mathbf{x}}$ is small enough, then the attacker can distinguish $D_i'$ and $D$, making DP failed. Specifically, we have the following result:

\end{remark}

\begin{theorem}\label{thm:lower_bound}
    For $0<\varepsilon<1$ and $\delta\leq N^{-(1+u)}$ for some $u>0$, we have 
    \begin{align*}
& \inf_{M\in \mathcal{Q}_{\varepsilon,\delta}}
        \sup_{p\in \mathcal{P}}\mathbb{E}_{D\sim p^N, M}[\mathcal{L}(M(D))-\mathcal{L}(\mathbf{w})]\geq \\
      & \frac{1}{4} \inf_{M\in \mathcal{Q}_{\varepsilon,\delta}}
        \sup_{p\in \mathcal{P}}\mathbb{E}_{D\sim p^N, M}[\|M(D)-\mathbf{w}\|_{\Sigma_\mathbf{x}}^2]\geq O( \frac{\sigma^2 d^2}{N^2\varepsilon^2}). 
    \end{align*}
\end{theorem}
\begin{remark}
    \cref{thm:lower_bound} shows that under differential privacy, if the estimator $M$ is too accurate, it may inadvertently leak information about the presence of specific data points. Therefore, the mechanism must maintain the error rate of $O( \frac{d^2}{N^2\varepsilon^2})$,  which aligns with our previous upper bound, thus confirming the tightness of our results.
\end{remark}

\begin{figure*}[t]
    \begin{subfigure}{0.31\textwidth}
        \centering
        \includegraphics[width=\linewidth]{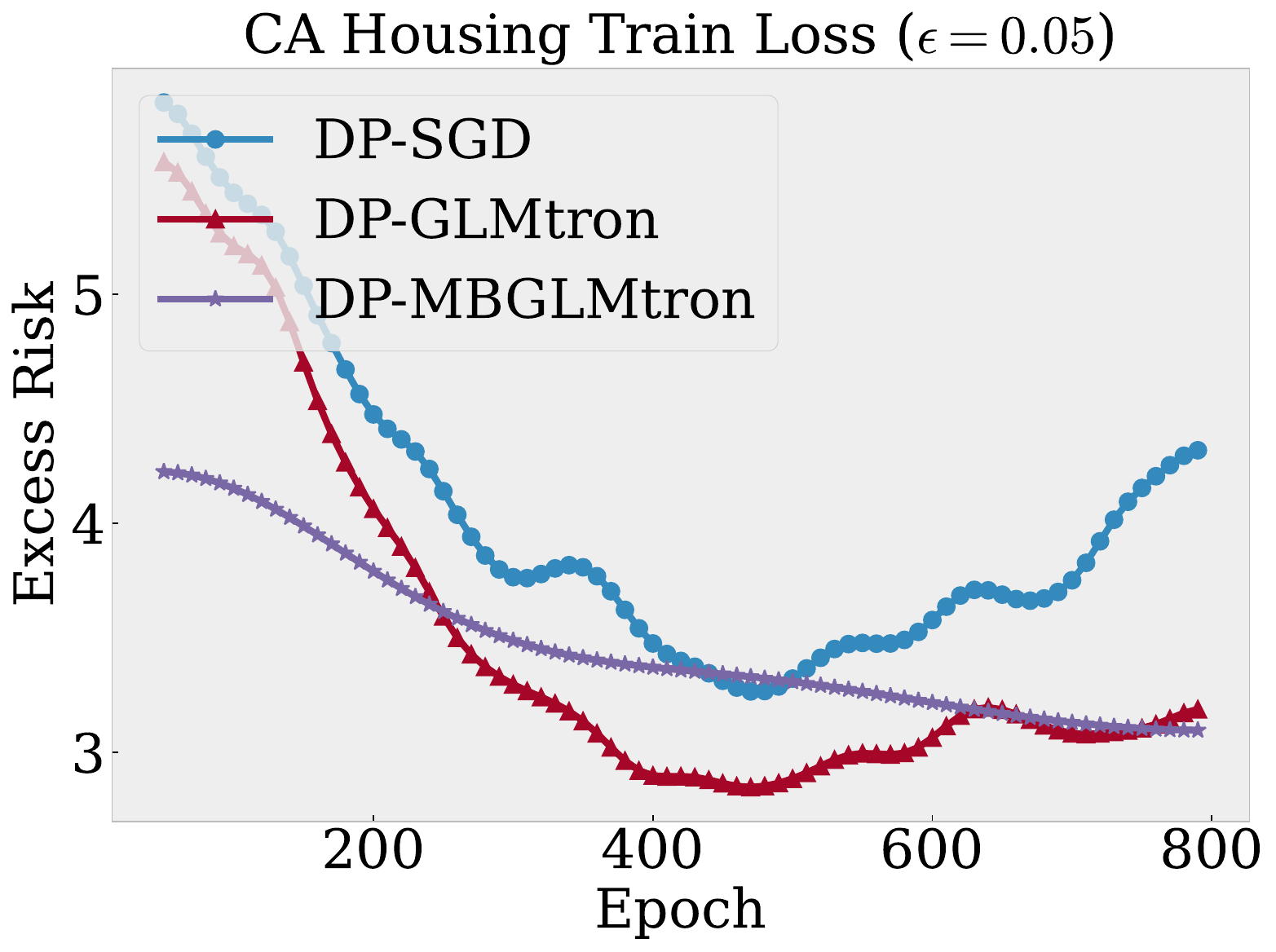}
        \caption{}
        \label{fig:catrain005}
    \end{subfigure}%
    \begin{subfigure}{0.32\textwidth}
        \centering
        \includegraphics[width=\linewidth]{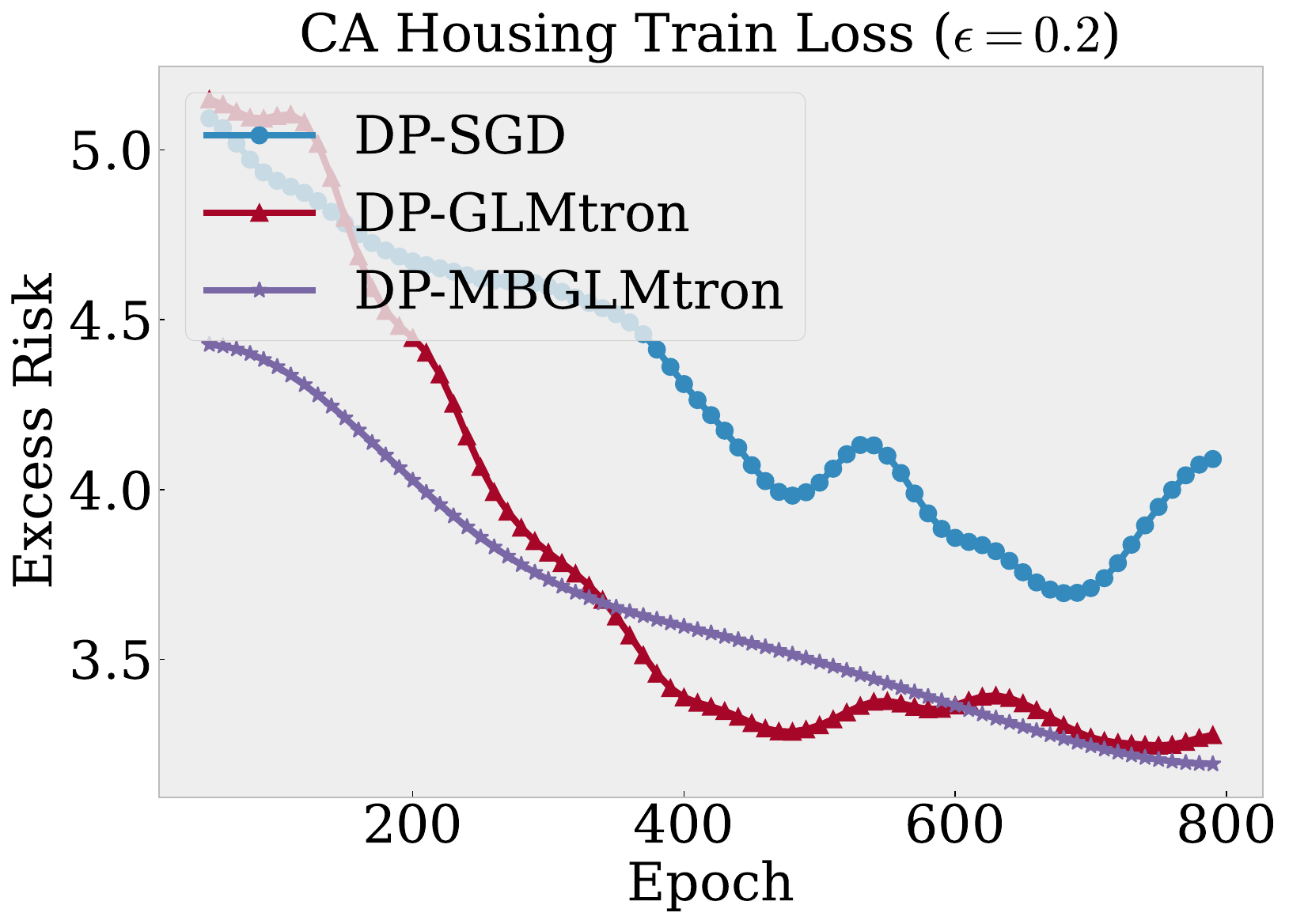}
        \caption{}
        \label{fig:catrain02}
    \end{subfigure}
    \begin{subfigure}{0.32\textwidth}
        \centering
        \includegraphics[width=\linewidth]{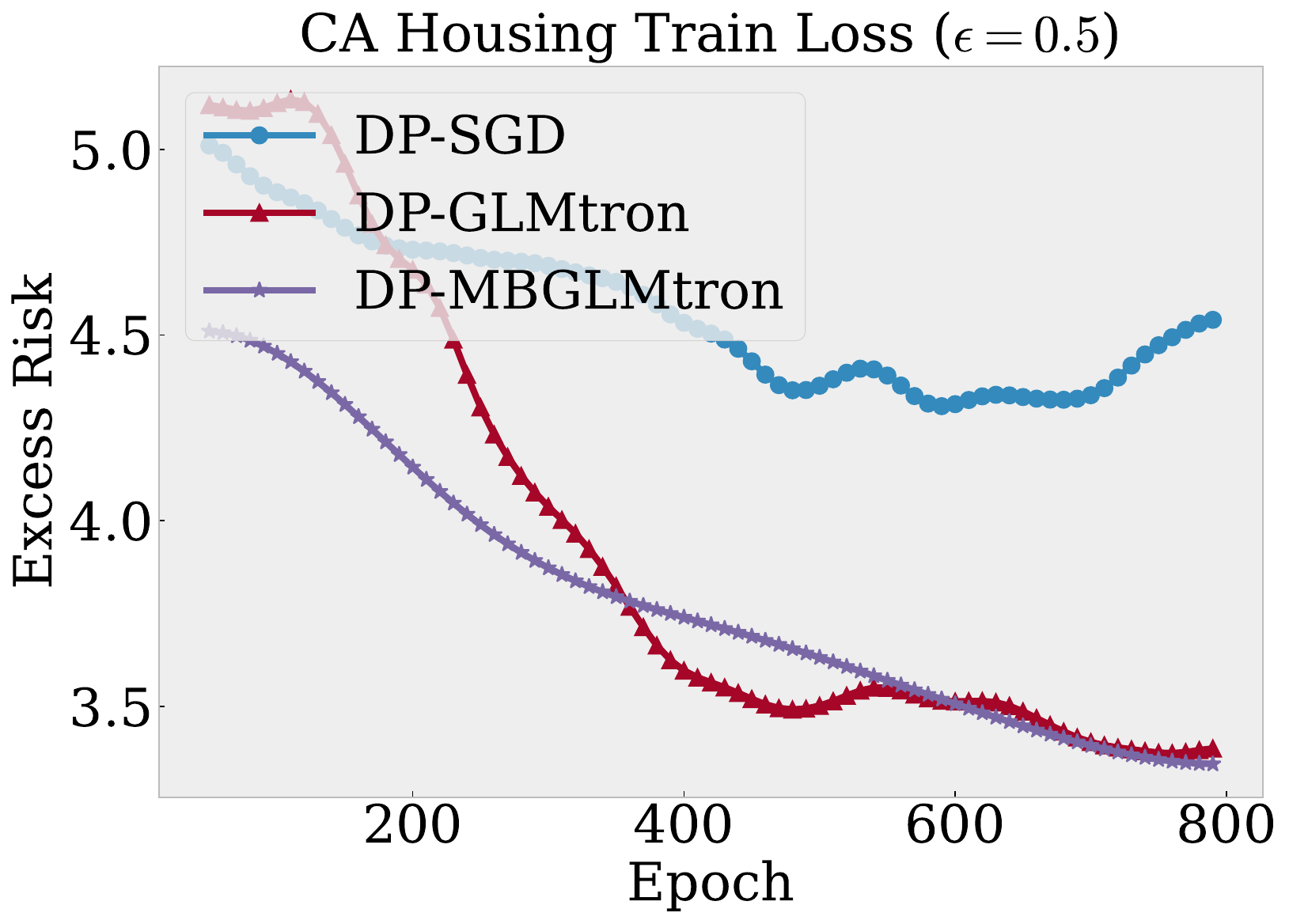}
        \caption{}
        \label{fig:catrain05}
    \end{subfigure}%

    \begin{subfigure}{0.32\textwidth}
        \centering
        \includegraphics[width=\linewidth]{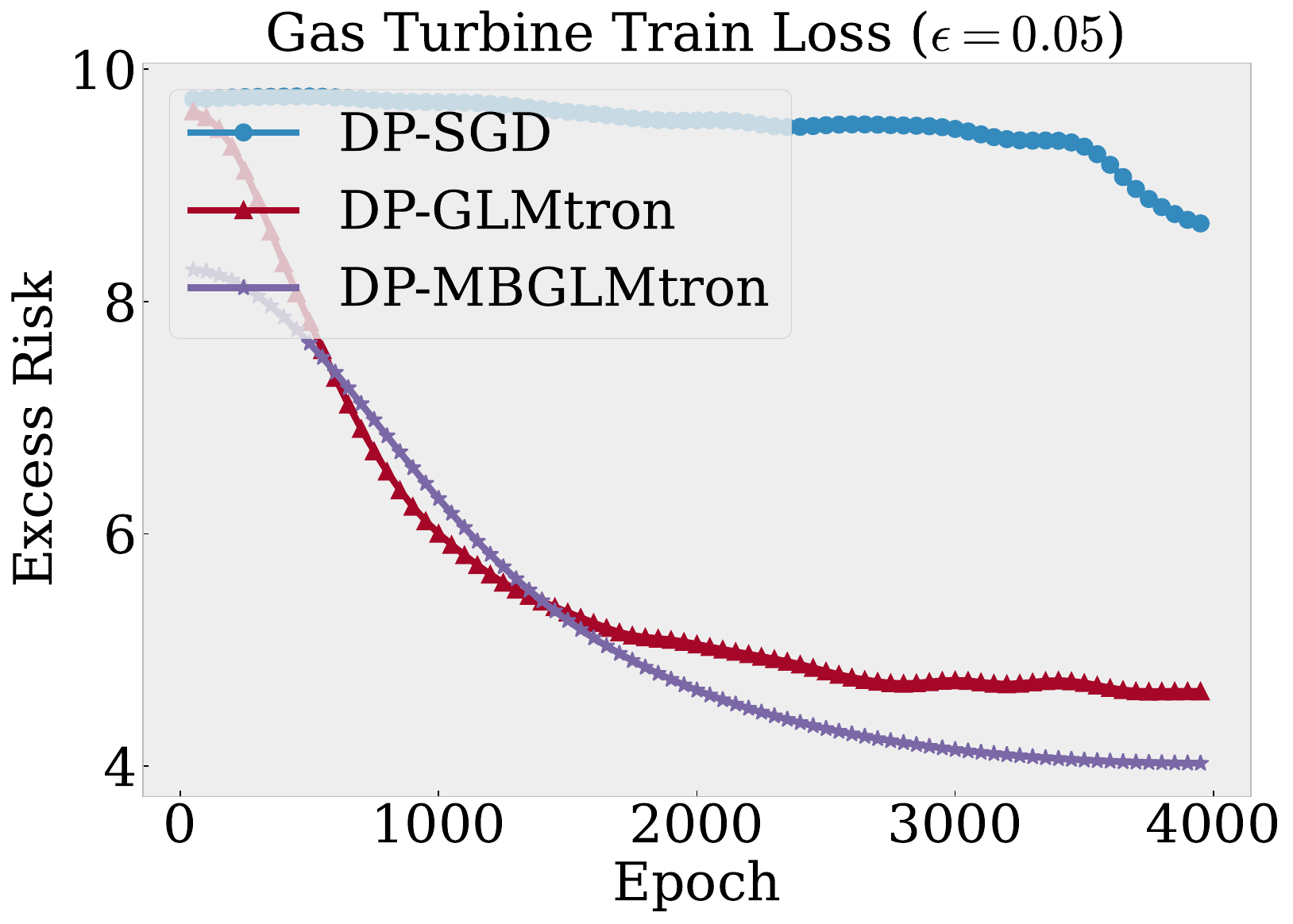}
        \caption{}
        \label{fig:gastrain005}
    \end{subfigure}%
    \begin{subfigure}{0.32\textwidth}
        \centering
        \includegraphics[width=\linewidth]{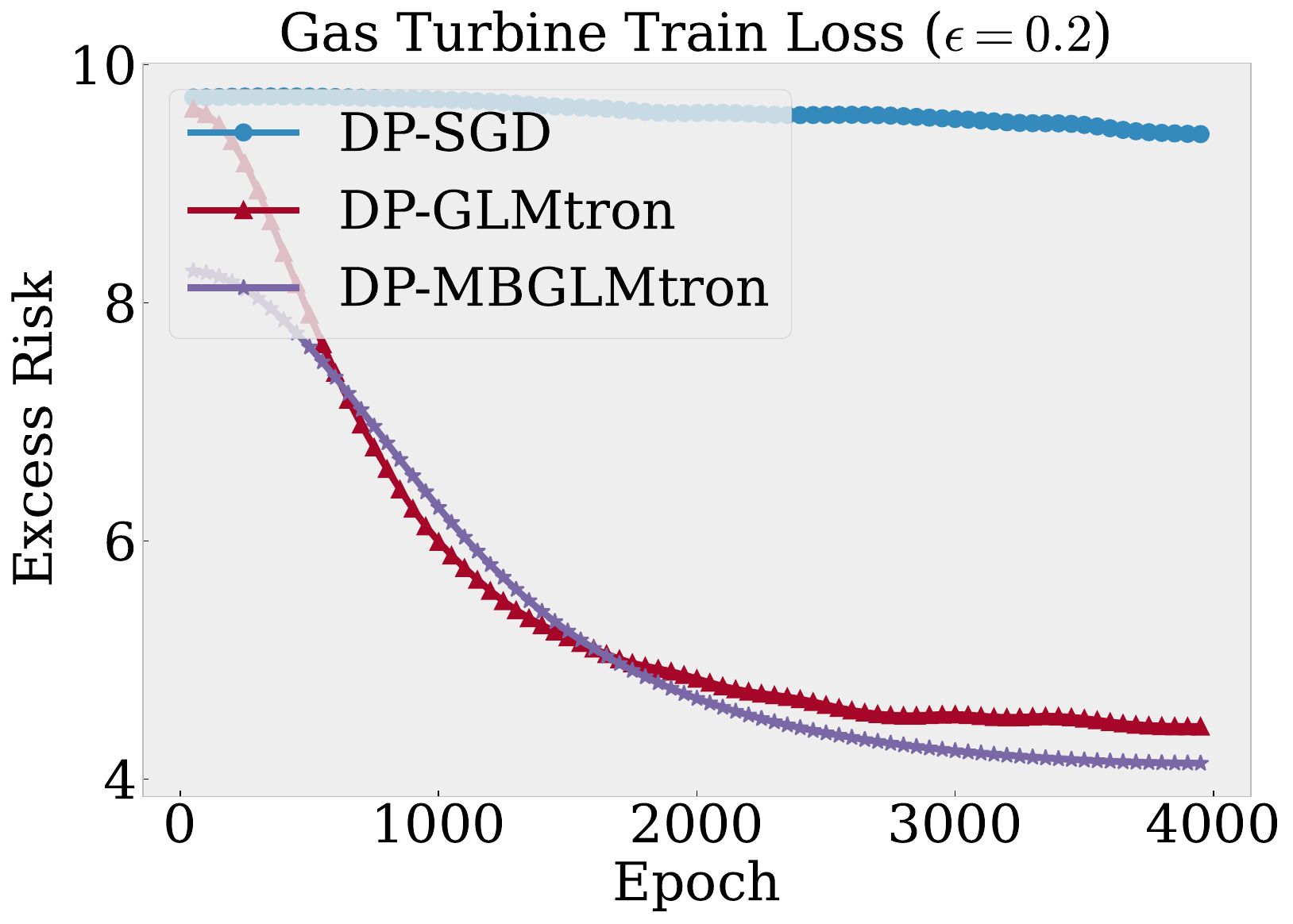}
        \caption{}
        \label{fig:gastrain02}
    \end{subfigure}%
    \begin{subfigure}{0.32\textwidth}
        \centering
        \includegraphics[width=\linewidth]{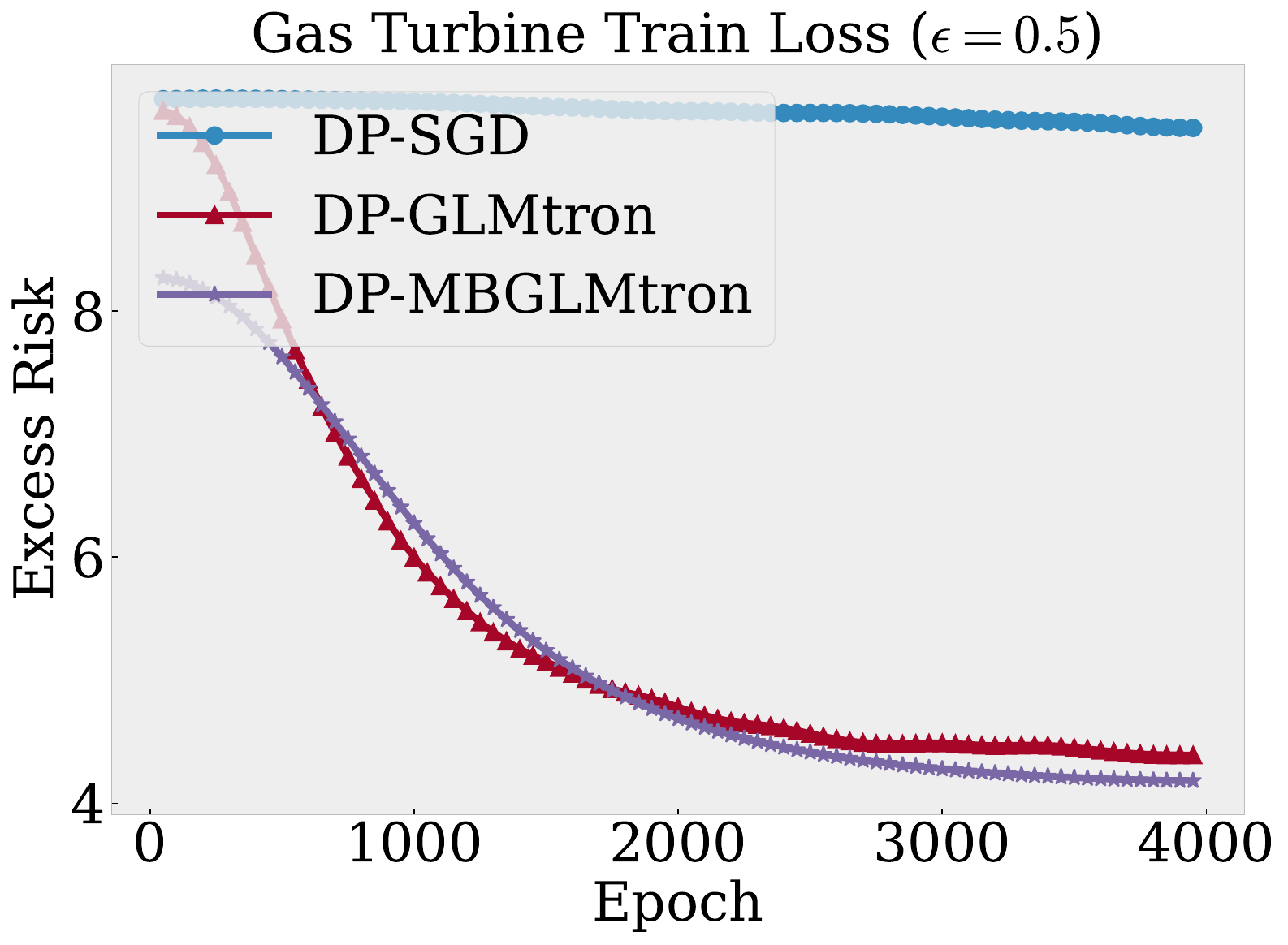}
        \caption{}
        \label{fig:gastrain05}
    \end{subfigure}

    \begin{subfigure}{0.32\textwidth}
        \centering
        \includegraphics[width=\linewidth]{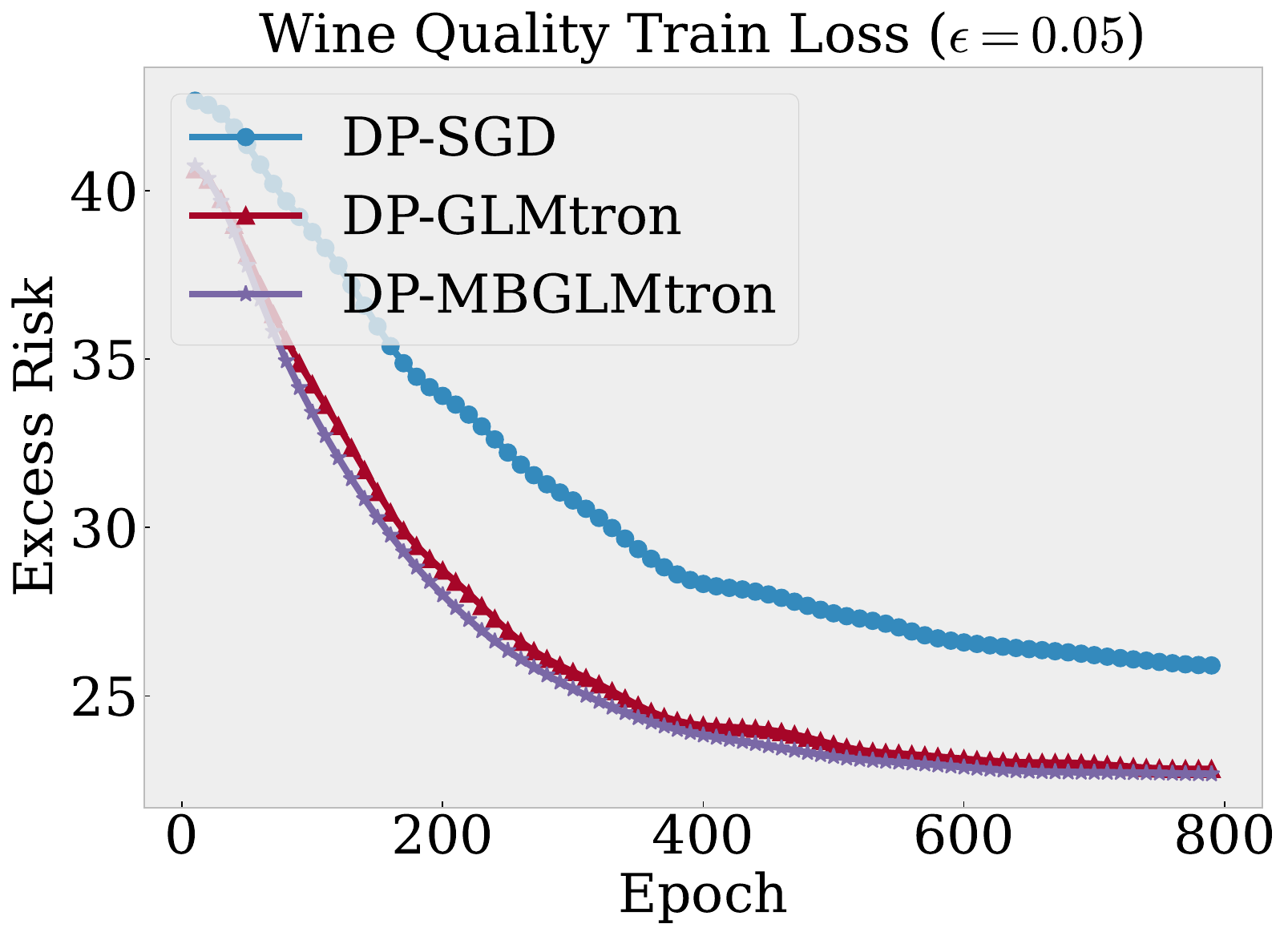}
        \caption{}
        \label{fig:blogtrain005}
    \end{subfigure}%
    \begin{subfigure}{0.32\textwidth}
        \centering
        \includegraphics[width=\linewidth]{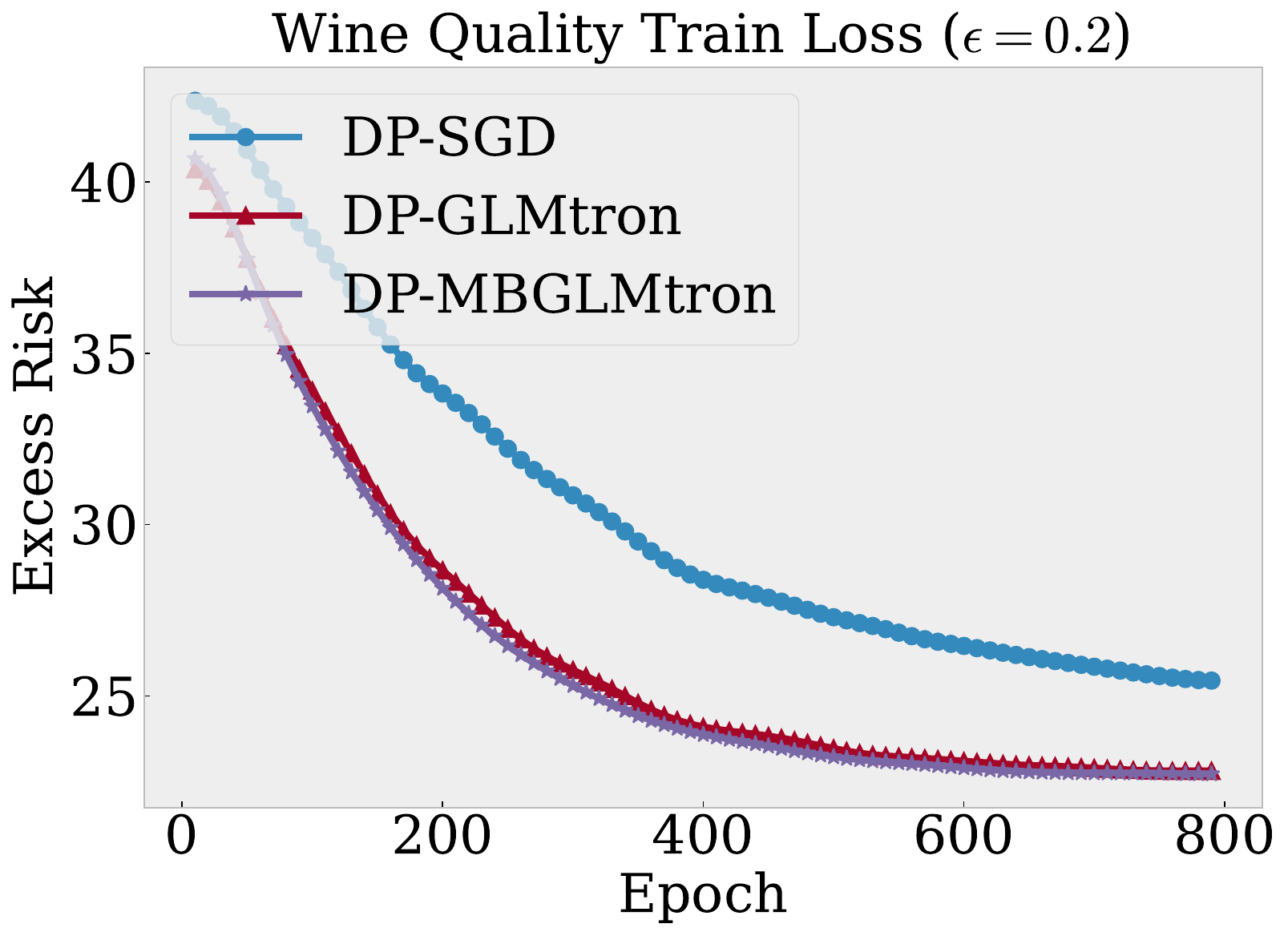}
        \caption{}
        \label{fig:blogtrain02}
    \end{subfigure}%
        \begin{subfigure}{0.32\textwidth}
        \centering
        \includegraphics[width=\linewidth]{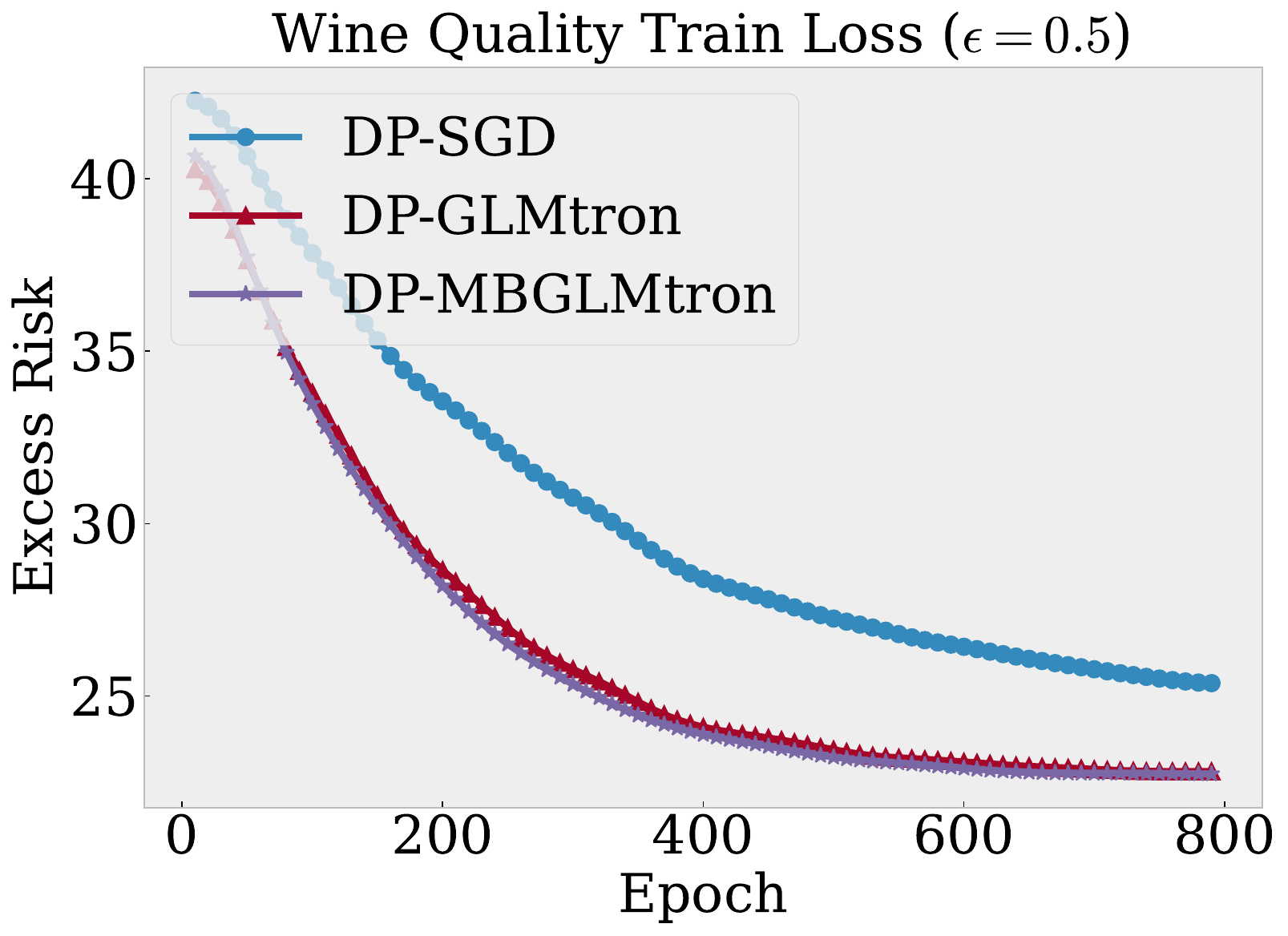}
        \caption{}
        \label{fig:blogtrain05}
    \end{subfigure}%
    
\caption{Training loss over epochs for DP-SGD, DP-GLMtron, and DP-MBGLMtron on three regression datasets: California Housing, Gas Turbine, and Wine Quality, under varying privacy budgets ($\varepsilon = 0.05, 0.2, 0.5$)}
\label{fig:regression_train}
\end{figure*}
 
\section{Experiments}
  
In this section, we present experimental results to validate our theoretical findings. Due to space constraints, the detailed experimental setup and implementation details are provided in \cref{sec:appendix_exp}.

\textbf{Datasets and Models.}
We conducted experiments using three regression datasets: California Housing \citep{california_housing_1997}, Gas Turbine CO and NOx Emission DataSet \citep{gas_turbine_co_and_nox_emission_data_set_551}, and Wine Quality \citep{wine_quality_186}. The information of three datasets used in our experiments is summarized in \cref{tab:data}. For each dataset, the data was randomly split into an 80\% training set and a 20\% test set. All numeric attributes were standardized to have a mean of zero and a standard deviation of one. The target variables were normalized by dividing them by the maximum absolute value of the target variable across the entire dataset. The model used for the experiments was ReLU regression, and evaluations were performed under three different privacy budgets with $\delta=\frac{1}{N^{1.1}}$, $\varepsilon = \{0.05,0.2,0.5\}$. See Appendix \ref{sec:appendix_exp} for more details. 

\begin{table}[h]
\label{tab:data}
\centering
\caption{Summary of Dataset Statistics.}
\resizebox{.4\textwidth}{!}{\begin{tabular}{lcc}
\hline
\textbf{Dataset} & \textbf{ Samples} & \textbf{ Attributes} \\
\hline
\hline
California Housing       & 20640  & 8  \\
Gas Turbine CO and NOx Emission & 36733  & 9  \\
Wine Quality & 4898 & 11 \\
\hline
\end{tabular}}
\end{table}

\textbf{Implementation Details.}
We implemented DP-SGD, DP-GLMtron, and DP-MBGLMtron for regression tasks, tuning hyperparameters to ensure fair comparisons. Specifically, we set the learning rate to 0.01 for DP-SGD and DP-MBGLMtron, while DP-GLMtron used a higher learning rate of 0.05 to account for its single-pass training strategy. Each model was trained for 500 epochs to allow sufficient training progress, with DP-MBGLMtron utilizing a minibatch size of 32.
To ensure the robustness of our findings, every experiment was repeated five times, and the average performance was reported along with standard deviations where applicable. The experiments were conducted on an NVIDIA A6000 GPU. Throughout the training, we monitored the training loss, validation loss, and gradient norms to track convergence and model stability under varying privacy constraints.

\textbf{Experiment Results.} We evaluated the implemented algorithms using two criteria: training loss and test loss, both measured against the number of training epochs. We report the training loss here, with additional experimental results provided in \cref{sec:appendix_exp}.
From Fig.~\ref{fig:regression_train}, we can see 
across all datasets and privacy budgets, DP-MBGLMtron and DP-GLMtron consistently outperform DP-SGD, achieving lower excess risk and faster convergence. These results indicate that the minibatch approach in DP-MBGLMtron is particularly effective under strict privacy constraints, improving both stability and performance in differential privacy settings. The minibatch strategy in DP-MBGLMtron allows for more frequent gradient updates, which helps mitigate the negative effects of privacy-induced noise and stabilize training. In contrast, DP-SGD struggles with convergence, particularly at smaller privacy budgets (e.g., $\varepsilon=0.05$). In Figures \ref{fig:catrain005}-\ref{fig:gastrain05}, the excess risks remain high or do not decrease as effectively as with DP-MBGLMtron and DP-GLMtron, highlighting DP-SGD's limitations in maintaining performance in the ReLU regression model.

\section{Conclusion}
In this work, we revisited differentially private learning in the ReLU regression model under standard assumptions of i.i.d. data from $O(1)$-sub-Gaussian distributions, presenting nearly optimal guarantees for excess population risk.
We introduced and analyzed two algorithms: DP-GLMtron and DP-MBGLMtron. DP-GLMtron leverages adaptive gradient clipping derived from additional public data, achieving an excess risk upper bound of $\Tilde{O}(\frac{d^2}{N^2 \varepsilon^2})$. To mitigate its limitations regarding privacy budgets and public data dependence, we proposed DP-MBGLMtron, which uses mini-batching to eliminate the need for public data and supports larger privacy budgets without compromising performance.
Additionally, we established a matching lower bound through a tracing attack, confirming the tightness of our theoretical results. Empirical evaluations on regression tasks further validated our theoretical insights.



\begin{acknowledgements} 

    Di Wang is supported in part by the  funding BAS/1/1689-01-01, URF/1/4663-01-01,  REI/1/5232-01-01,  REI/1/5332-01-01,  and URF/1/5508-01-01  from KAUST, and funding from KAUST - Center of Excellence for Generative AI, under award number 5940.
\end{acknowledgements}

\bibliography{uai25}

\begin{thebibliography}{65}
\providecommand{\natexlab}[1]{#1}
\providecommand{\url}[1]{\texttt{#1}}
\expandafter\ifx\csname urlstyle\endcsname\relax
  \providecommand{\doi}[1]{doi: #1}\else
  \providecommand{\doi}{doi: \begingroup \urlstyle{rm}\Url}\fi

\bibitem[gas(2019)]{gas_turbine_co_and_nox_emission_data_set_551}
{Gas Turbine CO and NOx Emission Data Set}.
\newblock UCI Machine Learning Repository, 2019.
\newblock {DOI}: https://doi.org/10.24432/C5WC95.

\bibitem[Agarwal et~al.(2017)Agarwal, Allen-Zhu, Bullins, Hazan, and
  Ma]{agarwal2017finding}
Naman Agarwal, Zeyuan Allen-Zhu, Brian Bullins, Elad Hazan, and Tengyu Ma.
\newblock Finding approximate local minima faster than gradient descent.
\newblock In \emph{Proceedings of the 49th Annual ACM SIGACT Symposium on
  Theory of Computing}, pages 1195--1199, 2017.

\bibitem[Amin et~al.(2019)Amin, Kulesza, Munoz, and
  Vassilvtiskii]{amin2019bounding}
Kareem Amin, Alex Kulesza, Andres Munoz, and Sergei Vassilvtiskii.
\newblock Bounding user contributions: A bias-variance trade-off in
  differential privacy.
\newblock In \emph{International Conference on Machine Learning}, pages
  263--271. PMLR, 2019.

\bibitem[Andrew et~al.(2021)Andrew, Thakkar, McMahan, and
  Ramaswamy]{andrew2021differentially}
Galen Andrew, Om~Thakkar, Brendan McMahan, and Swaroop Ramaswamy.
\newblock Differentially private learning with adaptive clipping.
\newblock \emph{Advances in Neural Information Processing Systems},
  34:\penalty0 17455--17466, 2021.

\bibitem[Asi et~al.(2021)Asi, Duchi, Fallah, Javidbakht, and
  Talwar]{asi2021private}
Hilal Asi, John Duchi, Alireza Fallah, Omid Javidbakht, and Kunal Talwar.
\newblock Private adaptive gradient methods for convex optimization.
\newblock In \emph{International Conference on Machine Learning}, pages
  383--392. PMLR, 2021.

\bibitem[Bassily et~al.(2014)Bassily, Smith, and Thakurta]{bassily2014private}
Raef Bassily, Adam Smith, and Abhradeep Thakurta.
\newblock Private empirical risk minimization: Efficient algorithms and tight
  error bounds.
\newblock In \emph{2014 IEEE 55th annual symposium on foundations of computer
  science}, pages 464--473. IEEE, 2014.

\bibitem[Bassily et~al.(2019)Bassily, Feldman, Talwar, and
  Guha~Thakurta]{bassily2019private}
Raef Bassily, Vitaly Feldman, Kunal Talwar, and Abhradeep Guha~Thakurta.
\newblock Private stochastic convex optimization with optimal rates.
\newblock \emph{Advances in neural information processing systems}, 32, 2019.

\bibitem[Bassily et~al.(2021{\natexlab{a}})Bassily, Guzm{\'a}n, and
  Menart]{bassily2021differentially}
Raef Bassily, Crist{\'o}bal Guzm{\'a}n, and Michael Menart.
\newblock Differentially private stochastic optimization: New results in convex
  and non-convex settings.
\newblock \emph{Advances in Neural Information Processing Systems},
  34:\penalty0 9317--9329, 2021{\natexlab{a}}.

\bibitem[Bassily et~al.(2021{\natexlab{b}})Bassily, Guzm{\'a}n, and
  Nandi]{bassily2021non}
Raef Bassily, Crist{\'o}bal Guzm{\'a}n, and Anupama Nandi.
\newblock Non-euclidean differentially private stochastic convex optimization.
\newblock In \emph{Conference on Learning Theory}, pages 474--499. PMLR,
  2021{\natexlab{b}}.

\bibitem[Bun and Steinke(2016)]{bun2016concentrated}
Mark Bun and Thomas Steinke.
\newblock Concentrated differential privacy: Simplifications, extensions, and
  lower bounds.
\newblock In \emph{Theory of Cryptography Conference}, pages 635--658.
  Springer, 2016.

\bibitem[Cai et~al.(2021)Cai, Wang, and Zhang]{cai2021cost}
T~Tony Cai, Yichen Wang, and Linjun Zhang.
\newblock The cost of privacy: Optimal rates of convergence for parameter
  estimation with differential privacy.
\newblock \emph{The Annals of Statistics}, 49\penalty0 (5):\penalty0
  2825--2850, 2021.

\bibitem[Chaudhuri et~al.(2011)Chaudhuri, Monteleoni, and
  Sarwate]{chaudhuri2011differentially}
Kamalika Chaudhuri, Claire Monteleoni, and Anand~D Sarwate.
\newblock Differentially private empirical risk minimization.
\newblock \emph{Journal of Machine Learning Research}, 12\penalty0 (3), 2011.

\bibitem[Cortez et~al.(2009)Cortez, Cerdeira, Almeida, Matos, and
  Reis]{wine_quality_186}
Paulo Cortez, A.~Cerdeira, F.~Almeida, T.~Matos, and J.~Reis.
\newblock {Wine Quality}.
\newblock UCI Machine Learning Repository, 2009.
\newblock {DOI}: https://doi.org/10.24432/C56S3T.

\bibitem[Ding et~al.()Ding, Lei, Zhu, Wang, Wang, and Xu]{dingrevisiting}
Meng Ding, Mingxi Lei, Liyang Zhu, Shaowei Wang, Di~Wang, and Jinhui Xu.
\newblock Revisiting differentially private relu regression.
\newblock In \emph{The Thirty-eighth Annual Conference on Neural Information
  Processing Systems}.

\bibitem[Ding et~al.(2024)Ding, Ji, Wang, and Xu]{ding2024understanding}
Meng Ding, Kaiyi Ji, Di~Wang, and Jinhui Xu.
\newblock Understanding forgetting in continual learning with linear
  regression.
\newblock \emph{arXiv preprint arXiv:2405.17583}, 2024.

\bibitem[Du et~al.(2018)Du, Zhai, Poczos, and Singh]{du2018gradient}
Simon~S Du, Xiyu Zhai, Barnabas Poczos, and Aarti Singh.
\newblock Gradient descent provably optimizes over-parameterized neural
  networks.
\newblock \emph{arXiv preprint arXiv:1810.02054}, 2018.

\bibitem[Dwork et~al.(2006)Dwork, McSherry, Nissim, and
  Smith]{dwork2006calibrating}
Cynthia Dwork, Frank McSherry, Kobbi Nissim, and Adam Smith.
\newblock Calibrating noise to sensitivity in private data analysis.
\newblock In \emph{Theory of Cryptography: Third Theory of Cryptography
  Conference, TCC 2006, New York, NY, USA, March 4-7, 2006. Proceedings 3},
  pages 265--284. Springer, 2006.

\bibitem[Feldman et~al.(2020)Feldman, Koren, and Talwar]{feldman2020private}
Vitaly Feldman, Tomer Koren, and Kunal Talwar.
\newblock Private stochastic convex optimization: optimal rates in linear time.
\newblock In \emph{Proceedings of the 52nd Annual ACM SIGACT Symposium on
  Theory of Computing}, pages 439--449, 2020.

\bibitem[Feldman et~al.(2022)Feldman, McMillan, and Talwar]{feldman2022hiding}
Vitaly Feldman, Audra McMillan, and Kunal Talwar.
\newblock Hiding among the clones: A simple and nearly optimal analysis of
  privacy amplification by shuffling.
\newblock In \emph{2021 IEEE 62nd Annual Symposium on Foundations of Computer
  Science (FOCS)}, pages 954--964. IEEE, 2022.

\bibitem[Frei et~al.(2020)Frei, Cao, and Gu]{frei2020agnostic}
Spencer Frei, Yuan Cao, and Quanquan Gu.
\newblock Agnostic learning of a single neuron with gradient descent.
\newblock \emph{Advances in Neural Information Processing Systems},
  33:\penalty0 5417--5428, 2020.

\bibitem[Goel and Klivans(2019)]{goel2019learning}
Surbhi Goel and Adam~R Klivans.
\newblock Learning neural networks with two nonlinear layers in polynomial
  time.
\newblock In \emph{Conference on Learning Theory}, pages 1470--1499. PMLR,
  2019.

\bibitem[Hu et~al.(2022)Hu, Ni, Xiao, and Wang]{hu2022high}
Lijie Hu, Shuo Ni, Hanshen Xiao, and Di~Wang.
\newblock High dimensional differentially private stochastic optimization with
  heavy-tailed data.
\newblock In \emph{Proceedings of the 41st ACM SIGMOD-SIGACT-SIGAI Symposium on
  Principles of Database Systems}, pages 227--236, 2022.

\bibitem[Iyengar et~al.(2019)Iyengar, Near, Song, Thakkar, Thakurta, and
  Wang]{iyengar2019towards}
Roger Iyengar, Joseph~P Near, Dawn Song, Om~Thakkar, Abhradeep Thakurta, and
  Lun Wang.
\newblock Towards practical differentially private convex optimization.
\newblock In \emph{2019 IEEE symposium on security and privacy (SP)}, pages
  299--316. IEEE, 2019.

\bibitem[Jain and Thakurta(2014)]{jain2014near}
Prateek Jain and Abhradeep~Guha Thakurta.
\newblock (near) dimension independent risk bounds for differentially private
  learning.
\newblock In \emph{International Conference on Machine Learning}, pages
  476--484. PMLR, 2014.

\bibitem[Jain et~al.(2012)Jain, Kothari, and Thakurta]{jain2012differentially}
Prateek Jain, Pravesh Kothari, and Abhradeep Thakurta.
\newblock Differentially private online learning.
\newblock In \emph{Conference on Learning Theory}, pages 24--1. JMLR Workshop
  and Conference Proceedings, 2012.

\bibitem[Jain et~al.(2018)Jain, Kakade, Kidambi, Netrapalli, and
  Sidford]{jain2018parallelizing}
Prateek Jain, Sham Kakade, Rahul Kidambi, Praneeth Netrapalli, and Aaron
  Sidford.
\newblock Parallelizing stochastic gradient descent for least squares
  regression: mini-batching, averaging, and model misspecification.
\newblock \emph{Journal of machine learning research}, 18, 2018.

\bibitem[Jayaraman et~al.(2018)Jayaraman, Wang, Evans, and
  Gu]{jayaraman2018distributed}
Bargav Jayaraman, Lingxiao Wang, David Evans, and Quanquan Gu.
\newblock Distributed learning without distress: Privacy-preserving empirical
  risk minimization.
\newblock \emph{Advances in Neural Information Processing Systems}, 31, 2018.

\bibitem[Kakade et~al.(2011)Kakade, Kanade, Shamir, and
  Kalai]{kakade2011efficient}
Sham~M Kakade, Varun Kanade, Ohad Shamir, and Adam Kalai.
\newblock Efficient learning of generalized linear and single index models with
  isotonic regression.
\newblock \emph{Advances in Neural Information Processing Systems}, 24, 2011.

\bibitem[Karwa and Vadhan(2017)]{karwa2017finite}
Vishesh Karwa and Salil Vadhan.
\newblock Finite sample differentially private confidence intervals.
\newblock \emph{arXiv preprint arXiv:1711.03908}, 2017.

\bibitem[Kifer et~al.(2012)Kifer, Smith, and Thakurta]{kifer2012private}
Daniel Kifer, Adam Smith, and Abhradeep Thakurta.
\newblock Private convex empirical risk minimization and high-dimensional
  regression.
\newblock In \emph{Conference on Learning Theory}, pages 25--1. JMLR Workshop
  and Conference Proceedings, 2012.

\bibitem[Kulkarni et~al.(2021)Kulkarni, Lee, and Liu]{kulkarni2021private}
Janardhan Kulkarni, Yin~Tat Lee, and Daogao Liu.
\newblock Private non-smooth empirical risk minimization and stochastic convex
  optimization in subquadratic steps.
\newblock \emph{arXiv preprint arXiv:2103.15352}, 2021.

\bibitem[Liu et~al.(2022)Liu, Kong, Jain, and Oh]{liu2022dp}
Xiyang Liu, Weihao Kong, Prateek Jain, and Sewoong Oh.
\newblock Dp-pca: Statistically optimal and differentially private pca.
\newblock \emph{Advances in Neural Information Processing Systems},
  35:\penalty0 29929--29943, 2022.

\bibitem[Liu et~al.(2023{\natexlab{a}})Liu, Jain, Kong, Oh, and
  Suggala]{liu2023label}
Xiyang Liu, Prateek Jain, Weihao Kong, Sewoong Oh, and Arun Suggala.
\newblock Label robust and differentially private linear regression:
  Computational and statistical efficiency.
\newblock In \emph{Thirty-seventh Conference on Neural Information Processing
  Systems}, 2023{\natexlab{a}}.

\bibitem[Liu et~al.(2023{\natexlab{b}})Liu, Jain, Kong, Oh, and
  Suggala]{liu2023near}
Xiyang Liu, Prateek Jain, Weihao Kong, Sewoong Oh, and Arun~Sai Suggala.
\newblock Near optimal private and robust linear regression.
\newblock \emph{arXiv preprint arXiv:2301.13273}, 2023{\natexlab{b}}.

\bibitem[McMahan et~al.(2017)McMahan, Ramage, Talwar, and
  Zhang]{mcmahan2017learning}
H~Brendan McMahan, Daniel Ramage, Kunal Talwar, and Li~Zhang.
\newblock Learning differentially private recurrent language models.
\newblock \emph{arXiv preprint arXiv:1710.06963}, 2017.

\bibitem[Pace and Barry(1997)]{california_housing_1997}
R.~Kelley Pace and Ronald Barry.
\newblock California housing data.
\newblock StatLib Repository, 1997.
\newblock URL \url{http://www.spatial-statistics.com}.
\newblock Data obtained from the 1990 U.S. Census. The manuscript describing
  the data can be found at \url{http://www.spatial-statistics.com}. DOI:
  https://doi.org/10.1016/S0167-7152(97)00107-X.

\bibitem[Rudelson and Vershynin(2013)]{rudelson2013hanson}
Mark Rudelson and Roman Vershynin.
\newblock Hanson-wright inequality and sub-gaussian concentration.
\newblock 2013.

\bibitem[Shen et~al.(2023)Shen, Wang, Xiang, Ying, and
  Wang]{shen2023differentially}
Hanpu Shen, Cheng-Long Wang, Zihang Xiang, Yiming Ying, and Di~Wang.
\newblock Differentially private non-convex learning for multi-layer neural
  networks.
\newblock \emph{arXiv preprint arXiv:2310.08425}, 2023.

\bibitem[Song et~al.(2020)Song, Thakkar, and Thakurta]{song2020characterizing}
Shuang Song, Om~Thakkar, and Abhradeep Thakurta.
\newblock Characterizing private clipped gradient descent on convex generalized
  linear problems.
\newblock \emph{arXiv preprint arXiv:2006.06783}, 2020.

\bibitem[Song et~al.(2021)Song, Steinke, Thakkar, and
  Thakurta]{song2021evading}
Shuang Song, Thomas Steinke, Om~Thakkar, and Abhradeep Thakurta.
\newblock Evading the curse of dimensionality in unconstrained private glms.
\newblock In \emph{International Conference on Artificial Intelligence and
  Statistics}, pages 2638--2646. PMLR, 2021.

\bibitem[Su and Wang(2021)]{su2021faster}
Jinyan Su and Di~Wang.
\newblock Faster rates of differentially private stochastic convex
  optimization.
\newblock \emph{arXiv preprint arXiv}, 2108, 2021.

\bibitem[Su et~al.(2023)Su, Zhao, and Wang]{su2023differentially}
Jinyan Su, Changhong Zhao, and Di~Wang.
\newblock Differentially private stochastic convex optimization in
  (non)-euclidean space revisited.
\newblock In \emph{Uncertainty in Artificial Intelligence}, pages 2026--2035.
  PMLR, 2023.

\bibitem[Su et~al.(2024)Su, Hu, and Wang]{su2024faster}
Jinyan Su, Lijie Hu, and Di~Wang.
\newblock Faster rates of differentially private stochastic convex
  optimization.
\newblock \emph{Journal of Machine Learning Research}, 25\penalty0
  (114):\penalty0 1--41, 2024.

\bibitem[Talwar et~al.(2014)Talwar, Thakurta, and Zhang]{talwar2014private}
Kunal Talwar, Abhradeep Thakurta, and Li~Zhang.
\newblock Private empirical risk minimization beyond the worst case: The effect
  of the constraint set geometry.
\newblock \emph{arXiv preprint arXiv:1411.5417}, 2014.

\bibitem[Tao et~al.(2025)Tao, Zhang, Yu, Cheng, Dressler, and
  Wang]{tao2025second}
Youming Tao, Zuyuan Zhang, Dongxiao Yu, Xiuzhen Cheng, Falko Dressler, and
  Di~Wang.
\newblock Second-order convergence in private stochastic non-convex
  optimization.
\newblock \emph{arXiv preprint arXiv:2505.15647}, 2025.

\bibitem[Varshney et~al.(2022)Varshney, Thakurta, and Jain]{varshney2022nearly}
Prateek Varshney, Abhradeep Thakurta, and Prateek Jain.
\newblock (nearly) optimal private linear regression via adaptive clipping.
\newblock \emph{arXiv preprint arXiv:2207.04686}, 2022.

\bibitem[Wang and Xu(2019{\natexlab{a}})]{wang2019differentially}
Di~Wang and Jinhui Xu.
\newblock Differentially private empirical risk minimization with smooth
  non-convex loss functions: A non-stationary view.
\newblock In \emph{Proceedings of the AAAI Conference on Artificial
  Intelligence}, volume~33, pages 1182--1189, 2019{\natexlab{a}}.

\bibitem[Wang and Xu(2019{\natexlab{b}})]{wang2019sparse}
Di~Wang and Jinhui Xu.
\newblock On sparse linear regression in the local differential privacy model.
\newblock In \emph{International Conference on Machine Learning}, pages
  6628--6637. PMLR, 2019{\natexlab{b}}.

\bibitem[Wang and Xu(2021)]{wang2021escaping}
Di~Wang and Jinhui Xu.
\newblock Escaping saddle points of empirical risk privately and scalably via
  dp-trust region method.
\newblock In \emph{Machine Learning and Knowledge Discovery in Databases:
  European Conference, ECML PKDD 2020, Ghent, Belgium, September 14--18, 2020,
  Proceedings, Part III}, pages 90--106. Springer, 2021.

\bibitem[Wang and Xu(2024)]{wang2024gradient}
Di~Wang and Jinhui Xu.
\newblock Gradient complexity and non-stationary views of differentially
  private empirical risk minimization.
\newblock \emph{Theoretical Computer Science}, 982:\penalty0 114259, 2024.

\bibitem[Wang et~al.(2017)Wang, Ye, and Xu]{wang2017differentially}
Di~Wang, Minwei Ye, and Jinhui Xu.
\newblock Differentially private empirical risk minimization revisited: Faster
  and more general.
\newblock \emph{Advances in Neural Information Processing Systems}, 30, 2017.

\bibitem[Wang et~al.(2019)Wang, Chen, and Xu]{wang2019differentially1}
Di~Wang, Changyou Chen, and Jinhui Xu.
\newblock Differentially private empirical risk minimization with non-convex
  loss functions.
\newblock In \emph{International Conference on Machine Learning}, pages
  6526--6535. PMLR, 2019.

\bibitem[Wang et~al.(2023{\natexlab{a}})Wang, Ding, Hu, Xie, Pan, and
  Xu]{wang2023finite}
Di~Wang, Jiahao Ding, Lijie Hu, Zejun Xie, Miao Pan, and Jinhui Xu.
\newblock Finite sample guarantees of differentially private expectation
  maximization algorithm.
\newblock In \emph{ECAI 2023}, pages 2435--2442. IOS Press, 2023{\natexlab{a}}.

\bibitem[Wang et~al.(2023{\natexlab{b}})Wang, Hu, Zhang, Gaboardi, and
  Xu]{wang2023generalized}
Di~Wang, Lijie Hu, Huanyu Zhang, Marco Gaboardi, and Jinhui Xu.
\newblock Generalized linear models in non-interactive local differential
  privacy with public data.
\newblock \emph{Journal of Machine Learning Research}, 24\penalty0
  (132):\penalty0 1--57, 2023{\natexlab{b}}.

\bibitem[Wang et~al.(2023{\natexlab{c}})Wang, Jayaraman, Evans, and
  Gu]{wang2023efficient}
Lingxiao Wang, Bargav Jayaraman, David Evans, and Quanquan Gu.
\newblock Efficient privacy-preserving stochastic nonconvex optimization.
\newblock In \emph{Uncertainty in Artificial Intelligence}, pages 2203--2213.
  PMLR, 2023{\natexlab{c}}.

\bibitem[Wu et~al.(2023)Wu, Zou, Chen, Braverman, Gu, and Kakade]{wu2023finite}
Jingfeng Wu, Difan Zou, Zixiang Chen, Vladimir Braverman, Quanquan Gu, and
  Sham~M Kakade.
\newblock Finite-sample analysis of learning high-dimensional single relu
  neuron.
\newblock 2023.

\bibitem[Wu et~al.(2017)Wu, Li, Kumar, Chaudhuri, Jha, and
  Naughton]{wu2017bolt}
Xi~Wu, Fengan Li, Arun Kumar, Kamalika Chaudhuri, Somesh Jha, and Jeffrey
  Naughton.
\newblock Bolt-on differential privacy for scalable stochastic gradient
  descent-based analytics.
\newblock In \emph{Proceedings of the 2017 ACM International Conference on
  Management of Data}, pages 1307--1322, 2017.

\bibitem[Xiao et~al.(2023)Xiao, Xiang, Wang, and Devadas]{xiao2023theory}
Hanshen Xiao, Zihang Xiang, Di~Wang, and Srinivas Devadas.
\newblock A theory to instruct differentially-private learning via clipping
  bias reduction.
\newblock In \emph{2023 IEEE Symposium on Security and Privacy (SP)}, pages
  2170--2189. IEEE Computer Society, 2023.

\bibitem[Zhang et~al.(2017)Zhang, Zheng, Mou, and Wang]{zhang2017efficient}
Jiaqi Zhang, Kai Zheng, Wenlong Mou, and Liwei Wang.
\newblock Efficient private erm for smooth objectives.
\newblock \emph{arXiv preprint arXiv:1703.09947}, 2017.

\bibitem[Zhang et~al.(2021)Zhang, Ma, Lou, and Xiong]{zhang2021private}
Qiuchen Zhang, Jing Ma, Jian Lou, and Li~Xiong.
\newblock Private stochastic non-convex optimization with improved utility
  rates.
\newblock In \emph{Proceedings of the Thirtieth International Joint Conference
  on Artificial Intelligence}, 2021.

\bibitem[Zhang et~al.(2025)Zhang, Lei, Ding, Xiang, Xu, and
  Wang]{zhang2025improved}
Ruijia Zhang, Mingxi Lei, Meng Ding, Zihang Xiang, Jinhui Xu, and Di~Wang.
\newblock Improved rates of differentially private nonconvex-strongly-concave
  minimax optimization.
\newblock In \emph{Proceedings of the AAAI Conference on Artificial
  Intelligence}, volume~39, pages 22524--22532, 2025.

\bibitem[Zhou et~al.(2020)Zhou, Wu, and Banerjee]{zhou2020bypassing}
Yingxue Zhou, Zhiwei~Steven Wu, and Arindam Banerjee.
\newblock Bypassing the ambient dimension: Private sgd with gradient subspace
  identification.
\newblock \emph{arXiv preprint arXiv:2007.03813}, 2020.

\bibitem[Zhu et~al.(2023)Zhu, Ding, Aggarwal, Xu, and Wang]{zhu2023improved}
Liyang Zhu, Meng Ding, Vaneet Aggarwal, Jinhui Xu, and Di~Wang.
\newblock Improved analysis of sparse linear regression in local differential
  privacy model.
\newblock \emph{arXiv preprint arXiv:2310.07367}, 2023.

\bibitem[Zhu et~al.(2024)Zhu, Manseur, Ding, Liu, Xu, and
  Wang]{zhu2024truthful}
Liyang Zhu, Amina Manseur, Meng Ding, Jinyan Liu, Jinhui Xu, and Di~Wang.
\newblock Truthful high dimensional sparse linear regression.
\newblock \emph{arXiv preprint arXiv:2410.13046}, 2024.

\bibitem[Zou et~al.(2021)Zou, Wu, Braverman, Gu, and Kakade]{zou2021benign}
Difan Zou, Jingfeng Wu, Vladimir Braverman, Quanquan Gu, and Sham Kakade.
\newblock Benign overfitting of constant-stepsize sgd for linear regression.
\newblock In \emph{Conference on Learning Theory}, pages 4633--4635. PMLR,
  2021.

\end{thebibliography}

\newpage
\appendix 
\onecolumn

\title{Nearly Optimal Differentially Private ReLU Regression\\(Supplementary Material)}
\maketitle

\section{Additional Experiment} \label{sec:appendix_exp}

\textbf{Datasets Information.} The information of three datasets used in our experiments is summarized in \cref{tab:data1}.

\begin{table}[h]
\label{tab:data1}
\centering
\caption{Summary of Dataset Statistics.}
\begin{tabular}{lcc}
\hline
\textbf{Dataset} & \textbf{ Samples} & \textbf{ Attributes} \\
\hline
\hline
California Housing       & 20640  & 8  \\
Gas Turbine CO and NOx Emission & 36733  & 9  \\
Wine Quality & 4898 & 11 \\
\hline
\end{tabular}
\end{table}

\textbf{Experimental Results.} Across all datasets and privacy budgets, DP-GLMtron and DP-MBGLMtron consistently outperform DP-SGD. The two methods converge faster and stabilize at test loss, suggesting that they are more effective at maintaining performance while adhering to privacy constraints. The trend holds consistent across varying privacy budgets ($\varepsilon = 0.05, 0.2, 0.5$), further highlighting the robustness of our approaches.

\begin{figure*}[h]
    \centering
    \begin{subfigure}{0.29\textwidth}
        \centering
        \includegraphics[width=\linewidth]{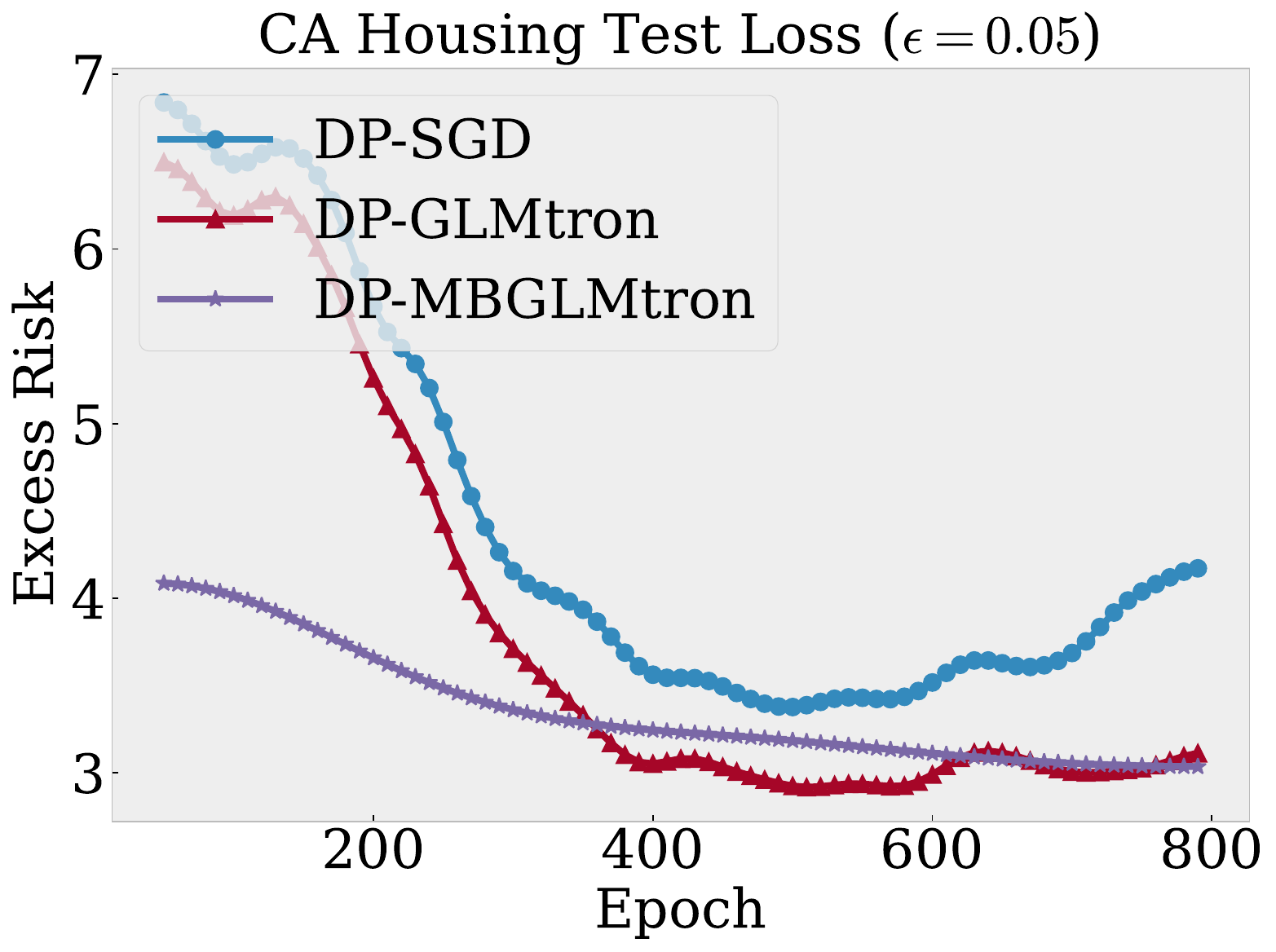}
        \caption{}
        \label{fig:catest005}
    \end{subfigure}%
    \begin{subfigure}{0.3\textwidth}
        \centering
        \includegraphics[width=\linewidth]{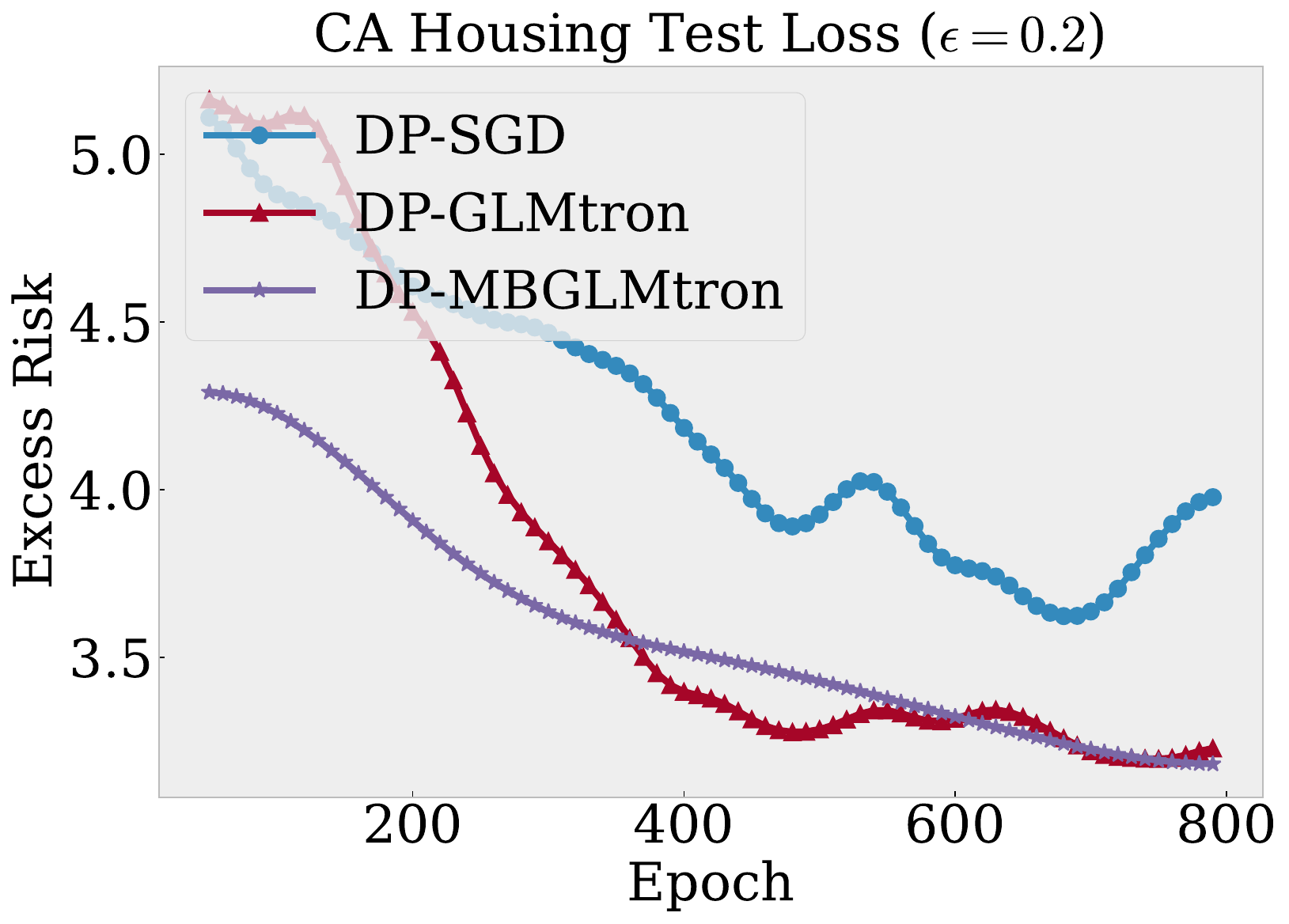}
        \caption{}
        \label{fig:catest02}
    \end{subfigure}
    \begin{subfigure}{0.3\textwidth}
        \centering
        \includegraphics[width=\linewidth]{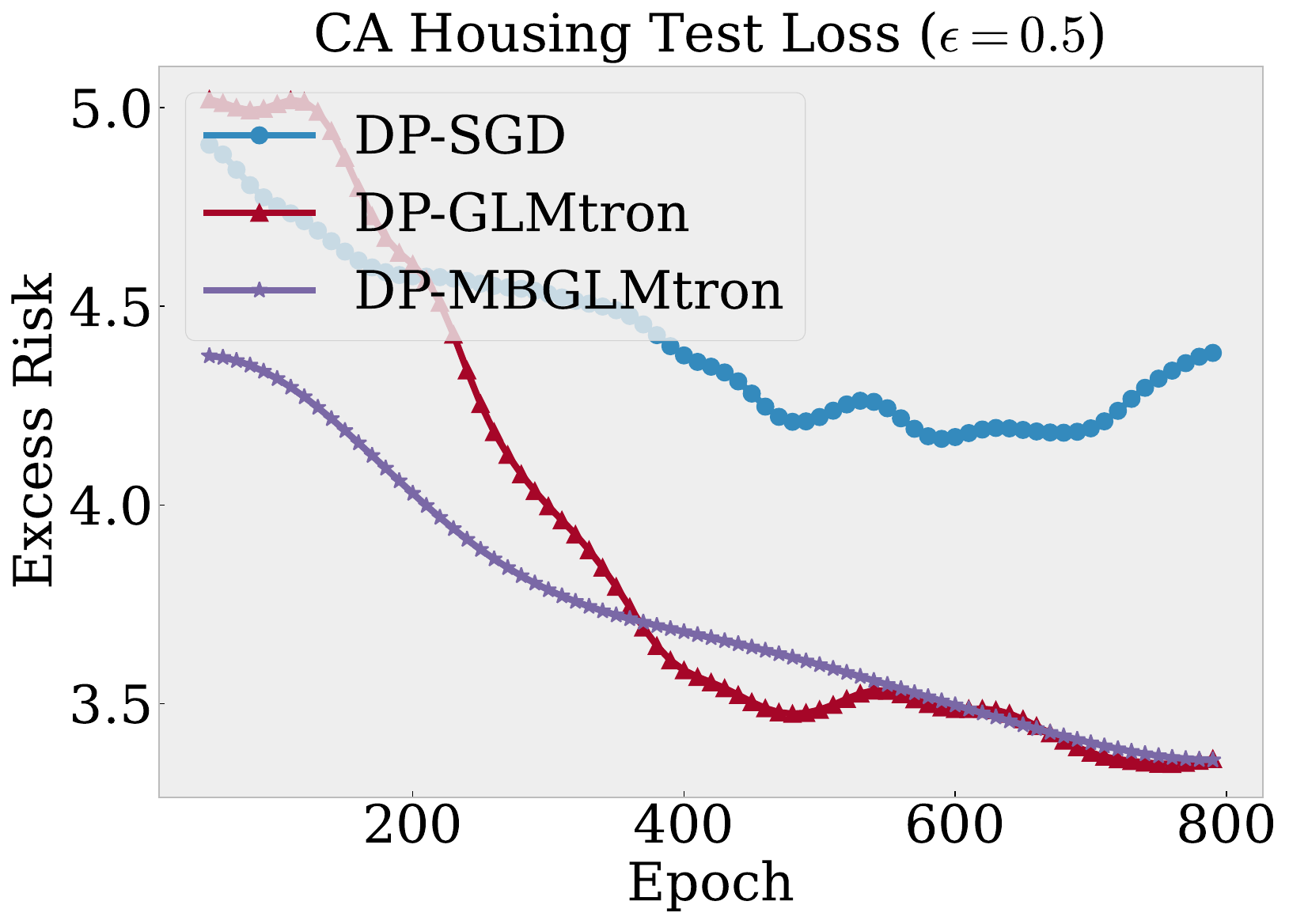}
        \caption{}
        \label{fig:catest05}
    \end{subfigure}%

    \begin{subfigure}{0.3\textwidth}
        \centering
        \includegraphics[width=\linewidth]{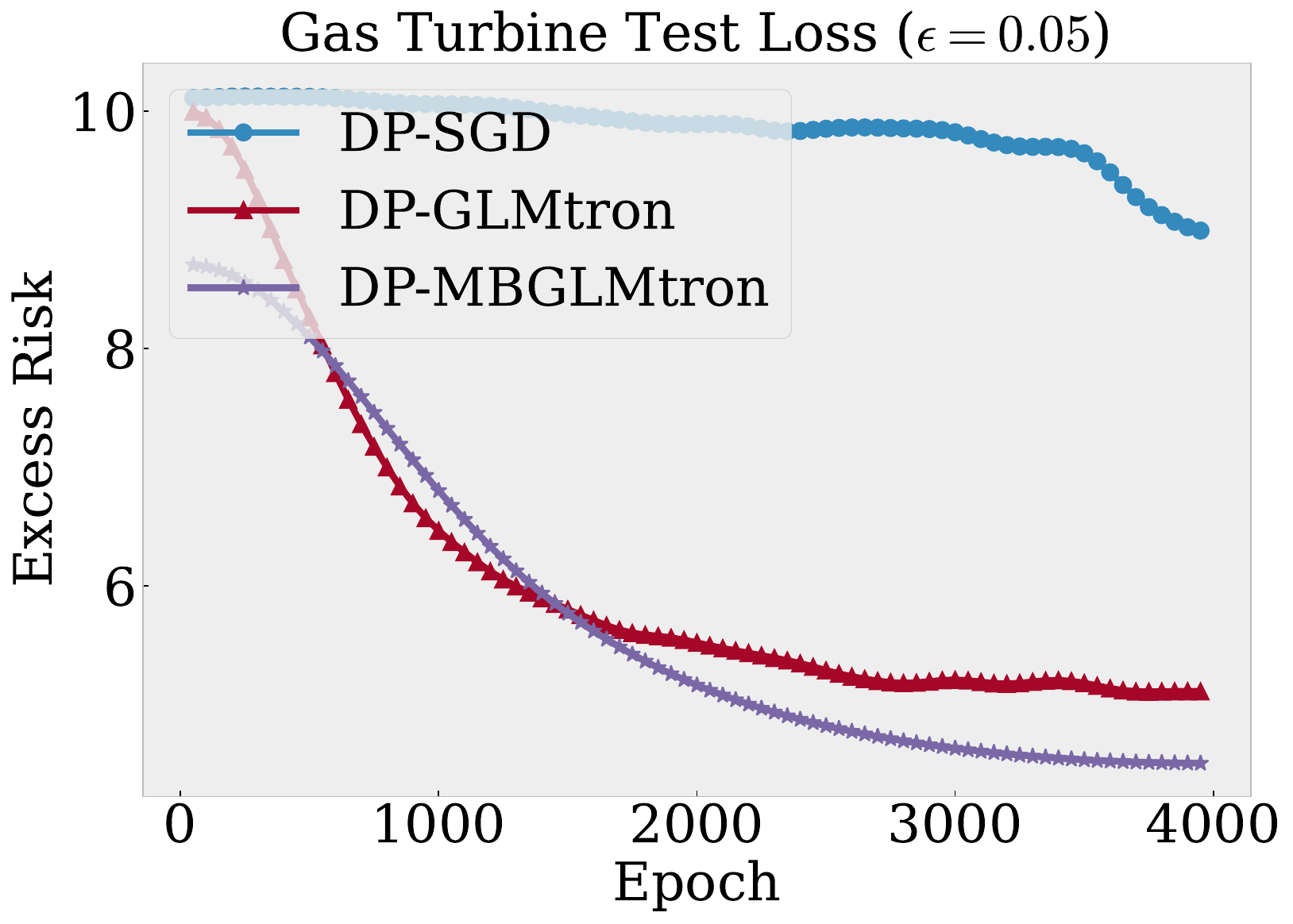}
        \caption{}
        \label{fig:gastest005}
    \end{subfigure}%
    \begin{subfigure}{0.3\textwidth}
        \centering
        \includegraphics[width=\linewidth]{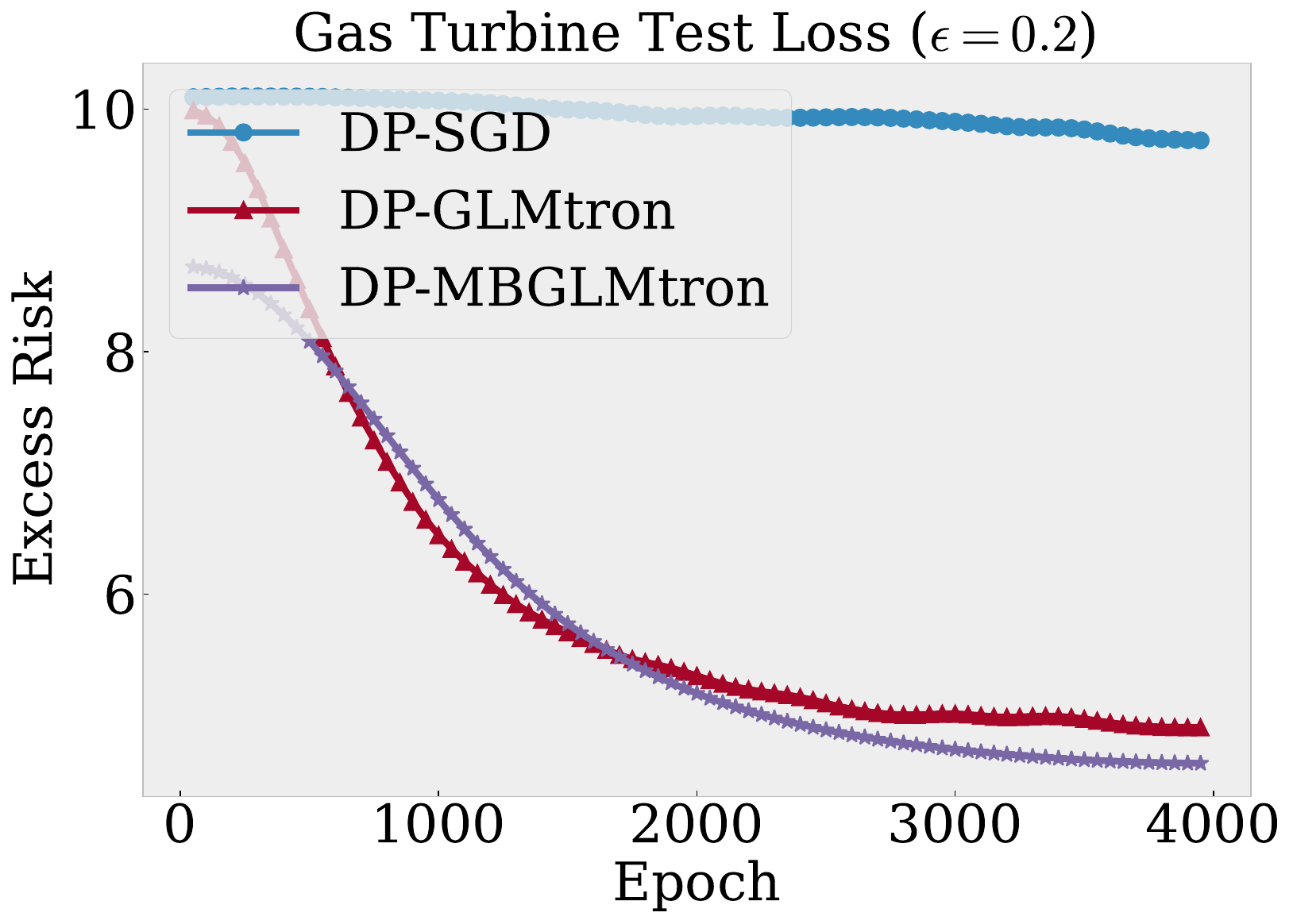}
        \caption{}
        \label{fig:gastest02}
    \end{subfigure}%
    \begin{subfigure}{0.3\textwidth}
        \centering
        \includegraphics[width=\linewidth]{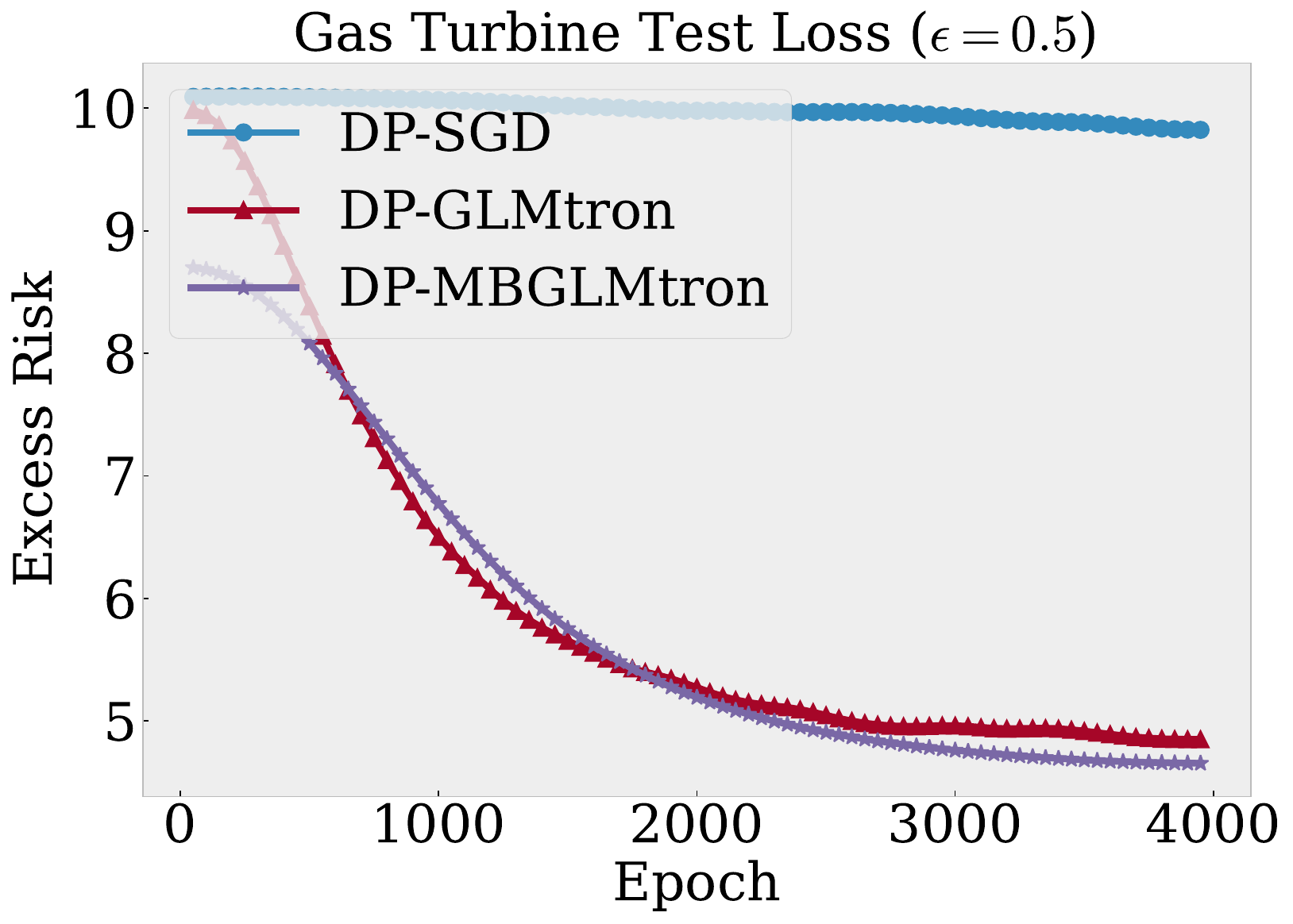}
        \caption{}
        \label{fig:gastest05}
    \end{subfigure}

    \begin{subfigure}{0.3\textwidth}
        \centering
        \includegraphics[width=\linewidth]{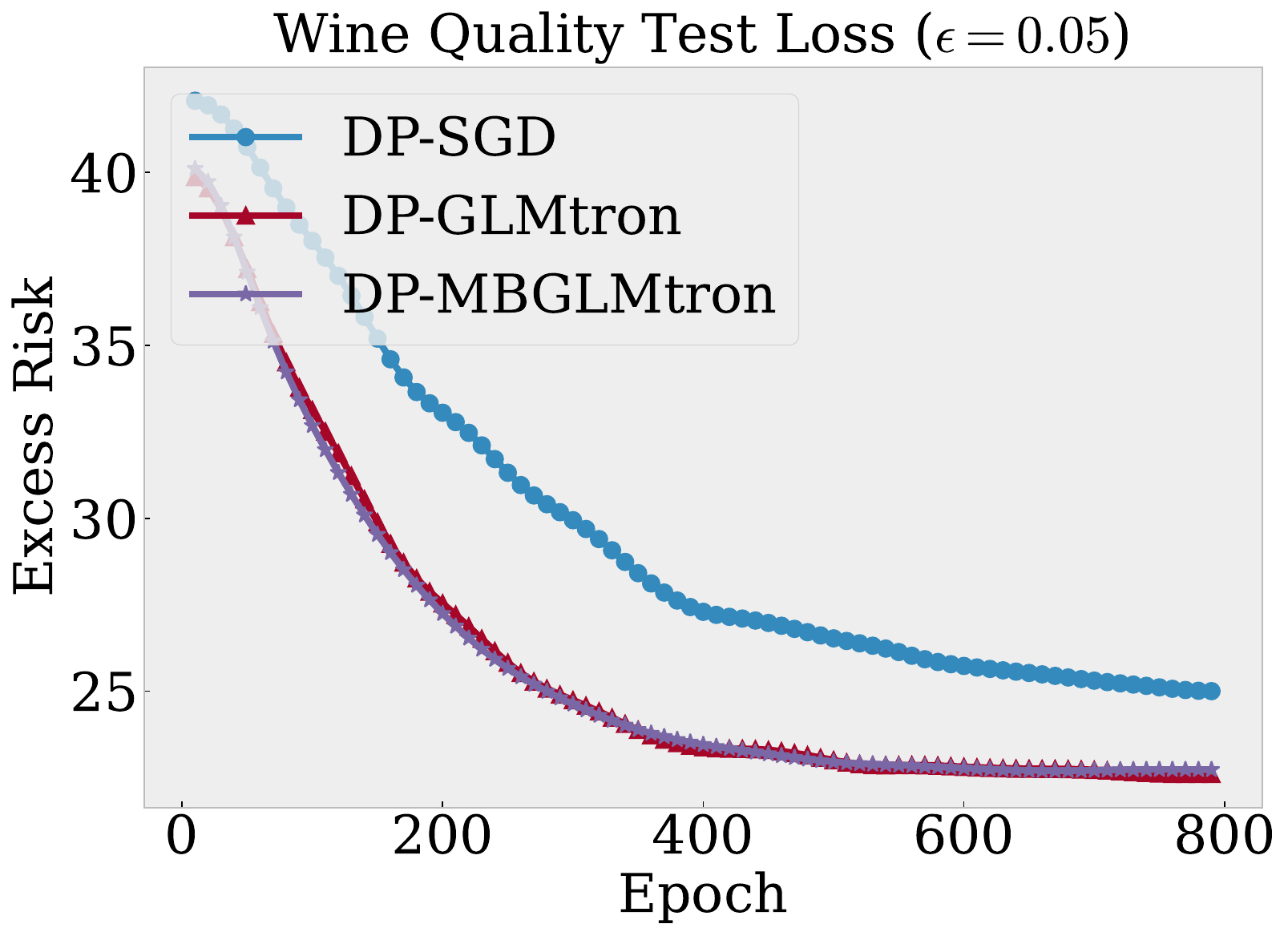}
        \caption{}
        \label{fig:blogtest005}
    \end{subfigure}%
    \begin{subfigure}{0.3\textwidth}
        \centering
        \includegraphics[width=\linewidth]{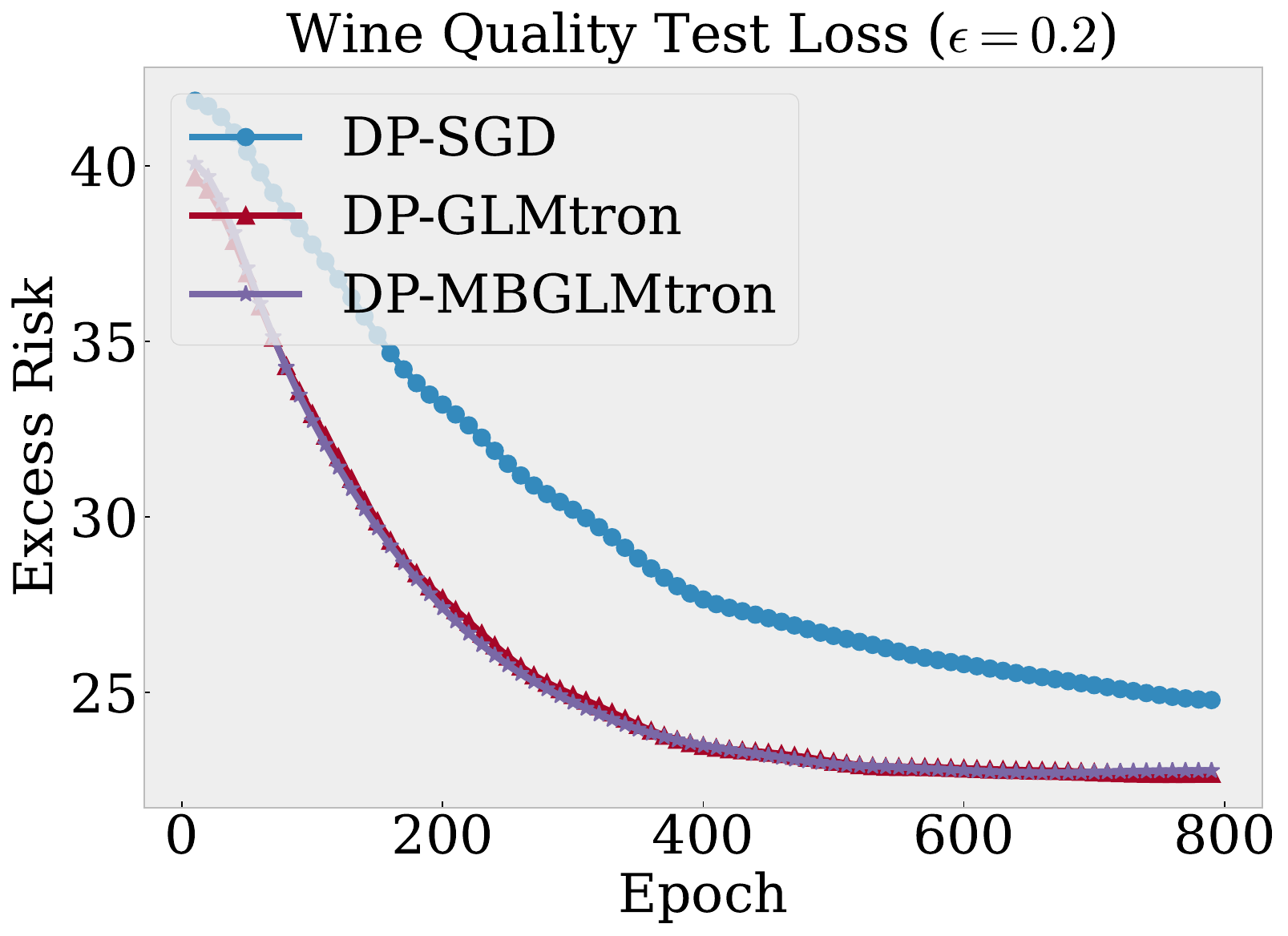}
        \caption{}
        \label{fig:blogtest02}
    \end{subfigure}%
        \begin{subfigure}{0.3\textwidth}
        \centering
        \includegraphics[width=\linewidth]{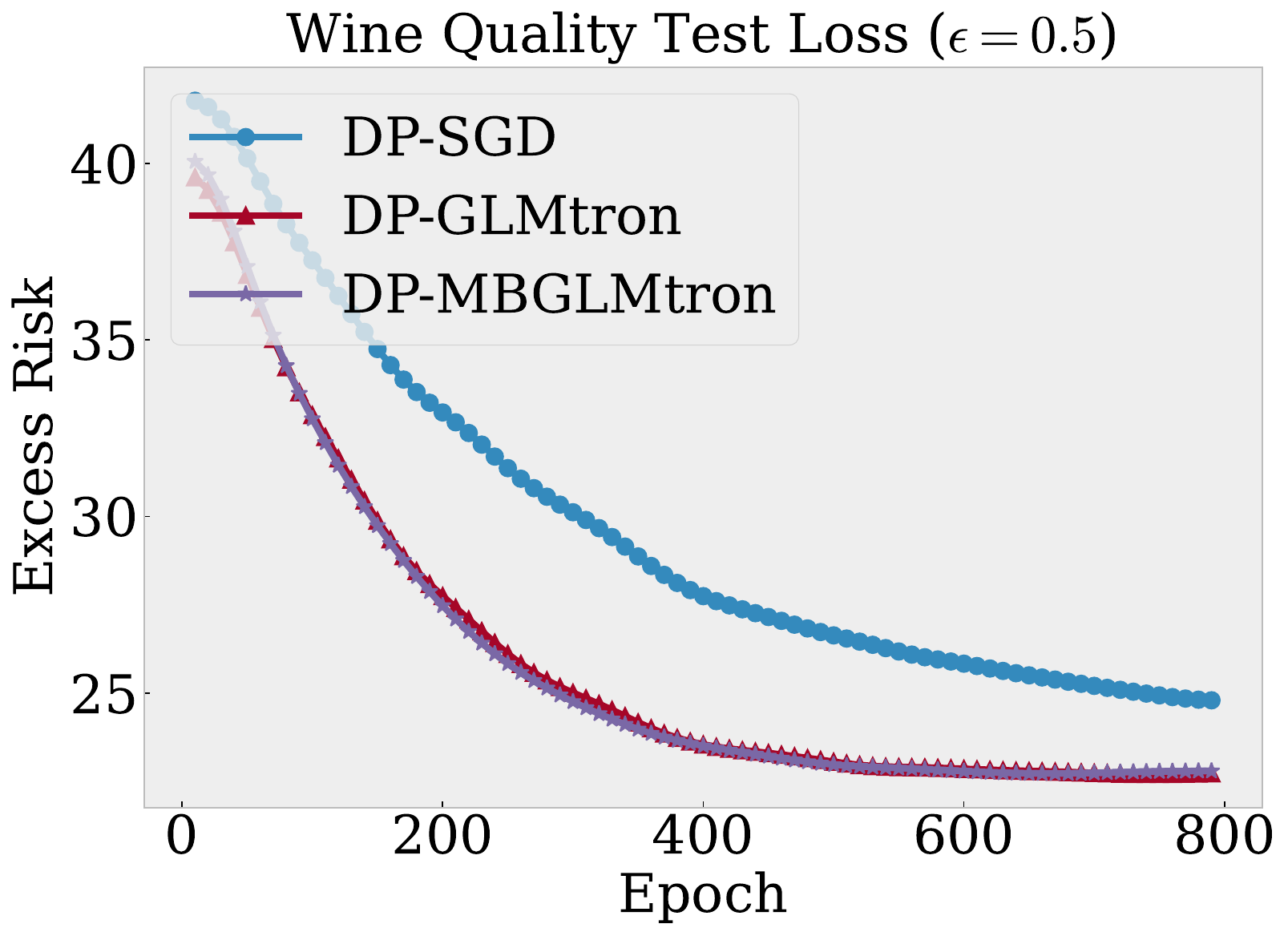}
        \caption{}
        \label{fig:blogtest05}
    \end{subfigure}%
    
\caption{Test loss over epochs for DP-SGD, DP-GLMtron, and DP-MiniBatch GLMtron on three regression datasets: California Housing, Gas Turbine, and Blog Feedback, under varying privacy budgets ($\varepsilon = 0.05, 0.2, 0.5$)}
\label{fig:regression_test}
\end{figure*}

\textbf{Computing Infrastructures.} The information of training configuration used in our experiments is summarized in \cref{tab:hardware_software_config}.

\begin{table}[h]
\centering
\caption{Hardware and Software Configuration.}
\label{tab:hardware_software_config}
\begin{tabular}{lc} 
\hline
\textbf{Components}              & \textbf{Details}         \\
\hline
\hline
Operating System        & Ubuntu 16.04.6                    \\ 
CPU                     & AMD EPYC 7552, 48-Core Processor  \\ 
CPU Memory              & 1.0 TB                            \\ 
GPU&NVIDIA RTX A6000\\
Programming Language    & Python~3.9.12                     \\ 
Deep Learning Framework & Pytorch~1.12.1                    \\
\hline
\end{tabular}
\end{table}

\begin{table}[h]
\centering
\begin{tabular}{llccccc}
\toprule
& \textbf{$\varepsilon$} & \textbf{DP-SGD} & \textbf{DP-GLMtron} & \textbf{DP-MBGLMtron} & \textbf{DP-TAGLMtron} \\
\midrule
\multirow{3}{*}{\textbf{Train}} 
& 0.05 & 9.14 & 8.26 & 5.19 & 7.77 \\
& 0.2  & 6.40 & 5.47 & 4.40 & 4.90 \\
& 0.5  & 5.77 & 4.83 & 4.38 & 4.55 \\
\midrule
\multirow{3}{*}{\textbf{Test}} 
& 0.05 & 9.74 & 8.74 & 5.63 & 8.06 \\
& 0.2  & 7.38 & 6.25 & 4.82 & 5.14 \\
& 0.5  & 6.82 & 5.63 & 4.81 & 5.25 \\
\bottomrule
\end{tabular}
\caption{Excess Risk (lower is better) under different privacy budgets.}
\label{tab:dp_methods_comparison}
\end{table}

To better illustrate our theoretical findings, we have added an additional experiment using synthetic data that satisfies our assumptions. In this setting, we could observe that the excess risk is closely related to the privacy budget in the table~\ref{tab:dp_methods_comparison}.

\section{Additional Definitions}
\begin{definition}[zCDP
\citet{bun2016concentrated}]\label{def:zcdp}
	A randomized algorithm $\mathcal{A}$ is $\rho$-zCDP if for any pair of data sets $D$ and $D^{\prime}$ that differ in one record, we have $D_p(\mathcal{A}(D) \| \mathcal{A}(D^{\prime})) \leq \rho p$ for all $p>1$, where $D_p$ is the Rényi divergence of order $p$.
\end{definition} 

\begin{definition}[Sub-Gaussian random variable]
     A zero-mean random variable $X \in \mathbb{R}$ is said to be sub-Gaussian with variance $\sigma^2(X \sim \operatorname{subG} (\sigma^2))$ if its moment generating function satisfies $\mathbb{E}[\exp (t X)] \leq \exp \left(\frac{\sigma^2 t^2}{2}\right)$ for all $t>0$. For a sub-Gaussian random variable $X$, its sub-Gaussian norm $\|X\|_{\psi_2}$ is defined as $\|X\|_{\psi_2}=\inf \{c>0: \mathbb{E}[\exp \left(\frac{X^2}{c^2})\right] \leq 2\}$. Specifically, if $X \sim \operatorname{subG}(\sigma^2)$ we have $\|X\|_{\psi_2} \leq O(\sigma)$.
\end{definition}

\begin{definition}[Sub-Gaussian random vector]
    A zero mean random vector $\mathbf{X} \in \mathbb{R}^d$ is said to be sub-Gaussian with variance $\sigma^2$ (for simplicity, we call it $\sigma^2$-sub-Gaussian), which is denoted as $(\mathbf{X} \sim \operatorname{subG}_d(\sigma^2))$, if $\langle \mathbf{X}, \mathbf{u}\rangle$ is sub-Gaussian with variance $\sigma^2$ for any unit vector $\mathbf{u} \in \mathbb{R}^d$.
\end{definition}

\section{DP-GLMtron}
    
\subsection{Privacy Guarantee}
To guarantee the privacy of DP-GLMtron, we should first ensure that each step of DP-GLMtron is private.

\begin{lemma}\label{lem:dp_glm}
     Each update step of DP-GLMtron (Algorithm 1) ensures $ (\varepsilon_0, \delta_0 )$-differential privacy, provided that ${f} = \frac{c}{\varepsilon_0}$, where $c \geq \sqrt{2 \log  (1.25 / \delta_0 )}$, and $s_t$ denotes the clipping norm.
\end{lemma}

\begin{proof}[{\bf Proof}]
We first consider $\{\mathbf{w}_t\}_{t=0}^{N-1}$.

Each update step (excluding the DP-noise addition) is of the form:
$$
\mathbf{w}_{t+1}  \leftarrow \mathbf{w}_t-\eta \operatorname{clip}_{s_t} (\mathbf{x}_t (\operatorname{ReLU} (\mathbf{x}_t^{\top} \mathbf{w}_{t} )-y_t ) ),
$$
where $\operatorname{clip}_{s_t}(\mathbf{v}) = \mathbf{v} \cdot \max  \{1, \frac{{s_t}}{\|\mathbf{v}\|_2} \}$. Consequently, the local $L_2$ sensitivity of $\mathbf{w}_{t+1}$ is determined by analyzing the variation in the $t^{\text{th}}$ iteration data sample, as follows:
\begin{align*}
\Delta_2 & = \|\mathbf{w}_{t+1}^{\prime}-\mathbf{w}_{t+1} \| \\
& = \|\eta \operatorname{clip}_{s_t} (\mathbf{x}_t^{\prime} (\operatorname{ReLU} ({\mathbf{x}_t^{\prime}}^{\top} \mathbf{w}_{t} )-y_t^{\prime} ) )-\eta \operatorname{clip}_{s_t} (\mathbf{x}_t (\operatorname{ReLU} (\mathbf{x}_t^{\top} \mathbf{w}_{t} )-y_t ) ) \| \\
& \leq 2 \eta \| \operatorname{clip}_{s_t} (\mathbf{x}_t (\operatorname{ReLU} (\mathbf{x}_t^{\top} \mathbf{w}_{t} )-y_t ) ) \| \\
& =2 \eta {s_t}.
\end{align*}

Moreover, denoting $\varepsilon^{\prime} =\varepsilon_0 / \sqrt{8 N \log (2 / \delta_0)}$ and $\delta^{\prime}=\delta_0 /(2 N)$, according to Lemma 2.3 in \citet{karwa2017finite}, we could know that $\{s_t\}_{t=0}^{N}$ is $(\varepsilon^{\prime},\delta^{\prime})$-DP.

\end{proof}

\begin{lemma}[\citet{feldman2022hiding}]\label{lem:feld_one}
     For a domain $\mathcal{D}$, let $\mathcal{R}^{(i)}: f \times \mathcal{D}  \rightarrow \mathcal{S}^{(i)}$ for $i \in[n]$ (where $\mathcal{S}^{(i)}$ is the range space of $\mathcal{R}^{(i)}$ ) be a sequence of algorithms such that $\mathcal{R}^{(i)} (z_{1: i-1}, \cdot )$ is a $ (\varepsilon_0, \delta_0 )$-DP local randomizer for all values of auxiliary inputs $z_{1: i-1} \in \mathcal{S}^{(1)} \times \cdots \times \mathcal{S}^{(i-1)}$. Let $\mathcal{A}_s: \mathcal{D}^N  \rightarrow \mathcal{S}^{(1)} \times \cdots \times \mathcal{S}^{(n)}$ be the algorithm that given a dataset $x_{1: N} \in \mathcal{D}^N$, samples a uniformly random permutation $\pi$, then sequentially computes $z_i=\mathcal{R}^{(i)} (z_{1: i-1}, x_{\pi(i)} )$ for $i \in[n]$, and outputs $z_{1: N}$. Then for any $\delta \in[0,1]$ such that $\varepsilon_0 \leq \log  (\frac{n}{16 \log (2 / \delta)} ), \mathcal{A}_s$ is $ (\varepsilon, \delta+O (e^{\varepsilon} \delta_0 n ) )$-DP where $\varepsilon$ is:
$$
\varepsilon=O ( (1-e^{-\varepsilon_0} ) (\frac{\sqrt{e^{\varepsilon_0} \log (1 / \delta)}}{\sqrt{n}}+\frac{e^{\varepsilon_0}}{n} ) ).
$$
\end{lemma}

We are now prepared to prove the privacy of DP-GLMtron, utilizing the lemmas discussed above.

Firstly, we reformulate the update rule into a sequence of one-step algorithms as follows:
$$
\mathcal{R}^{(t+1)} (u_{0: t},(\mathbf{x}, y) ):=\mathbf{w}_{t+1}  \leftarrow \mathbf{w}_t (u_{0: t} )-\eta \operatorname{clip}_{s_t} (\mathbf{x}_{\pi(t)} (\operatorname{\text{ReLU}} (\mathbf{x}_{\pi(t)}^{\top} \mathbf{w}_t (u_{0: t} ) )-y_{\pi(t)} ) )-2 \eta {s_t} {f} g_t,
$$
where $u$ denotes auxiliary inputs, and $\pi(t)$ represents the sample at the $t$-th iteration after randomly permuting the input data.

From \cref{lem:dp_glm}, each $\mathcal{R}^{(t+1)}(u_{0: t}, \cdot)$ is a $(\varepsilon_0, \delta_0)$-DP local randomizer algorithm, where $\varepsilon_0 \leq \log  (\frac{N}{16 \log (2 / \widehat{\delta})} )$. The output of DP-GLMtron is derived through post-processing of the shuffled outputs $u_{t+1} = \mathcal{R}^{(t+1)}(u_{0: t},(\mathbf{x}, y))$ for $t \in {0, \ldots, N-1}$. Therefore, by \cref{lem:feld_one}, Algorithm DP-GLMtron adheres to $(\widehat{\varepsilon}, \widehat{\delta}+O(e^{\widehat{\varepsilon}} \delta_0 N))$-DP, where:
$$
\widehat{\varepsilon}=O ( (1-e^{-\varepsilon_0} ) (\frac{\sqrt{e^{\varepsilon_0} \log (1 / \widehat{\delta})}}{\sqrt{N}}+\frac{e^{\varepsilon_0}}{N} ) ).
$$
Assuming $\varepsilon_0 \leq \frac{1}{2}$, we can infer the existence of some constant $c_1 > 0$ such that:
\begin{align} \notag
\widehat{\varepsilon} & \leq c_1 \cdot (1-e^{-\varepsilon_0} ) (\frac{\sqrt{e^{\varepsilon_0} \log (1 / \widehat{\delta})}}{\sqrt{N}}+\frac{e^{\varepsilon_0}}{N} ) \\ \notag
& \leq c_1 \cdot ( (e^{\varepsilon_0 / 2}-e^{-\varepsilon_0 / 2} ) \sqrt{\frac{\log (1 / \widehat{\delta})}{N}}+ (e^{\varepsilon_0}-1 ) \frac{1}{N} ) \\
& \leq c_1 \cdot ( ( (1+\varepsilon_0 )- (1-\varepsilon_0 / 2 ) ) \sqrt{\frac{\log (1 / \widehat{\delta})}{N}}+ ( (1+2 \varepsilon_0 )-1 ) \frac{1}{N} ) \\ \notag
& =c_1 \cdot \varepsilon_0 (\frac{1}{2} \sqrt{\frac{\log (1 / \widehat{\delta})}{N}}+\frac{2}{N} ) .
\end{align}
By setting ${f} = \frac{\sqrt{2 \log (1.25 / \delta_0)}}{\varepsilon_0}$ in \cref{lem:dp_glm}, we ensure that each update step of DP-GLMtron independently satisfies $(\varepsilon_0, \delta_0)$-DP, based on standard Gaussian mechanism. Replacing $\varepsilon_0 = \frac{\sqrt{2 \log (1.25 / \delta_0)}}{{f}}$, we obtain:
\begin{equation}
\widehat{\varepsilon} \leq c_1 \cdot \frac{\sqrt{2 \log  (1.25 / \delta_0 ) \log (1 / \widehat{\delta})}}{{f} \sqrt{N}}
\end{equation}
To satisfy overall $(\varepsilon, \delta)$-DP, set $\widehat{\delta} = \frac{\delta}{2}$, and $\delta_0 = c_2 \cdot \frac{\delta}{e^{\widehat{\varepsilon} N}}$ for some constant $c_2 > 0$. From this, we have:
\begin{equation}
\widehat{\varepsilon} \leq c_1 \cdot \frac{\sqrt{2 \log  (c_2 \cdot 1.25 \cdot e^{\widehat{\varepsilon}} N / \delta ) \cdot \log (2 / \delta)}}{{f} \sqrt{N}}
\end{equation}
For any $\varepsilon \leq 1$, setting ${f} = c_3 \cdot \frac{\log (N / \delta)}{\varepsilon \sqrt{N}} \geq c_3 \frac{\sqrt{\log (N / \delta) \log (1 / \delta)}}{\varepsilon \sqrt{N}}$ for a sufficiently large $c_3 > 0$ ensures that $\widehat{\varepsilon} \leq \varepsilon$. Additionally, to fulfill \cref{lem:feld_one}'s assumption, $\varepsilon_0 < \frac{1}{2}$ must be satisfied, which is attainable by setting $\varepsilon = O(\sqrt{\frac{\log (N / \delta)}{N}})$.

This implies that for ${f} = \Omega(\frac{\log (N / \delta)}{\varepsilon \sqrt{N}})$, DP-GLMtron achieves $(\varepsilon, \delta)$-DP as long as $\varepsilon = O(\sqrt{\frac{\log (N / \delta)}{N}})$, thereby completing the proof.

\subsection{Utility Guarantee}
Here we first provide several auxiliary results that will be used in our DP-GLMtron utility bound. 
\begin{assumption}[Moment symmetricity conditions \cite{wu2023finite}]\label{asm:sym_restated}
    Assume that
    \begin{enumerate}
        \item[(A).] For every $\mathbf{u} \in \mathbb{H}$, it holds that
$$
\mathbb{E} [\mathbf{x} \mathbf{x}^{\top} \cdot \mathbbm{1} [\mathbf{x}^{\top} \mathbf{u}>0 ] ]=\mathbb{E} [\mathbf{x} \mathbf{x}^{\top} \cdot \mathbbm{1} [\mathbf{x}^{\top} \mathbf{u}<0 ] ].
$$
        \item[(B).] For every $\mathbf{u} \in \mathbb{H}$ and $\mathbf{v} \in \mathbb{H}$, it holds that
$$
\mathbb{E} [\mathbf{x} \mathbf{x}^{\top} \cdot \mathbbm{1} [\mathbf{x}^{\top} \mathbf{u}>0, \mathbf{x}^{\top} \mathbf{v}>0 ] ]=\mathbb{E} [\mathbf{x} \mathbf{x}^{\top} \cdot \mathbbm{1} [\mathbf{x}^{\top} \mathbf{u}<0, \mathbf{x}^{\top} \mathbf{v}<0 ] ] .
$$
        \item[(C).] For every $\mathbf{u} \in \mathbb{H}$, it holds that
$$
\mathbb{E} [\mathbf{x}^{\otimes 4} \cdot \mathbbm{1} [\mathbf{x}^{\top} \mathbf{u}>0 ] ]=\mathbb{E} [\mathbf{x}^{\otimes 4} \cdot \mathbbm{1} [\mathbf{x}^{\top} \mathbf{u}<0 ] ].
$$
        \item[(D).] For every $\mathbf{u} \in \mathbb{H}$ and $\mathbf{v} \in \mathbb{H}$, it holds that
$$
\mathbb{E} [ (\mathbf{x}^{\top} \mathbf{v} )^2 \mathbf{x} \mathbf{x}^{\top} \cdot \mathbbm{1} [\mathbf{x}^{\top} \mathbf{u}>0, \mathbf{x}^{\top} \mathbf{v}>0 ] ]=\mathbb{E} [ (\mathbf{x}^{\top} \mathbf{v} )^2 \mathbf{x} \mathbf{x}^{\top} \cdot \mathbbm{1} [\mathbf{x}^{\top} \mathbf{u}<0, \mathbf{x}^{\top} \mathbf{v}<0 ] ].
$$
    \end{enumerate}
\end{assumption}

The following results are direct consequences of \cref{asm:sym_restated}.
\begin{lemma}[\citet{wu2023finite}]\label{lem:sysm}
The following results are direct consequences of \cref{asm:sym_restated}.
    \begin{enumerate}
        \item Under \cref{asm:sym_restated} (A), it holds that: for every vector $\mathbf{u} \in \mathbb{H}$,
$$
\mathbb{E} [\mathbf{x} \mathbf{x}^{\top} \cdot \mathbbm{1} [\mathbf{x}^{\top} \mathbf{u}>0 ] ]=\frac{1}{2} \cdot \mathbb{E} [\mathbf{x} \mathbf{x}^{\top} ]=: \frac{1}{2} \cdot \mathbf{H} .
$$
        \item Under \cref{asm:sym_restated} (C), it holds that: for every vector $\mathbf{u} \in \mathbb{H}$,
$$
\mathbb{E} [\mathbf{x}^{\otimes 4} \cdot \mathbbm{1} [\mathbf{x}^{\top} \mathbf{u}>0 ] ]=\frac{1}{2} \cdot \mathbb{E} [\mathbf{x}^{\otimes 4} ]=: \frac{1}{2} \cdot \mathcal{M}
$$
    \end{enumerate}
\end{lemma}

\begin{lemma}[\citet{rudelson2013hanson}]
     Hanson-Wright Inequality: For any $\mathbf{X} \sim$ $\mathcal{N}(\mathbf{0}, \boldsymbol{\Sigma})$, the following holds for $t \geq 0$
    $$
    \mathbb{P}(\|\mathbf{X}\|^2 \geq \operatorname{Tr}(\boldsymbol{\Sigma})+2 \sqrt{t}\|\boldsymbol{\Sigma}\|_{\mathrm{F}}+2 t\|\boldsymbol{\Sigma}\|_{\text {op }}) \leq e^{-t}.
    $$
\end{lemma}

\begin{lemma}\label{lem: P_t^2}
    Let $\mathbf{P}_t:=\mathbf{I}-\eta \mathbbm{1}[\mathbf{x}_t^{\top} \mathbf{w}_{t-1}>0] \mathbf{x}_t \mathbf{x}_t^{\top}$, where each $\mathbf{x}_t \in \mathbb{R}^d$, $\forall t \in\{j, \ldots, T-1\}$ has been sampled i.i.d. from $\mathcal{D}$ and $\mathbf{w}_{t-1} \in \mathbb{R}^d$ is the weight parameter for iteration $t-1$. Let $\mathbf{z} \in \mathbb{R}^d$ be a vector independent of all $\mathbf{P}_t$ 's. Then for $b>0$ and $\eta<\frac{1}{R_x^2}$, we have with probability $\geq 1-b$ :
    $$
    \|\mathbf{P}_{T-1} \mathbf{P}_{T-2} \ldots \mathbf{P}_j \mathbf{z}\|^2 \leq \frac{1}{b} e^{-\eta \mu(T-j)}\|\mathbf{z}\|^2.
    $$
\end{lemma}

\begin{proof}
Note that
\begin{align*}
\underset{(\mathbf{x}, y) \sim \mathcal{D}}{\mathbb{E}}[\|\mathbf{P}_{T-1} \mathbf{P}_{T-2} \ldots \mathbf{P}_j \mathbf{z}\|^2]& =\underset{(\mathbf{x}, y) \sim \mathcal{D}}{\mathbb{E}}[(\mathbf{P}_{T-1} \mathbf{P}_{T-2} \ldots \mathbf{P}_j \mathbf{z})^{\top} \mathbf{P}_{T-1} \mathbf{P}_{T-2} \ldots \mathbf{P}_j \mathbf{z}] \\
& =\underset{(\mathbf{x}, y) \sim \mathcal{D}}{\mathbb{E}}[\mathbf{z}^{\top} \mathbf{P}_j^{\top} \ldots \mathbf{P}_{T-2}^{\top} \mathbf{P}_{T-1}^{\top} \mathbf{P}_{T-1} \mathbf{P}_{T-2} \ldots \mathbf{P}_j \mathbf{z}] \\
& =\underset{(\mathbf{x}, y) \sim \mathcal{D}}{\mathbb{E}}[\mathbf{z}^{\top} \mathbf{P}_j^{\top} \ldots \mathbf{P}_{T-2}^{\top} \underset{(\mathbf{x}_{T-1}, y_{T-1}) \sim \mathcal{D}}{\mathbb{E}}[\mathbf{P}_{T-1}^{\top} \mathbf{P}_{T-1}] \mathbf{P}_{T-2} \ldots \mathbf{P}_j \mathbf{z}].
\end{align*}
We undertake a focused examination of the expectation $ \underset{(\mathbf{x}_{T-1}, y_{T-1}) \sim \mathcal{D}}{\mathbb{E}}[\mathbf{P}_{T-1}^{\top} \mathbf{P}_{T-1}]$, considered independently.
\begin{align*}
    \underset{(\mathbf{x}, y) \sim \mathcal{D}}{\mathbb{E}}[\mathbf{P}_{T-1}^{\top} \mathbf{P}_{T-1}] &  = \underset{(\mathbf{x}, y) \sim \mathcal{D}}{\mathbb{E}}[(\mathbf{I}-\eta \mathbbm{1}[\mathbf{x}_{T-1}^{\top} \mathbf{w}_{T-2}>0] \mathbf{x}_{T-1} \mathbf{x}_{T-1}^{\top})(\mathbf{I}-\eta \mathbbm{1}[\mathbf{x}_{T-1}^{\top} \mathbf{w}_{T-2}>0] \mathbf{x}_{T-1} \mathbf{x}_{T-1}^{\top})^{\top}]\\
    & = \mathbf{I} - \frac{\eta}{2}\mathbf{H}  - \frac{\eta}{2}\mathbf{H} + \eta^2\mathcal{M}\\
    & \leq \mathbf{I} - {\eta}\mathbf{H} + \eta^2 R_x^2 \mathbf{H},
\end{align*}
where the first equality is a direct result of \cref{lem:sysm} and the last inequality drives from the \cref{asm:fourth_upper} and \cref{def:tail}.
As a result, we have
\begin{align*}
\underset{(\mathbf{x}, y) \sim \mathcal{D}}{\mathbb{E}}[\|\mathbf{P}_{T-1} \mathbf{P}_{T-2} \ldots \mathbf{P}_j \mathbf{z}\|^2] & \leq \underset{(\mathbf{x}, y) \sim \mathcal{D}}{\mathbb{E}}[\lambda_{\max }(\mathbf{I}-2 \eta \mathbf{H}+\eta^2 R_x^2 \mathbf{H})\|\mathbf{P}_{T-2} \ldots \mathbf{P}_j \mathbf{z}\|^2] \\
& =\lambda_{\max }(\mathbf{I}-2 \eta \mathbf{H}+\eta^2 R_x^2 \mathbf{H}) \underset{(\mathbf{x}, y) \sim \mathcal{D}}{\mathbb{E}}[\|\mathbf{P}_{T-2} \ldots \mathbf{P}_j \mathbf{z}\|^2] \\
& \leq(1-\eta(2-\eta R_x^2) \mu) \underset{(\mathbf{x}, y) \sim \mathcal{D}}{\mathbb{E}}[\|\mathbf{P}_{T-2} \ldots \mathbf{P}_j \mathbf{z}\|^2] \\
& \leq(1-\eta \mu) \underset{(\mathbf{x}, y) \sim \mathcal{D}}{\mathbb{E}}[\|\mathbf{P}_{T-2} \ldots \mathbf{P}_j \mathbf{z}\|^2] \\
& \leq \text{repeat the same procedure} \\
& \leq(1-\eta \mu)^{(T-j)}\|\mathbf{z}\|^2 \\
& \leq e^{-\eta \mu(T-j)}\|\mathbf{z}\|^2,
\end{align*}
With Markov Inequality, for $b>0$ indicating $\operatorname{Pr}\{\mathbf{Z} \geq \frac{\mathbb{E}[\mathbf{Z}]}{b}\} \leq b$, we have
$$
\operatorname{Pr}\{\|\mathbf{P}_{T-1} \ldots \mathbf{P}_j \mathbf{z}\|^2 \leq \frac{\underset{(\mathbf{x}, y) \sim \mathcal{D}}{\mathbb{E}}[\|\mathbf{P}_{T-1} \ldots \mathbf{P}_j \mathbf{z}\|^2]}{\beta}\} \geq 1-b.
$$
Therefore with probability at least $1-b$ :
$$
\|\mathbf{P}_{T-1} \ldots \mathbf{P}_0 \mathbf{z}\|^2 \leq \frac{1}{b} e^{-\eta \mu(T-j)}\|\mathbf{z}\|^2.
$$
\end{proof}

\begin{lemma}
    Let $\eta$ be stepsize such that $\eta \leq \min \{ \frac{\lambda_d}{\lambda_1 R_x^2\log ^{2a} N}, \frac{1}{3{f}\sqrt{d}} \}$, where $c_1, c_2>0$ are global constants and $\Gamma=4 C_2 R_x \cdot \log ^{2 a} N \cdot(\sqrt{\|\mathbf{H}\|_2}\|\mathbf{w}_*\|+\sqrt{\kappa} \sigma)$. Furthermore, let ${f}={f}_{\varepsilon, \delta, N}$ be a function of $\varepsilon, \delta, N$. Then, with probability $\geq 1-\frac{1}{N^{100}},\|\mathbf{x}_t(\langle\mathbf{x}_t, \mathbf{w}_t\rangle-y_t)\| \leq \Gamma$ for all $0 \leq t \leq N-1 ; \mathbf{w}_t$ is the $t^{\text {th }}$ iterate of Algorithm DP-GLMtron.
\end{lemma}

\begin{proof}

    We begin by examining the base case when $t=0$ and the norm of the "gradient" can be expressed as:
    $$
    \begin{aligned} \|\mathbf{x}_0(\operatorname{ReLU}(\mathbf{x}_0^{\top} \mathbf{w}_{0})-y_0)\| & =\|\mathbf{x}_0(\operatorname{ReLU}(\mathbf{x}_0^{\top} \cdot\mathbf{0})-y_0)\| \\
    & =\|\mathbf{x}_0 y_0\| \\
    & \leq\|\mathbf{x}_0\|\left|\mathbf{x}_0^{\top} \mathbf{w}_* + z_0\right|.
    \end{aligned}
    $$
    By the distribution of $\mathbf{x}$ and \cref{def:tail}, w.p. at least $1-b_{\mathbf{x}}$, we have:
    $$\| \mathbf{x}_0 \| \leq R_x \log ^a(1 / b_x),$$
    and by the triangle inequality $\left|\mathbf{x}_0^{\top} \mathbf{w}_* + z_0\right| \leq \|\mathbf{x}_0^{\top} \mathbf{w}_* \|+ \| z_0\|$, w.p. at least $1-b_{\mathbf{w}_*}-b_{\sigma}$, we have:
    $$
    \left|\mathbf{x}_0^{\top} \mathbf{w}_*\right|+\left|z_0\right| \leq C_2 \sqrt{\|\mathbf{H}\|_2}\|\mathbf{w}_*\| \log ^a(1 / b_{\mathbf{w}_*})+\sigma C_2 \log ^a(1 / b_\sigma).
    $$
    Since each $b$ is $1/\operatorname{poly}(N)$, the lemma holds.
    
    Now let us assume that the Lemma is valid for the $(t-1)$-th iteration. Proceeding with this assumption, we turn our attention to the $t$-th iteration:
    $$\|\mathbf{x}_t(\operatorname{ReLU}(\mathbf{x}_t^{\top} \mathbf{w}_t )-y_t)\| = \|\mathbf{x}_t(\operatorname{max}(0,\mathbf{x}_t^{\top} \mathbf{w}_t)-y_t)\|.$$
    
    It is obvious that when $\mathbf{x}_t^{\top} \mathbf{w}_t\leq 0$, the norm of gradient simplifies to $\|\mathbf{x}_t y_t\|$, which aligns closely with base case.
    
    If $\mathbf{x}_t^{\top} \mathbf{w}_t, \mathbf{x}_t^{\top} \mathbf{w}_* \geq 0$, then we will have
    \begin{align*}
    \|\mathbf{x}_t(\operatorname{max}(0,\mathbf{x}_t^{\top} \mathbf{w}_t)-y_t)\|& =\|\mathbf{x}_t \mathbf{x}_t^{\top}(\mathbf{w}_t-\mathbf{w}_*)+\mathbf{x}_t z_t\| \\
    & \leq\|\mathbf{x}\|(\|\mathbf{x}_t^{\top}(\mathbf{w}_t-\mathbf{w}_*)\|+\|z_t\|) \\
    & \leq R_x \log ^a(1 / b_x)(C_2 \sqrt{\|\mathbf{H}\|_2}\|\mathbf{w}_t-\mathbf{w}_*\| \log ^a(1 / b_{\mathbf{w}_t})+\sigma C_2 \log ^a(1 / b_\sigma)) \\
    & =C_2 R_x \log ^{2 a} N(\sqrt{\|\mathbf{H}\|_2}\|\mathbf{w}_t-\mathbf{w}_*\|+\sigma),
    \end{align*}       
    where $b_x, b_{\mathbf{w}_t}, b_\sigma$ is $1 / \operatorname{poly}(N)$. 
    
    If $\mathbf{x}_t^{\top} \mathbf{w}_t \geq 0$ and $  \mathbf{x}_t^{\top} \mathbf{w}_* \leq 0$, then we will have:
    \begin{equation}\label{eq:gamma_t}
        \begin{aligned}
        \|\mathbf{x}_t(\operatorname{max}(0,\mathbf{x}_t^{\top} \mathbf{w}_t)-y_t)\|& =\|\mathbf{x}_t \mathbf{x}_t^{\top}\mathbf{w}_t+\mathbf{x}_t z_t\| \\
        & \leq\|\mathbf{x}\|(\|\mathbf{x}_t^{\top}(\mathbf{w}_t-\mathbf{w}_*)\|+\|\mathbf{w}_*\|+\|z_t\|) \\
        & \leq R_x \log ^a(1 / b_x)(C_2 \sqrt{\|\mathbf{H}\|_2}\|\mathbf{w}_t-\mathbf{w}_*\| \log ^a(1 / b_{\mathbf{w}_t})+\|\mathbf{w}_*\|+\sigma C_2 \log ^a(1 / b_\sigma)) \\
        & =C_2 R_x \log ^{2 a} N(\sqrt{\|\mathbf{H}\|_2}\|\mathbf{w}_t-\mathbf{w}_*\|+\|\mathbf{w}_*\|+\sigma),  
        \end{aligned}        
    \end{equation}

    Given that the threshold $s_{t-1}$ has not been exceeded in iterations, we can observe the following decomposition at iteration $t-1$:
    \begin{equation}\label{eq:wt-w*_clipping_bound}
        \begin{aligned}
        \mathbf{w}_t-\mathbf{w}_* & =\mathbf{w}_{t-1}-\mathbf{w}_*-\eta(\operatorname{clip}_{s_{t-1}}(\mathbf{x}_{t-1}(\operatorname{ReLU}(\mathbf{x}_{t-1}^{\top} \mathbf{w}_{t-1})-y_{t-1}))+2 s_{t-1} {f} \mathbf{g}_{t-1}) \\
        & =\mathbf{w}_{t-1}-\mathbf{w}_*-\eta(\mathbf{x}_{t-1}(\operatorname{ReLU}(\mathbf{x}_{t-1}^{\top} \mathbf{w}_{t-1})-y_{t-1})+2 s_{t-1} {f} \mathbf{g}_{t-1}) \\
        &=  \mathbf{w}_{t-1}- \mathbf{w}_*- \eta \mathbbm{1}[\mathbf{x}_t^{\top} \mathbf{w}_{t-1}>0] \cdot \mathbf{x}_t \mathbf{x}_t^{\top} \mathbf{w}_{t-1}+\eta \mathbbm{1}[\mathbf{x}_t^{\top} \mathbf{w}_*>0] \cdot \mathbf{x}_t \mathbf{x}_t^{\top} \mathbf{w}_*\\
        & +\eta z_t \mathbf{x}_t +2 \eta s_{t-1} {f} \mathbf{g}_{t-1} \\ 
        &=  \mathbf{w}_{t-1}- \mathbf{w}_* - \eta \mathbbm{1}[\mathbf{x}_t^{\top} \mathbf{w}_{t-1}>0] \cdot \mathbf{x}_t \mathbf{x}_t^{\top}(\mathbf{w}_{t-1}-\mathbf{w}_*) \\ 
        & +\eta(\mathbbm{1}[\mathbf{x}_t^{\top} \mathbf{w}_*>0]-\mathbbm{1}[\mathbf{x}_t^{\top} \mathbf{w}_{t-1}>0]) \cdot \mathbf{x}_t \mathbf{x}_t^{\top} \mathbf{w}_*+\eta z_t\mathbf{x}_t -2 \eta s_{t-1} {f} \mathbf{g}_{t-1} \\
        & = (\mathbf{I}-\eta \mathbbm{1}[\mathbf{x}_t^{\top} \mathbf{w}_{t-1}>0] \mathbf{x}_t \mathbf{x}_t^{\top})(\mathbf{w}_{t-1}-\mathbf{w}_*) \\ 
        & +\eta(\mathbbm{1}[\mathbf{x}_t^{\top} \mathbf{w}_*>0]-\mathbbm{1}[\mathbf{x}_t^{\top} \mathbf{w}_{t-1}>0]) \mathbf{x}_t \mathbf{x}_t^{\top} \mathbf{w}_*+\eta z_t\mathbf{x}_t -2 \eta s_{t-1} {f} \mathbf{g}_{t-1}.
        \end{aligned}
    \end{equation}
    We introduce the following notations for clarity
    $$\mathbf{P}_t:=\mathbf{I}-\eta \mathbbm{1}[\mathbf{x}_t^{\top} \mathbf{w}_{t-1}>0] \mathbf{x}_t \mathbf{x}_t^{\top}, \quad \mathbf{u}_t = (\mathbbm{1}[\mathbf{x}_t^{\top} \mathbf{w}_*>0]-\mathbbm{1}[\mathbf{x}_t^{\top} \mathbf{w}_{t-1}>0]) \mathbf{x}_t \mathbf{x}_t^{\top} \mathbf{w}_*, \quad \mathbf{v}_t = z_t\mathbf{x}_t -2 \Gamma {f} \mathbf{g}_{t-1}.$$
    Then the expected inner product w.r.t $\mathbf{H}$ can be reformulated as follows.
    \begin{equation}\label{eq:expec_norm}
        \begin{aligned}
        &\underset{(\mathbf{x}, y) \sim \mathcal{D}}{\mathbb{E}}[\|\mathbf{w}_t-\mathbf{w}^*\|_{\mathbf{H}}^2] = \underset{(\mathbf{x}, y) \sim \mathcal{D}}{\mathbb{E}}[\|\mathbf{P}_t (\mathbf{w}_{t-1}-\mathbf{w}_*) + \eta \mathbf{u}_{t-1}+\eta \mathbf{v}_{t-1}\|_{\mathbf{H}}^2] = \underset{(\mathbf{x}, y) \sim \mathcal{D}}{\mathbb{E}}[\|\mathbf{P}_t (\mathbf{w}_{t-1}-\mathbf{w}_*) + \eta \mathbf{u}_{t-1}+\eta \mathbf{v}_{t-1}\|_{\mathbf{H}}^2]\\
        & = \underbrace{\underset{(\mathbf{x}, y) \sim \mathcal{D}}{\mathbb{E}}[\|\mathbf{P}_t (\mathbf{w}_{t-1}- \mathbf{w}_*)\|_{\mathbf{H}}^2]}_{\text {(quadratic term 1) }}  + \underbrace{\underset{(\mathbf{x}, y) \sim \mathcal{D}}{\mathbb{E}}[\| \eta \mathbf{u}_{t-1}\|_{\mathbf{H}}^2]}_{\text {(quadratic term 1) }} + \underbrace{\underset{(\mathbf{x}, y) \sim \mathcal{D}}{\mathbb{E}}[\| \eta \mathbf{v}_{t-1}\|_{\mathbf{H}}^2]}_{\text {(quadratic term 1) }} \\
        & + \underbrace{2 \underset{(\mathbf{x}, y) \sim \mathcal{D}}{\mathbb{E}} [\mathbf{u}_{t-1} \mathbf{H} (\mathbf{P}_t (\mathbf{w}_{t-1}- \mathbf{w}_*))]+ 2 \underset{(\mathbf{x}, y) \sim \mathcal{D}}{\mathbb{E}} [\mathbf{v}_{t-1} \mathbf{H} (\mathbf{P}_t (\mathbf{w}_{t-1}- \mathbf{w}_*))] + 2 \underset{(\mathbf{x}, y) \sim \mathcal{D}}{\mathbb{E}} [\mathbf{u}_{t-1} \mathbf{H} \mathbf{v}_{t-1}]}_{\text {(crossing term ) }}.
    \end{aligned}
    \end{equation}
where the cross terms involving $z$ and $\mathbf{g}_{t}$ have zero expectation, attributable to the fact that $\mathbb{E}[z \mid \mathbf{x}_t] = 0$ and $\mathbb{E}[\mathbf{g}_t \mid \mathbf{x}_t] = 0$.

For the second quadratic term in \cref{eq:expec_norm}, we observe the following
$$
(\mathbbm{1}[\mathbf{x}_t^{\top} \mathbf{w}_*>0]-\mathbbm{1}[\mathbf{x}_t^{\top} \mathbf{w}_{t-1}>0])^2=\mathbbm{1}[\mathbf{x}_t^{\top} \mathbf{w}_{t-1}>0, \mathbf{x}_t^{\top} \mathbf{w}_*<0]+\mathbbm{1}[\mathbf{x}_t^{\top} \mathbf{w}_{t-1}<0, \mathbf{x}_t^{\top} \mathbf{w}_*>0]
$$
Then, we have
$$
\begin{aligned}
&\underset{(\mathbf{x}, y) \sim \mathcal{D}}{\mathbb{E}}[\text { (quadratic term 2) }] \\
&=\underset{(\mathbf{x}, y) \sim \mathcal{D}}{\mathbb{E}}((\mathbbm{1}[\mathbf{x}_t^{\top} \mathbf{w}_*>0]-\mathbbm{1}[\mathbf{x}_t^{\top} \mathbf{w}_{t-1}>0])^2 (\mathbf{x}_t \mathbf{x}_t^{\top} \mathbf{w}_*)^{\top} (\mathbf{x}_t \mathbf{x}_t^{\top} \mathbf{w}_*) ) \\
& =\underset{(\mathbf{x}, y) \sim \mathcal{D}}{\mathbb{E}}((\mathbbm{1}[\mathbf{x}_t^{\top} \mathbf{w}_{t-1}>0, \mathbf{x}_t^{\top} \mathbf{w}_*<0]+\mathbbm{1}[\mathbf{x}_t^{\top} \mathbf{w}_{t-1}<0, \mathbf{x}_t^{\top} \mathbf{w}_*>0]) \cdot(\mathbf{w}_*^{\top}\mathbf{x}_t )^2 \cdot \mathbf{x}_t^{\top} \mathbf{x}_t) \\
& =2 \underset{(\mathbf{x}, y) \sim \mathcal{D}}{\mathbb{E}}(\mathbbm{1}[\mathbf{x}_t^{\top} \mathbf{w}_{t-1}>0, \mathbf{x}_t^{\top} \mathbf{w}_*<0] \cdot(\mathbf{w}_*^{\top}\mathbf{x}_t )^2 \cdot \mathbf{x}_t^{\top} \mathbf{x}_t),
\end{aligned}
$$
where the last equation follows from \cref{asm:symmetric}. Similarly, for the crossing terms in \cref{eq:expec_norm}, we have
$$
\begin{aligned}
 &\underset{(\mathbf{x}, y) \sim \mathcal{D}}{\mathbb{E}}[ \text { (crossing term) }]\\ 
& =\underset{(\mathbf{x}, y) \sim \mathcal{D}}{\mathbb{E}}[(\mathbbm{1}[\mathbf{x}_t^{\top} \mathbf{w}_*>0]-\mathbbm{1}[\mathbf{x}_t^{\top} \mathbf{w}_{t-1}>0])^2 \cdot (\mathbf{x}_t \mathbf{x}_t^{\top} \mathbf{w}_* )^{\top}(\mathbf{I}-\eta \mathbbm{1}[\mathbf{x}_t^{\top} \mathbf{w}_{t-1}>0] \cdot \mathbf{x}_t \mathbf{x}_t^{\top}) (\mathbf{w}_{t-1}-\mathbf{w}_*) ] \\
& =2 \underset{(\mathbf{x}, y) \sim \mathcal{D}}{\mathbb{E}}[\mathbbm{1}[\mathbf{x}_t^{\top} \mathbf{w}_{t-1}>0, \mathbf{x}_t^{\top} \mathbf{w}_*<0] \cdot (\mathbf{x}_t \mathbf{x}_t^{\top} \mathbf{w}_* )^{\top}(\mathbf{I}-\eta \mathbbm{1}[\mathbf{x}_t^{\top} \mathbf{w}_{t-1}>0] \cdot \mathbf{x}_t \mathbf{x}_t^{\top})(\mathbf{w}_{t-1}-\mathbf{w}_*) ] \\
& = 2 \underset{(\mathbf{x}, y) \sim \mathcal{D}}{\mathbb{E}} [\mathbbm{1}[\mathbf{x}_t^{\top} \mathbf{w}_{t-1}>0, \mathbf{x}_t^{\top} \mathbf{w}_*<0] \cdot (\mathbf{x}_t \mathbf{x}_t^{\top} \mathbf{w}_* )^{\top}(\mathbf{w}_{t-1}-\mathbf{w}_*-\eta \mathbbm{1}[\mathbf{x}_t^{\top} \mathbf{w}_{t-1}>0] \cdot \mathbf{x}_t \mathbf{x}_t^{\top} (\mathbf{w}_{t-1}-\mathbf{w}_*) ) ] \\
& = 2 \underset{(\mathbf{x}, y) \sim \mathcal{D}}{\mathbb{E}}[\mathbbm{1}[\mathbf{x}_t^{\top} \mathbf{w}_{t-1}>0, \mathbf{x}_t^{\top} \mathbf{w}_*<0] \cdot (\mathbf{w}_*^{\top} \mathbf{x}_t \mathbf{x}_t^{\top} (\mathbf{w}_{t-1}-\mathbf{w}_*)-\eta \mathbbm{1}[\mathbf{x}_t^{\top} \mathbf{w}_{t-1}>0] \cdot \mathbf{w}_*^{\top} \mathbf{x}_t \mathbf{x}_t^{\top} \mathbf{x}_t \mathbf{x}_t^{\top} (\mathbf{w}_{t-1}-\mathbf{w}_*) ) ] \\
& = 2 \underset{(\mathbf{x}, y) \sim \mathcal{D}}{\mathbb{E}}[\mathbbm{1}[\mathbf{x}_t^{\top} \mathbf{w}_{t-1}>0, \mathbf{x}_t^{\top} \mathbf{w}_*<0] \cdot (1-\eta \mathbf{x}_{t}^{\top}\mathbf{x}_{t}) \cdot \mathbf{w}_*^{\top} \mathbf{x}_t \cdot \mathbf{x}_t^{\top} (\mathbf{w}_{t-1}-\mathbf{w}_*)  ]. 
\end{aligned}
$$
By applying the indicator function and considering $\eta \leq 1/R_{x}^2$, we deduce that
$$(1-\eta \mathbf{x}_{t}^{\top}\mathbf{x}_{t}), \mathbf{x}_t^{\top} (\mathbf{w}_{t-1}-\mathbf{w}_*) \geq 0 \quad \text{and} \quad \mathbf{w}_*^{\top} \mathbf{x}_t \leq 0 \Rightarrow  \mathbb{E} \text { (crossing term) } \leq 0
$$
By invoking \cref{asm:symmetric}, it indicates that
$$
\mathbb{E}(\mathbbm{1}[\mathbf{x}_t^{\top} \mathbf{w}_{t-1}>0] \cdot \mathbf{x}_t \mathbf{x}_t^{\top})=\frac{1}{2} \mathbf{H}.
$$
Moreover, if $\eta \leq \frac{1}{2R_x^2}$, it holds that 
\begin{align*}
    \underset{(\mathbf{x}, y) \sim \mathcal{D}}{\mathbb{E}}[\mathbf{P}_{T-1}^{\top} \mathbf{P}_{T-1}] &  = \underset{(\mathbf{x}, y) \sim \mathcal{D}}{\mathbb{E}} [(\mathbf{I}-\eta \mathbbm{1}[\mathbf{x}_{T-1}^{\top} \mathbf{w}_{T-2}>0] \mathbf{x}_{T-1} \mathbf{x}_{T-1}^{\top})(\mathbf{I}-\eta \mathbbm{1}[\mathbf{x}_{T-1}^{\top} \mathbf{w}_{T-2}>0] \mathbf{x}_{T-1} \mathbf{x}_{T-1}^{\top})^{\top}]\\
    & = \mathbf{I} - \frac{\eta}{2}\mathbf{H}  - \frac{\eta}{2}\mathbf{H} + \eta^2\mathcal{M}\\
    & \leq \mathbf{I} - {\eta}\mathbf{H} + \eta^2 R_x^2 \mathbf{H} \leq \mathbf{I} - \frac{\eta}{2}\mathbf{H}.
\end{align*}
Combining the above results, the update of $\underset{(\mathbf{x}, y) \sim \mathcal{D}}{\mathbb{E}}[\|\mathbf{w}_t-\mathbf{w}^*\|_{\mathbf{H}}^2]$ holds that
\begin{align*}
    \underset{(\mathbf{x}, y) \sim \mathcal{D}}{\mathbb{E}}[\|\mathbf{w}_t-\mathbf{w}^*\|_{\mathbf{H}}^2] \leq (1-\frac{\eta \mu}{2}) \underset{(\mathbf{x}, y) \sim \mathcal{D}}{\mathbb{E}}[\|\mathbf{w}_{t-1}-\mathbf{w}^*\|_{\mathbf{H}}^2]+{\sigma^2\eta^2} \operatorname{Tr}(\mathbf{H}) +\underset{(\mathbf{x}, y) \sim \mathcal{D}}{\mathbb{E}}[{4 \eta^2 s_{t-1}^2 f^2}] \operatorname{Tr}(\mathbf{H}).
\end{align*}
Considering that the \cref{eq:gamma_t} and the adaptive clipping algorithm, we have
$$
s_t \leq R_x C_2 \log ^{2 a} N(\sqrt{\|\mathbf{H}\|}\|\mathbf{w}_t-\mathbf{w}_*\|+\|\mathbf{w}_*\|+\sigma+\Delta).
$$
As a results, the update of $\underset{(\mathbf{x}, y) \sim \mathcal{D}}{\mathbb{E}}[\|\mathbf{w}_t-\mathbf{w}^*\|_{\mathbf{H}}^2]$ can be reformulated as
\begin{align*}
    \underset{(\mathbf{x}, y) \sim \mathcal{D}}{\mathbb{E}}[\|\mathbf{w}_t-\mathbf{w}^*\|_{\mathbf{H}}^2] &\leq (1-\frac{\eta \mu}{2}) \underset{(\mathbf{x}, y) \sim \mathcal{D}}{\mathbb{E}}[\|\mathbf{w}_{t-1}-\mathbf{w}^*\|_{\mathbf{H}}^2]\\
    &+{\sigma^2\eta^2} \operatorname{Tr}(\mathbf{H}) +{16 R_x^2 C_2^2 \log ^{4 a} N \eta^2 f^2} \operatorname{Tr}(\mathbf{H}) (\kappa\|\mathbf{w}_t-\mathbf{w}_*\|_{\mathbf{H}}^2+\|\mathbf{w}_*\|^2+\sigma^2+\Delta^2)\\
    &= (1-(\frac{\eta \mu}{2}-16 \eta^2 {f^2} C_2^2 R_x^2 \kappa \log ^{4 a} N \operatorname{Tr}(\mathbf{H}))) \underset{(\mathbf{x}, y) \sim \mathcal{D}}{\mathbb{E}}[\|\mathbf{w}_{t-1}-\mathbf{w}^*\|_{\mathbf{H}}^2] \\
    & +{\eta^2\sigma^2} \operatorname{Tr}(\mathbf{H} )+16 \eta^2 {f^2} C_2^2 R_x^2 \log ^{4 a} N(\sigma^2+\Delta^2+ \|\mathbf{w}_*\|^2) \operatorname{Tr}(\mathbf{H}),
\end{align*}
where $\Delta=\frac{\|\mathbf{w}^*\|_{\mathbf{H}}+\sigma}{N^{100}}$ and we use the fact $\mathbf{I} \mu \preceq \mathbf{H} \Longrightarrow\|\mathbf{H}\|\|\mathbf{w}_{t-1}-\mathbf{w}^*\|^2 \leq \kappa\|\mathbf{w}_{t-1}-\mathbf{w}^*\|_{\mathbf{H}}^2$. 

If $\frac{\eta \mu}{4} \geq 16 \eta^2 {f^2} C_2^2 R_x^2 \kappa \log ^{4 a} N \operatorname{Tr}(\mathbf{H})$, it means the step size $\eta$ satisfies that $\eta \leq \frac{\mu}{64 {f^2} C_2^2 R_x^2 \kappa \log ^{4 a} N \operatorname{Tr}(\mathbf{H})} $, therefore it holds that
\begin{align*}
    &\underset{(\mathbf{x}, y) \sim \mathcal{D}}{\mathbb{E}}[\|\mathbf{w}_t-\mathbf{w}^*\|_{\mathbf{H}}^2]\\
    &\leq (1-\frac{\eta \mu}{4}) \underset{(\mathbf{x}, y) \sim \mathcal{D}}{\mathbb{E}}[\|\mathbf{w}_{t-1}-\mathbf{w}^*\|_{\mathbf{H}}^2] +{\eta^2\sigma^2} \operatorname{Tr}(\mathbf{H} )+16 \eta^2 {f^2} C_2^2 R_x^2 \log ^{4 a} N(\sigma^2+\Delta^2+ \|\mathbf{w}_*\|^2) \operatorname{Tr}(\mathbf{H})\\
    &\leq (1-\eta \mu / 4)^t\|\mathbf{w}_0-\mathbf{w}^*\|_{\mathbf{H}}^2+\frac{2}{\eta \mu}({\eta^2\sigma^2} \operatorname{Tr}(\mathbf{H} )+16 \eta^2 {f^2} C_2^2 R_x^2 \log ^{4 a} N(\sigma^2+\Delta^2+ \|\mathbf{w}_*\|^2) \operatorname{Tr}(\mathbf{H}))\\
    &\leq e^{-\eta \mu t / 4}\|\mathbf{w}^*\|_{\mathbf{H}}^2 +\frac{2}{\eta \mu}({\eta^2\sigma^2} \operatorname{Tr}(\mathbf{H} )+16 \eta^2 {f^2} C_2^2 R_x^2 \log ^{4 a} N(\sigma^2+\Delta^2+ \|\mathbf{w}_*\|^2) \operatorname{Tr}(\mathbf{H})).
\end{align*}

Since $s_t \leq R_x C_2 \log ^{2 a} N(\sqrt{\|\mathbf{H}\|}\|\mathbf{w}_t-\mathbf{w}^*\|+\|\mathbf{w}_*\|+\sigma+\Delta)$, the bound on $\underset{(\mathbf{x}, y) \sim \mathcal{D}}{\mathbb{E}}[s_t^2]$ will be
\begin{align*}
 \underset{(\mathbf{x}, y) \sim \mathcal{D}}{\mathbb{E}}[s_t^2]& \leq 4 C_2^2 R_x^2 \log ^{4 a} N(\kappa \underset{(\mathbf{x}, y) \sim \mathcal{D}}{\mathbb{E}}[\|\mathbf{w}_t-\mathbf{w}^*\|_{\mathbf{H}}^2]+\sigma^2+\|\mathbf{w}_*\|^2 +\Delta^2) \\
& \leq 4 C_2^2 R_x^2 \log ^{4 a} N(\kappa (e^{-\eta \mu / 4 t}\|\mathbf{w}^*\|_{\mathbf{H}}^2+\frac{2 \eta \sigma^2}{\mu } \operatorname{Tr}(\mathbf{H} ) \\
& +\frac{32 \eta \alpha^2}{\mu } C_2^2 R_x^2 \log ^{4 a} N(\sigma^2+\|\mathbf{w}_*\|^2+\Delta^2) \operatorname{Tr}(\mathbf{H}))+\sigma^2+\|\mathbf{w}_*\|^2+ \Delta^2), \\
\end{align*}
which is decreasing with $t$ (w.p. $\geq 1-\frac{1}{\operatorname{Poly}(\mathrm{N})}$).

Thus, if we define $\Gamma$ s.t.
\begin{align*}
\Gamma^2 & =\max  \{\operatorname{Upper-Bound} (\underset{(\mathbf{x}, y) \sim \mathcal{D}}{\mathbb{E}} [s_{0}^2 ] ), \ldots, \operatorname{Upper-Bound} (\underset{(\mathbf{x}, y) \sim \mathcal{D}}{\mathbb{E}} [s_T^2 ] ) \} \\
& =\operatorname{Upper-Bound} (\underset{(\mathbf{x}, y) \sim \mathcal{D}}{\mathbb{E}} [s_{0}^2 ] ) \\
& =4 C_2^2 R_x^2 \log ^{4 a} N(\kappa \underset{(\mathbf{x}, y) \sim \mathcal{D}}{\mathbb{E}}[\|\mathbf{w}_{0}-\mathbf{w}^*\|_{\mathbf{H}}^2]+\sigma^2+\|\mathbf{w}_*\|^2+\Delta^2).
\end{align*}
\end{proof}

\begin{lemma}[Generic bounds on the DP-GLMtron iterates \citet{wu2023finite}]\label{lem:glm_iter}
    Suppose that \cref{asm:symmetric} holds. Considering the DP-GLMtron algorithm, we have the following recursion:
    \begin{itemize}
        \item $\mathbf{A}_t  \preceq \mathbf{A}_{t-1}-\frac{\eta}{2} (\mathbf{H} \mathbf{A}_{t-1}+\mathbf{A}_{t-1} \mathbf{H} )+\eta^2 \mathcal{M} \circ \mathbf{A}_{t-1}+\eta^2 \sigma^2 \mathbf{H}+ 4\eta^2{\Gamma}^2{f}^2 \mathbf{I} $
        \item $\mathbf{A}_t  \succeq \mathbf{A}_{t-1}-\frac{\eta}{2} (\mathbf{H} \mathbf{A}_{t-1}+\mathbf{A}_{t-1} \mathbf{H} )+\frac{\eta^2}{4} \mathcal{M} \circ \mathbf{A}_{t-1}+\eta^2 \sigma^2 \mathbf{H}+ 4\eta^2{\Gamma}^2{f}^2 \mathbf{I}$
    \end{itemize}
    where $\mathbf{A}_t:=\mathbb{E} (\mathbf{w}_t-\mathbf{w}_* ) (\mathbf{w}_t-\mathbf{w}_* )^{\top}$, $t \geq 0$
\end{lemma}

Now consider the recursion of $\mathbf{A}_t$ given in \cref{lem:glm_iter}. Note that $\mathbf{A}_t$ is related to $\mathbf{A}_{t-1}$ through a linear operator, therefore $\mathbf{A}_t$ can be understood as the sum of two iterates, i.e., $\mathbf{A}_t:=\mathbf{B}_t+\mathbf{C}_t$, where
\begin{equation}\label{eq:decom_recur}
     \left\{\begin{array} { l } 
    { 
    \mathbf{B}_t \preceq  (\mathcal{I}-\frac{\eta}{2} \cdot \mathcal{T} (2 \eta ) ) \circ \mathbf{B}_{t-1}; 
    } \\
    { \mathbf{B}_0 =  (\mathbf{w}_0-\mathbf{w}_* )^{\otimes 2}}
    \end{array} \quad  \left\{\begin{array}{l}
    { 
    \mathbf{C}_t \preceq  (\mathcal{I}-\frac{\eta}{2} \cdot \mathcal{T} (2 \eta ) ) \circ \mathbf{C}_{t-1} + \eta^2 \sigma^2 \mathbf{H} + 4\eta^2{\Gamma}^2{f}^2 \mathbf{I}; 
    } \\
    { \mathbf{C}_0 = 0}
    \end{array} \right. \right.
\end{equation}
and
\begin{equation}\label{eq:decom_recur_low}
     \left\{\begin{array} { l } 
    { 
    \mathbf{B}_t \succeq  (\mathcal{I}-\frac{\eta}{2} \cdot \mathcal{T} (\frac{\eta}{2} ) ) \circ \mathbf{B}_{t-1}; 
    } \\
    { \mathbf{B}_0 =  (\mathbf{w}_0-\mathbf{w}_* )^{\otimes 2}}
    \end{array} \quad  \left\{\begin{array}{l}
    { 
    \mathbf{C}_t \succeq  (\mathcal{I}-\frac{\eta}{2} \cdot \mathcal{T} (\frac{\eta}{2} ) ) \circ \mathbf{C}_{t-1} + \eta^2 \sigma^2 \mathbf{H} + 4\eta^2{\Gamma}^2{f}^2 \mathbf{I}; 
    } \\
    { \mathbf{C}_0 = 0}
    \end{array} \right. \right.
\end{equation}
where 
$$
 \left\{\begin{array} { l } 
    { 
     (\mathcal{I}-\frac{\eta}{2} \cdot \mathcal{T} (2 \eta ) ) \circ \mathbf{A}_{t-1}:=\mathbf{A}_{t-1}-\frac{\eta}{2} (\mathbf{H} \mathbf{A}_{t-1}+\mathbf{A}_{t-1} \mathbf{H} )+\eta^2 \mathcal{M} \circ \mathbf{A}_{t-1}; 
    } \\
    {  (\mathcal{I}-\frac{\eta}{2} \cdot \mathcal{T} (\frac{\eta}{2} ) ) \circ \mathbf{A}_{t-1}:=\mathbf{A}_{t-1}-\frac{\eta}{2} (\mathbf{H} \mathbf{A}_{t-1}+\mathbf{A}_{t-1} \mathbf{H} )+\frac{\eta^2 }{4}\mathcal{M} \circ \mathbf{A}_{t-1}}
    \end{array}  \right.
$$
Besides, since our DP-GLM-tron is run with constant stepsize $\eta$ and outputs the average of the iterates:
\begin{equation}\label{eq:w_avg}
    \overline{\mathbf{w}}_N:=\frac{1}{N} \sum_{t=0}^{N-1} \mathbf{w}_t .
\end{equation}
Then, the following lemma holds:
\begin{lemma}\label{lem:error_decom}
    Suppose that \cref{asm:symmetric} hold. For $\overline{\mathbf{w}}_N$ defined in \cref{eq:w_avg}, we have that
    \begin{align*}
    & \mathbb{E} \langle\mathbf{H}, (\overline{\mathbf{w}}_N-\mathbf{w}_* )^{\otimes 2} \rangle \leq \sum_{t=0}^{N-1} \sum_{k=t}^{N-1}\frac{1}{\eta N^2} \langle  (\mathbf{I}-\frac{\eta}{2} \mathbf{H} )^{k-t}\mathbf{H},  \mathbf{A}_t \rangle, \\
    & \mathbb{E} \langle\mathbf{H}, (\overline{\mathbf{w}}_N-\mathbf{w}_* )^{\otimes 2} \rangle \geq \sum_{t=0}^{N-1} \sum_{k=t}^{N-1}\frac{1}{2 \eta N^2} \langle  (\mathbf{I}-\frac{\eta}{2} \mathbf{H} )^{k-t} \mathbf{H},  \mathbf{A}_t \rangle .
    \end{align*}
\end{lemma}
\begin{proof}[{\bf Proof}]
    $$
    \begin{aligned}
    \mathbb{E}[\mathbf{w}_t-\mathbf{w}_* \mid \mathbf{w}_{t-1}]= & \mathbb{E}[(\mathbf{I}-\eta \mathbbm{1}[\mathbf{x}_t^{\top} \mathbf{w}_{t-1}>0] \mathbf{x}_t \mathbf{x}_t^{\top})(\mathbf{w}_{t-1}-\mathbf{w}_*) \mid \mathbf{w}_{t-1}]+2 \eta \Gamma {f} \mathbb{E}[\mathbf{g}_{t-1} \mid \mathbf{w}_{t-1}] \\
    & +\eta \cdot \mathbb{E}[(\mathbbm{1}[\mathbf{x}_t^{\top} \mathbf{w}_*>0]-\mathbbm{1}[\mathbf{x}_t^{\top} \mathbf{w}_{t-1}>0]) \mathbf{x}_t \mathbf{x}_t^{\top} \mathbf{w}_* \mid \mathbf{w}_{t-1}]+\eta \mathbb{E}[z_t \mathbf{x}_t \mid \mathbf{w}_{t-1}] \\
    = & \mathbb{E}[(\mathbf{I}-\eta \mathbbm{1}[\mathbf{x}_t^{\top} \mathbf{w}_{t-1}>0] \mathbf{x}_t \mathbf{x}_t^{\top})(\mathbf{w}_{t-1}-\mathbf{w}_*) \mid \mathbf{w}_{t-1}] \\
    = & (\mathbf{I}-\frac{\eta}{2} \mathbf{H})(\mathbf{w}_{t-1}-\mathbf{w}_*)
    \end{aligned}
    $$
    The remaining proof simply follows \cite{zou2021benign}.
\end{proof}

From the decomposition presented in \cref{eq:decom_recur} and \cref{eq:decom_recur_low}, we know that $\sum_{t=0}^N \mathbf{A}_t = \sum_{t=0}^N \mathbf{B}_t + \sum_{t=0}^N \mathbf{C}_t$. With this foundation, we can now bound the bias and variance terms separately.

\textbf{Variance error}

For $t=0$ we have $\mathbf{C}_0=0 \preceq \frac{\eta \sigma^2}{1-\eta { R_x^2}} \mathbf{I} + \frac{4\eta {\Gamma}^2 {f}^2}{1-\eta { R_x^2}} \mathbf{H}^{-1}$. 

We then assume that $\mathbf{C}_{t-1} \preceq \frac{\eta \sigma^2}{1-\eta { R_x^2}} \mathbf{I} + \frac{4\eta {\Gamma}^2 {f}^2}{1-\eta { R_x^2}} \mathbf{H}^{-1}$, and exam $\mathbf{C}_t$ based on \cref{eq:decom_recur}:
\begin{align*}
\mathbf{C}_t & \preceq  (\mathcal{I}-\eta \cdot \mathcal{T}  (\eta ) ) \circ \mathbf{C}_{t-1} + \eta^2 \sigma^2 \mathbf{H} + 4\eta^2{\Gamma}^2{f}^2 \mathbf{I} \\
& =\mathbf{C}_{t-1}-\frac{\eta}{2} (\mathbf{H} \mathbf{C}_{t-1}+\mathbf{C}_{t-1} \mathbf{H} )+\eta^2 \mathcal{M} \circ \mathbf{C}_{t-1}+\eta^2 \sigma^2 \mathbf{H} + 4\eta^2{\Gamma}^2{f}^2 \mathbf{I} \\
& \preceq \frac{\eta \sigma^2}{1-\eta { R_x^2}} \mathbf{I} + \frac{4\eta {\Gamma}^2 {f}^2}{1-\eta { R_x^2}} \mathbf{H}^{-1} -\eta (\frac{\eta \sigma^2}{1-\eta { R_x^2}} \mathbf{H} + \frac{4\eta {\Gamma}^2 {f}^2}{1-\eta { R_x^2}} \mathbf{I}  )\\
& + \eta^2 { R_x^2}  (\frac{\eta \sigma^2}{1-\eta { R_x^2}} \mathbf{H} + \frac{4\eta {\Gamma}^2 {f}^2}{1-\eta { R_x^2}} \mathbf{I}  ) +\eta^2 \sigma^2 \mathbf{H} + 4\eta^2{\Gamma}^2{f}^2 \mathbf{I}\\
& \preceq \frac{\eta \sigma^2}{1-\eta { R_x^2}} \mathbf{I} + \frac{4\eta {\Gamma}^2 {f}^2}{1-\eta { R_x^2}} \mathbf{H}^{-1}.
\end{align*}

For the simplicity, we define $\mathbf{\Sigma}:= \sigma^2\mathbf{H} + 4{\Gamma}^2{f}^2\mathbf{I}$. By the definitions of $\mathcal{T}$ and $\widetilde{\mathcal{T}}$, we have:
$$
\begin{aligned}\label{eq:c_t_up} \notag
\mathbf{C}_t & = (\mathcal{I}-\frac{\eta}{2} \cdot \mathcal{T} (2 \eta ) )  \circ \mathbf{C}_{t-1}+\eta^2 \mathbf{\Sigma} \\
& = (\mathcal{I}-\frac{\eta}{2} \cdot \widetilde{\mathcal{T}} (2 \eta ) )  \circ \mathbf{C}_{t-1}+\eta^2(\mathcal{M}-\frac{1}{4}\widetilde{\mathcal{M}}) \circ \mathbf{C}_{t-1}+\eta^2 \bm{\Sigma} \\\notag
& \preceq (\mathcal{I}-\frac{\eta}{2} \cdot \widetilde{\mathcal{T}} (2 \eta ) ) \circ \mathbf{C}_{t-1}+\eta^2 \mathcal{M} \circ \mathbf{C}_{t-1}+\eta^2 \bm{\Sigma},
\end{aligned}
$$
where the last inequality is due to the fact that $\widetilde{\mathcal{M}}$ is a PSD mapping. Then by the iteration of variance, we have for all $t \geq 0$,
$$
\mathcal{M} \circ \mathbf{C}_t \preceq \mathcal{M} \circ  ( \frac{\eta \sigma^2}{1-\eta { R_x^2}} \mathbf{I} + \frac{4\eta {\Gamma}^2 {f}^2}{1-\eta { R_x^2}} \mathbf{H}^{-1}  )  \preceq \frac{\eta \sigma^2 { R_x^2}}{1-\eta { R_x^2}} \mathbf{H} + \frac{4\eta {\Gamma}^2 {f}^2 { R_x^2}}{1-\eta { R_x^2}} \mathbf{I}.
$$
Substituting the above into the previous result, we obtain
\begin{align*}
    \mathbf{C}_t & \preceq  (\mathcal{I}-\frac{\eta}{2} \cdot \widetilde{\mathcal{T}} (2 \eta ) )  \circ \mathbf{C}_{t-1}+\frac{\eta^3 { R_x^2} }{1-\eta { R_x^2}} \cdot  ( \sigma^2 \mathbf{H} + 4{\Gamma}^2{f}^2 \mathbf{I} )+ \eta^2 \mathbf{\Sigma} \\
    & =  (\mathcal{I}-\frac{\eta}{2} \cdot \widetilde{\mathcal{T}} (2 \eta ) )  \circ \mathbf{C}_{t-1}+\frac{\eta^3 { R_x^2} }{1-\eta { R_x^2}} \cdot  ( \sigma^2 \mathbf{H} + 4{\Gamma}^2{f}^2 \mathbf{I} ) +\eta^2  (\sigma^2\mathbf{H} + 4{\Gamma}^2{f}^2\mathbf{I}  ) \\
    & =  (\mathcal{I}-\frac{\eta}{2} \cdot \widetilde{\mathcal{T}} (2 \eta ) )  \circ \mathbf{C}_{t-1}+\frac{\eta^2 }{1-\eta { R_x^2}} \cdot  ( \sigma^2 \mathbf{H} + 4{\Gamma}^2{f}^2 \mathbf{I} ).
\end{align*}
It follows
\begin{align*}
    \mathbf{C}_t& \preceq \frac{\eta^2 }{1-\eta { R_x^2}} \cdot \sum_{k=0}^{t-1}  (\mathcal{I}-\frac{\eta}{2} \cdot \widetilde{\mathcal{T}} (2 \eta ) ) ^k \circ  ( \sigma^2 \mathbf{H} + 4{\Gamma}^2{f}^2 \mathbf{I} )  \\
    & =\frac{\eta^2 }{1-\eta { R_x^2}} \cdot \sum_{k=0}^{t-1}(\mathbf{I}-\frac{\eta}{2}\mathbf{H})^k  ( \sigma^2 \mathbf{H} + 4{\Gamma}^2{f}^2 \mathbf{I} )(\mathbf{I}-\frac{\eta}{2}\mathbf{H})^k  \\ 
    & \preceq \frac{\eta^2 }{1-\eta { R_x^2}} \cdot \sum_{k=0}^{t-1}(\mathbf{I}-\frac{\eta}{2}\mathbf{H})^k  ( \sigma^2 \mathbf{H} + 4{\Gamma}^2{f}^2 \mathbf{I} ) \\ \notag
    & =\frac{\eta \sigma^2}{1-\eta { R_x^2}} \cdot (\mathbf{I}-(\mathbf{I}-\frac{\eta}{2}\mathbf{H})^t ) + \frac{4\eta {\Gamma}^2{f}^2}{1-\eta { R_x^2}} \cdot (\mathbf{I}-(\mathbf{I}-\frac{\eta}{2}\mathbf{H})^t ) \cdot \mathbf{H}^{-1}.
\end{align*}

Consequently, the variance error can be represented as follows, in accordance with \cref{lem:error_decom}
\begin{align*}
 \text{variance error} & \leq \frac{1}{\eta N^2} \langle\mathbf{I}- (\mathbf{I}-\frac{\eta}{2} \mathbf{H} )^N, \sum_{t=0}^N \mathbf{C}_t \rangle\\ 
 &\leq \frac{1}{\eta N^2} \langle\mathbf{I}- (\mathbf{I}-\frac{\eta}{2} \mathbf{H} )^N, \sum_{t=0}^N \frac{\eta \sigma^2}{1-\eta { R_x^2}} \cdot (\mathbf{I}-(\mathbf{I}-\frac{\eta}{2}\mathbf{H})^t )  \rangle \\
 &+ \frac{1}{\eta N^2} \langle\mathbf{I}- (\mathbf{I}-\frac{\eta}{2} \mathbf{H} )^N, \sum_{t=0}^N \frac{4\eta {\Gamma}^2{f}^2}{1-\eta { R_x^2}} \cdot (\mathbf{I}-(\mathbf{I}-\frac{\eta}{2}\mathbf{H})^t ) \cdot \mathbf{H}^{-1}  \rangle \\ 
 &\leq \frac{\sigma^2}{  (1-\eta { R_x^2} )N} \langle\mathbf{I}- (\mathbf{I}-\frac{\eta}{2} \mathbf{H} )^N,   (\mathbf{I}-(\mathbf{I}-\frac{\eta}{2}\mathbf{H})^N )  \rangle \\  
 & + \frac{4 {\Gamma}^2{f}^2}{  (1-\eta { R_x^2} )N} \langle\mathbf{I}- (\mathbf{I}-\frac{\eta}{2} \mathbf{H} )^N,  (\mathbf{I}-(\mathbf{I}-\frac{\eta}{2}\mathbf{H})^N ) \cdot \mathbf{H}^{-1}  \rangle.
\end{align*}

Therefore, by integrating the $\Gamma$ and $f$, the variance error follows that
\begin{align*}
    \text{variance error} \lesssim \frac{d\sigma^2}{N} + \frac{\Gamma^2 f^2}{N} \cdot \operatorname{tr}(\mathbf{H}^{-1}) \lesssim \frac{d\sigma^2}{N} + \frac{d^2\log^2(N/\delta)}{N^2 \varepsilon^2} \cdot C_2^2 \kappa^2 (\sigma^2+\|\mathbf{w}_*\|_{\mathbf{H}}^2+\Delta^2) .
\end{align*}

\textbf{Bias error}

Now we consider the bias error, which depends on the initial error regardless of noise.
According to \cref{lem:error_decom}, the bias error of average iterate follows that
\begin{align*}
    \text{bias error} \leq \sum_{t=0}^{N-1} \sum_{k=t}^{N-1} \frac{1}{\eta N^2}\langle(\mathbf{I}-\frac{\eta}{2} \mathbf{H})^{k-t} \mathbf{H}, \mathbf{B}_t\rangle \leq \frac{1}{\eta N^2} \langle\mathbf{I}- (\mathbf{I}-\frac{\eta}{2} \mathbf{H} )^N, \sum_{t=0}^N \mathbf{B}_t \rangle  \leq \sum_{t=0}^N\frac{1}{\eta N^2} \operatorname{tr}(\mathbf{B}_t).
\end{align*}

Considering the recursion of $\mathbf{B}_t$, we have $\mathbf{B}_t  \preceq \mathbf{B}_{t-1}-\frac{\eta}{2} (\mathbf{H} \mathbf{B}_{t-1}+\mathbf{B}_{t-1} \mathbf{H} )+\eta^2 \mathcal{M} \circ \mathbf{B}_{t-1},$
which indicates the recursion of $\mathbf{B}_t$ follows that
\begin{align*}
    \mathbf{B}_t &\preceq \mathbf{B}_{t-1} -\eta \mathbf{H}\mathbf{B}_{t-1} + \eta^2 R_x^2 \mathbf{H}\mathbf{B}_{t-1}\\
    & \preceq (\mathbf{I} -\frac{\eta}{2} \mathbf{H}) \mathbf{B}_{t-1}
\end{align*}
The last inequality derives from the choice of step size.
Consequently, the bias error will be
\begin{align*}
    \text{bias error} \leq \sum_{t=0}^N\frac{1}{\eta N^2} \operatorname{tr}(\mathbf{B}_t) \leq \sum_{t=0}^N\frac{1}{\eta N^2} (1-\frac{\eta \mu}{2})^t\operatorname{tr}(\mathbf{B}_{0}) \leq \frac{1}{\eta N}\|\mathbf{w}_*\|^2 \lesssim \frac{d\log^2(N/\delta)}{N^2 \varepsilon^2} \cdot C_2^2 \kappa^2 \|\mathbf{w}_*\|^2.
\end{align*}

Combining the previous variance error, we complete the proof.

\section{DP-MBGLMtron}
For DP-MBGLMtron algorithm, we perform the following update
$$
\mathbf{w}_{t+1} \leftarrow \mathbf{w}_t-\eta(\frac{1}{b} \sum_{i=0}^{b-1} \operatorname{clip}_{s_t}(\mathbf{x}_{\tau+m+i}(\operatorname{ReLU}(\mathbf{x}_{\tau+m+i}^{\top} \mathbf{w}_t)-y_{\tau+m+i}))+f \cdot \frac{2 s_t}{b} \cdot \mathbf{g}_t).
$$

\subsection{Privacy Guarantee}

\begin{lemma}\label{lem:dpmini_restated}
     Algorithm DP - mini-batch-GLMtron with noise multiplier ${f}$ satisfies $\frac{1}{{f}^2}$-zCDP, and correspondingly satisfies $(\varepsilon, \delta)$-differential privacy when we set the noise multiplier ${f} \geq \frac{2 \sqrt{\log (1 / \delta)+\varepsilon}}{\varepsilon}$. Furthermore, if $\varepsilon \leq \log (1 / \delta)$, then ${f} \geq \frac{\sqrt{8 \log (1 / \delta)}}{\varepsilon}$ suffices to ensure $(\varepsilon, \delta)$-differential privacy.
\end{lemma}

We first show the step of gradient estimation is $\frac{1}{2 {f}^2}$-zCDP.

Notice the update of $c$ has sensitivity one and the variance of DP noise is $\lceil\log _2(B / \Delta)\rceil {f}^2$, hence, each step is $\frac{1}{2 \lceil\log _2(B / \Delta)\rceil {f}^2}$-zCDP.

If we take at most $\lceil\log _2(B / \Delta)\rceil {f}^2$ operations, we will have the aggregated privacy accumulation $\frac{1}{2{f}^2}$, which completes the privacy guarantee of gradient estimation.

Now we turn our attention to the update of $\mathbf{w}$ and consider the step without the Gaussian noise.
$$
\mathbf{w}_{t+1}\leftarrow\mathbf{w}_t-\frac{\eta}{b} \sum_{i=0}^{b-1} \operatorname{clip}_{s_t}(\mathbf{x}_{t,i}(\operatorname{ReLU}(\mathbf{x}_{t,i}^{\top}\mathbf{w}_t)-y_{t,i})).
$$
where $\operatorname{clip}_\Gamma(\boldsymbol{\nu})=\boldsymbol{\nu} \cdot \max \{1, \frac{\Gamma}{\|\boldsymbol{\nu}\|_2} \}$. Therefore, the local $L_2$ sensitivity of the $\mathbf{w}_{t+1}$ due to a sample difference in the $t$-th batch is $\Delta_2=\frac{2 \eta s_t}{b}$. Meanwhile, we know the variance of DP noise is $\frac{2\eta s_t{f}}{b}$, the above step is $\frac{1}{2{f}^2}$-zCDP since $\frac{\Delta_2^2}{2 \cdot \frac{4 \eta^2 s_t^2 {f}^2}{b^2}}=\frac{1}{2 {f}^2}$. 

According to the previous results and composition theorem, we know each iteration step is $\frac{1}{{f}^2}$-zCDP.
In our algorithm, every individual data point, denoted as $(\mathbf{x}_i, y_i)$, where $i$ is an index belonging to the set of all indices $N$, is included in precisely one mini-batch, which indicates the algorithm traverses the complete dataset exactly once, thereby ensuring that each data point is processed in a single iteration. Hence, according to the parallel composition of zCDP, DP-mini-batch-FLMtron is $\frac{1}{{f}^2}$-zCDP.

Recall that $\rho$-zCDP is implies a $(\mu, \mu\rho)$-RDP. We aim to optimize for any $\mu \geq 1$ and verify that the noise scaler ${f}$ prescribed in the theorem satisfies $(\varepsilon, \delta)$-Differential Privacy.

It is noted that $(\mu, \mu\rho)$-RDP implies $(\varepsilon, \delta)$-Approximate Privacy where $\varepsilon = \mu\rho + \frac{\log(1/\delta)}{\mu - 1}$ for all $\mu > 1$. The minimum value of $\varepsilon$, denoted as $\varepsilon_{\min}$, which equals $\rho + 2\sqrt{\rho\log(1/\delta)}$, is obtained when the derivative of $\varepsilon$ with respect to $\mu$ is zero, yielding $\mu = 1 + \sqrt{\log(1/\delta) / \rho}$.

For a given $\varepsilon$, we seek to minimize ${f}$ (which scales as $1/\sqrt{\rho}$), such that the computed maximum allowable $\rho$ ensures that $\varepsilon_{\min}(\rho) \leq \varepsilon$. Since $\varepsilon_{\min}(\rho)$ is a monotonically increasing function of ${f}$ and forms a second-order polynomial in $\sqrt{\rho}$ with its vertex corresponding to the maximum at $\varepsilon_{\min}(\rho) = \varepsilon$, we obtain the following relation:
$$
\frac{1}{{f}^2} = ( \sqrt{\log(1/\delta)} + \varepsilon - \sqrt{\log(1/\delta)} )^2 = \frac{\varepsilon^2}{( \sqrt{\log(1/\delta)} + \varepsilon + \sqrt{\log(1/\delta)} )^2}
$$
As the derived ${f}$ satisfied $(\varepsilon, \delta)$-DP, it is deduced that ${f} \geq \frac{2\sqrt{\log(1/\delta) + \varepsilon}}{\varepsilon}$, which ensures the algorithm's compliance with $(\varepsilon, \delta)$-Differential Privacy.

\subsection{Utility Guarantee}
Similar to DP-GLMtron, we also provide several auxiliary results that will be used in our utility analysis.
\begin{lemma}\label{lem:p_t^2_MB}
    If $\eta \leq \frac{b}{R_x^2+(b-1)\|\mathbf{H}\|}, \tau(t)=t \cdot(b+s)$ and $\mathbf{P}_t:=(\mathbf{I}-\frac{\eta}{b} \sum_{i=0}^{b-1} \mathbbm{1}[\mathbf{x}_{\tau(t)+m+i}^{\top} \mathbf{w}_{t-1} \geq 0] \mathbf{x}_{\tau(t)+m+i} \mathbf{x}_{\tau(t)+m+i}^{\top})$, then $\forall t, \underset{(\mathbf{x}, y) \sim \mathcal{D}}{\mathbb{E}}[\mathbf{P}_t^{\top} \mathbf{P}_t] \preceq \mathbf{I}-\eta \mathbf{H}$.
\end{lemma}
\begin{proof}
    Note that
    \begin{align*}
        &\underset{(\mathbf{x}, y) \sim \mathcal{D}}{\mathbb{E}}[\mathbf{P}_t^{\top} \mathbf{P}_t] \\
        = &\underset{(\mathbf{x}, y) \sim \mathcal{D}}{\mathbb{E}} [(\mathbf{I}-\frac{\eta}{b} \sum_{i=0}^{b-1} \mathbbm{1}[\mathbf{x}_{\tau(t)+m+i}^{\top} \mathbf{w}_{t-1} \geq 0] \mathbf{x}_{\tau(t)+m+i} \mathbf{x}_{\tau(t)+m+i}^{\top})^{\top}(\mathbf{I}-\frac{\eta}{b} \sum_{i=0}^{b-1} \mathbbm{1}[\mathbf{x}_{\tau(t)+m+i}^{\top} \mathbf{w}_{t-1} \geq 0] \mathbf{x}_{\tau(t)+m+i} \mathbf{x}_{\tau(t)+m+i}^{\top})] \\
        = & \mathbf{I} - \frac{2 \eta}{b} \underset{(\mathbf{x}, y) \sim \mathcal{D}}{\mathbb{E}} [ \sum_{i=0}^{b-1} \mathbbm{1}[\mathbf{x}_{\tau(t)+m+i}^{\top} \mathbf{w}_{t-1} \geq 0] \mathbf{x}_{\tau(t)+m+i} \mathbf{x}_{\tau(t)+m+i}^{\top} ] \\
        + & \frac{\eta^2}{b^2} \underset{(\mathbf{x}, y) \sim \mathcal{D}}{\mathbb{E}} [\sum_{i=0}^{b-1} \sum_{j=0}^{b-1} \mathbbm{1}[\mathbf{x}_{\tau(t)+m+i}^{\top} \mathbf{w}_{t-1} \geq 0] \mathbf{x}_{\tau(t)+m+i} \mathbf{x}_{\tau(t)+m+i}^{\top} \cdot \mathbbm{1}[\mathbf{x}_{\tau(t)+m+j}^{\top} \mathbf{w}_{t-1} \geq 0] \mathbf{x}_{\tau(t)+m+j} \mathbf{x}_{\tau(t)+m+j}^{\top} ]
    \end{align*}
    where we know $\mathbb{E}(\mathbbm{1}[\mathbf{x}_t^{\top} \mathbf{w}_{t-1}>0] \cdot \mathbf{x}_t \mathbf{x}_t^{\top} \otimes \mathbf{x}_t \mathbf{x}_t^{\top}) \preceq \mathbb{E}(\mathbf{x}_t \mathbf{x}_t^{\top} \otimes \mathbf{x}_t \mathbf{x}_t^{\top})=\mathcal{M}$, it indicates that
    \begin{align*}
        \underset{(\mathbf{x}, y) \sim \mathcal{D}}{\mathbb{E}}[\mathbf{P}_t^{\top} \mathbf{P}_t] &\preceq   \mathbf{I} - \eta \mathbf{H} + \frac{\eta^2}{b^2} (b\mathcal{M} + b(b-1) \mathbf{H}\mathbf{H}) \quad (*) \\
         &\preceq   \mathbf{I} - \eta \mathbf{H} + \frac{\eta^2}{b^2} (b R_x^2 \mathbf{H} + b(b-1) \|\mathbf{H}\|\mathbf{H}) \\
         & = \mathbf{I} - \eta \mathbf{H}(1 - \frac{\eta}{b} ( R_x^2 + (b-1)\|\mathbf{H}\|)).
    \end{align*}
    With the assumption of stepsize $\eta \leq \frac{1}{( R_x^2 + (b-1)\|\mathbf{H}\|)}$, we complete the proof.
\end{proof}
\begin{lemma}\label{lem:v_t^2_MB}
    If $\mathbf{v}_t=\frac{1}{b} \sum_{i=0}^{b-1} z_{\tau(t)+m+i} \mathbf{x}_{\tau(t)+m+i}-\frac{2 s f}{b} \mathbf{g}_t$, and $\tau(t)= t \cdot(b+s)$, then $\forall t$
    $$
    \underset{(\mathbf{x}, y) \sim \mathcal{D}}{\mathbb{E}}[\mathbf{v}_t \mathbf{v}_t^{\top}]=\frac{1}{b}\bm{\Sigma}+\frac{4 f^2}{b^2} \underset{(\mathbf{x}, y) \sim \mathcal{D}}{\mathbb{E}}[s_t^2] \mathbf{I},
    $$
    where $\bm{\Sigma} := \sum_{i=0}^{b-1}\sum_{j=0}^{b-1} (z_{\tau(t)+m+i} \mathbf{x}_{\tau(t)+m+i})(z_{\tau(t)+m+j} \mathbf{x}_{\tau(t)+m+j})^{\top}$.
\end{lemma}
\begin{proof}
    Note that 
    \begin{align*}
        \underset{(\mathbf{x}, y) \sim \mathcal{D}}{\mathbb{E}}[\mathbf{v}_t \mathbf{v}_t^{\top}] & = \underset{(\mathbf{x}, y) \sim \mathcal{D}}{\mathbb{E}} [(\frac{1}{b} \sum_{i=0}^{b-1} z_{\tau(t)+m+i} \mathbf{x}_{\tau(t)+m+i}-\frac{2 s f}{b} \mathbf{g}_t)(\frac{1}{b} \sum_{i=0}^{b-1} z_{\tau(t)+m+i} \mathbf{x}_{\tau(t)+m+i}-\frac{2 s f}{b} \mathbf{g}_t)^{\top}] \\
        & = \underset{(\mathbf{x}, y) \sim \mathcal{D}}{\mathbb{E}} [(\frac{1}{b} \sum_{i=0}^{b-1} z_{\tau(t)+m+i} \mathbf{x}_{\tau(t)+m+i})(\frac{1}{b} \sum_{i=0}^{b-1} z_{\tau(t)+m+i} \mathbf{x}_{\tau(t)+m+i})^{\top}] +  \underset{(\mathbf{x}, y) \sim \mathcal{D}}{\mathbb{E}}[\frac{4 s_t^2 f^2}{b^2} \mathbf{g}_t \mathbf{g}_t^{\top}]\\
        & = \underset{(\mathbf{x}, y) \sim \mathcal{D}}{\mathbb{E}} [ \frac{1}{b^2}\sum_{i=0}^{b-1}\sum_{j=0}^{b-1} (z_{\tau(t)+m+i} \mathbf{x}_{\tau(t)+m+i})(z_{\tau(t)+m+j} \mathbf{x}_{\tau(t)+m+j})^{\top}] + \frac{4 f^2}{b^2} \underset{(\mathbf{x}, y) \sim \mathcal{D}}{\mathbb{E}}[s_t^2] \mathbf{I},
    \end{align*}
    where we have utilized the fact that $\underset{(\mathbf{x}, y) \sim \mathcal{D}}{\mathbb{E}}[(z_{\tau(t)+m+i} \mathbf{x}_{\tau(t)+m+i})(z_{\tau(t)+m+j} \mathbf{x}_{\tau(t)+m+j})^{\top}]=\mathbf{0}$ for $i \neq j$ that stems from the independence of samples and the fact that $\underset{(\mathbf{x}, y) \sim \mathcal{D}}{\mathbb{E}}[\mathbf{g}_j]=\mathbf{0}$ has been sampled independently at each step.
\end{proof}


    Considering no clipping, the $t$-th update is given by:
    $$
    \mathbf{w}_{t+1}=\mathbf{w}_t-\frac{\eta}{b} \sum_{i=0}^{b-1} \mathbf{x}_{\tau(t)+m+i}(\operatorname{ReLU}(\mathbf{x}_{\tau(t)+m+i}^{\top} \mathbf{w}_t) -y_{\tau(t)+m+i})-\frac{2 \eta s_t f}{b} \mathbf{g}_t,
    $$
    where $\tau(t)=t \cdot(b+s)$. Hence, we could derive the following
    \begin{align*}
        &\mathbf{w}_{t+1}-\mathbf{w}_* =  (\mathbf{I}-  \frac{\eta}{b} \sum_{i=0}^{b-1} \mathbbm{1}[\mathbf{x}_{\tau(t)+m+i}^{\top} \mathbf{w}_{t}>0] \mathbf{x}_{\tau(t)+m+i} \mathbf{x}_{\tau(t)+m+i}^{\top}  )(\mathbf{w}_{t}-\mathbf{w}_*) \\
        &+\frac{\eta}{b} \sum_{i=0}^{b-1} (\mathbbm{1}[\mathbf{x}_{\tau(t)+m+i}^{\top} \mathbf{w}_*>0]-\mathbbm{1}[\mathbf{x}_{\tau(t)+m+i}^{\top} \mathbf{w}_{t}>0]) \mathbf{x}_{\tau(t)+m+i} \mathbf{x}_{\tau(t)+m+i}^{\top} \mathbf{w}_* + \frac{\eta}{b} \sum_{i=0}^{b-1} z_{\tau(t)+m+i} \mathbf{x}_{\tau(t)+m+i}-\frac{2 \eta \Gamma \alpha}{b} \mathbf{g}_t \\
        & := \mathbf{P}_t (\mathbf{w}_{t-1}-\mathbf{w}_*) + \eta \mathbf{u}_t+\eta \mathbf{v}_t,
    \end{align*}
    where we denote
    \begin{align*}
        \mathbf{P}_t &= (\mathbf{I}-  \frac{\eta}{b} \sum_{i=0}^{b-1} \mathbbm{1}[\mathbf{x}_{\tau(t)+m+i}^{\top} \mathbf{w}_{t-1}>0] \mathbf{x}_{\tau(t)+m+i} \mathbf{x}_{\tau(t)+m+i}^{\top}  )\\
        \mathbf{u}_t &= \frac{1}{b}\sum_{i=0}^{b-1} (\mathbbm{1}[\mathbf{x}_{\tau(t)+m+i}^{\top} \mathbf{w}_*>0]-\mathbbm{1}[\mathbf{x}_{\tau(t)+m+i}^{\top} \mathbf{w}_{t-1}>0]) \mathbf{x}_{\tau(t)+m+i} \mathbf{x}_{\tau(t)+m+i}^{\top} \mathbf{w}_*\\
        \mathbf{v}_t &= \frac{1}{b}\sum_{i=0}^{b-1} z_{\tau(t)+m+i} \mathbf{x}_{\tau(t)+m+i}-\frac{2 \eta \Gamma \alpha}{b} \mathbf{g}_t
    \end{align*}
    Let us consider the expected inner product w.r.t $\mathbf{H}$:
    \begin{align*}
        &\underset{(\mathbf{x}, y) \sim \mathcal{D}}{\mathbb{E}}[\|\mathbf{w}_t-\mathbf{w}^*\|_{\mathbf{H}}^2] = \underset{(\mathbf{x}, y) \sim \mathcal{D}}{\mathbb{E}}[\|\mathbf{P}_t (\mathbf{w}_{t-1}-\mathbf{w}_*) + \eta \mathbf{u}_{t-1}+\eta \mathbf{v}_{t-1}\|_{\mathbf{H}}^2]\\
        & = \underset{(\mathbf{x}, y) \sim \mathcal{D}}{\mathbb{E}}[\|\mathbf{P}_t (\mathbf{w}_{t-1}- \mathbf{w}_*)\|_{\mathbf{H}}^2] + \underset{(\mathbf{x}, y) \sim \mathcal{D}}{\mathbb{E}}[\| \eta \mathbf{u}_{t-1}\|_{\mathbf{H}}^2] + \underset{(\mathbf{x}, y) \sim \mathcal{D}}{\mathbb{E}}[\| \eta \mathbf{v}_{t-1}\|_{\mathbf{H}}^2]\\
        & + 2 \underset{(\mathbf{x}, y) \sim \mathcal{D}}{\mathbb{E}} [\mathbf{u}_{t-1} \mathbf{H} (\mathbf{P}_t (\mathbf{w}_{t-1}- \mathbf{w}_*))]+ 2 \underset{(\mathbf{x}, y) \sim \mathcal{D}}{\mathbb{E}} [\mathbf{v}_{t-1} \mathbf{H} (\mathbf{P}_t (\mathbf{w}_{t-1}- \mathbf{w}_*))] + 2 \underset{(\mathbf{x}, y) \sim \mathcal{D}}{\mathbb{E}} [\mathbf{u}_{t-1} \mathbf{H} \mathbf{v}_{t-1}].
    \end{align*}
    Notice that $\underset{(\mathbf{x}, y) \sim \mathcal{D}}{\mathbb{E}}[\mathbf{u}_t] = 0$ and is independent, thus it holds that
    \begin{equation}\label{eq:w_t-w*|H_mb}
        \begin{aligned}
            \underset{(\mathbf{x}, y) \sim \mathcal{D}}{\mathbb{E}}[\|\mathbf{w}_t-\mathbf{w}^*\|_{\mathbf{H}}^2] &= \underset{(\mathbf{x}, y) \sim \mathcal{D}}{\mathbb{E}}[\|\mathbf{P}_t (\mathbf{w}_{t-1}- \mathbf{w}_*)\|_{\mathbf{H}}^2] + \underset{(\mathbf{x}, y) \sim \mathcal{D}}{\mathbb{E}}[\| \eta \mathbf{u}_{t-1}\|_{\mathbf{H}}^2] + \underset{(\mathbf{x}, y) \sim \mathcal{D}}{\mathbb{E}}[\| \eta \mathbf{v}_{t-1}\|_{\mathbf{H}}^2]\\
            & +   \underset{(\mathbf{x}, y) \sim \mathcal{D}}{\mathbb{E}} [\eta \mathbf{u}_{t-1}^{\top} \mathbf{H} (\mathbf{P}_t (\mathbf{w}_{t-1}- \mathbf{w}_*))] + \underset{(\mathbf{x}, y) \sim \mathcal{D}}{\mathbb{E}} [\eta (\mathbf{P}_t (\mathbf{w}_{t-1}- \mathbf{w}_*))^{\top} \mathbf{H} \mathbf{u}_{t-1} ] .
        \end{aligned}
    \end{equation}
    Recall that $\mathbf{u}_{t} = \frac{1}{b}\sum_{i=0}^{b-1} (\mathbbm{1}[\mathbf{x}_{\tau(t)+m+i}^{\top} \mathbf{w}_*>0]-\mathbbm{1}[\mathbf{x}_{\tau(t)+m+i}^{\top} \mathbf{w}_{t-1}>0]) \mathbf{x}_{\tau(t)+m+i} \mathbf{x}_{\tau(t)+m+i}^{\top} \mathbf{w}_*$, it follows
    \begin{align*}
        \underset{(\mathbf{x}, y) \sim \mathcal{D}}{\mathbb{E}}[\| \eta \mathbf{u}_{t-1}\|_{\mathbf{H}}^2] = \frac{\eta^2}{b^2}\underset{(\mathbf{x}, y) \sim \mathcal{D}}{\mathbb{E}} &[(\sum_{i=0}^{b-1} (\mathbbm{1}[\mathbf{x}_{\tau(t)+m+i}^{\top} \mathbf{w}_*>0]-\mathbbm{1}[\mathbf{x}_{\tau(t)+m+i}^{\top} \mathbf{w}_{t-1}>0]) \mathbf{x}_{\tau(t)+m+i} \mathbf{x}_{\tau(t)+m+i}^{\top} \mathbf{w}_*)^{\top} \mathbf{H}\\
        & \cdot (\sum_{i=0}^{b-1} (\mathbbm{1}[\mathbf{x}_{\tau(t)+m+i}^{\top} \mathbf{w}_*>0]-\mathbbm{1}[\mathbf{x}_{\tau(t)+m+i}^{\top} \mathbf{w}_{t-1}>0]) \mathbf{x}_{\tau(t)+m+i} \mathbf{x}_{\tau(t)+m+i}^{\top} \mathbf{w}_*)].
    \end{align*}
    According to \cref{lem:sysm} (where each $x$ is independent and symmetric), the following conditions hold when $i \neq j$:
    \begin{align*}
    \underset{(\mathbf{x}, y) \sim \mathcal{D}}{\mathbb{E}} &[ (\mathbbm{1}[\mathbf{x}_{\tau(t)+m+i}^{\top} \mathbf{w}_*>0]-\mathbbm{1}[\mathbf{x}_{\tau(t)+m+i}^{\top} \mathbf{w}_{t-1}>0]) \mathbf{x}_{\tau(t)+m+i} \mathbf{x}_{\tau(t)+m+i}^{\top} \\
    &\cdot  (\mathbbm{1}[\mathbf{x}_{\tau(t)+m+j}^{\top} \mathbf{w}_*>0]-\mathbbm{1}[\mathbf{x}_{\tau(t)+m+j}^{\top} \mathbf{w}_{t-1}>0]) \mathbf{x}_{\tau(t)+m+j} \mathbf{x}_{\tau(t)+m+j}^{\top} ]= 0.
    \end{align*}
    It implies that 
    \begin{align*}
        \underset{(\mathbf{x}, y) \sim \mathcal{D}}{\mathbb{E}}[\| \eta \mathbf{u}_{t-1}\|_{\mathbf{H}}^2] = \frac{\eta^2}{b^2} \underset{(\mathbf{x}, y) \sim \mathcal{D}}{\mathbb{E}} &[\sum_{i=0}^{b-1} (\mathbbm{1}[\mathbf{x}_{\tau(t)+m+i}^{\top} \mathbf{w}_*>0]-\mathbbm{1}[\mathbf{x}_{\tau(t)+m+i}^{\top} \mathbf{w}_{t-1}>0])^2 \\
        & \cdot (\mathbf{w}_*^{\top} \mathbf{x}_{\tau(t)+m+i} \mathbf{x}_{\tau(t)+m+i}^{\top}  \mathbf{H} \mathbf{x}_{\tau(t)+m+i} \mathbf{x}_{\tau(t)+m+i}^{\top} \mathbf{w}_* )]\\
         = \frac{\eta^2}{b^2} \underset{(\mathbf{x}, y) \sim \mathcal{D}}{\mathbb{E}} &[\sum_{i=0}^{b-1} (\mathbbm{1}[\mathbf{x}_{\tau(t)+m+i}^{\top} \mathbf{w}_{t-1}>0, \mathbf{x}_{\tau(t)+m+i}^{\top} \mathbf{w}_*<0]+\mathbbm{1}[\mathbf{x}_{\tau(t)+m+i}^{\top} \mathbf{w}_{t-1}<0, \mathbf{x}_{\tau(t)+m+i}^{\top} \mathbf{w}_*>0])\\
        &\cdot (\mathbf{w}_*^{\top} \mathbf{x}_{\tau(t)+m+i} \mathbf{x}_{\tau(t)+m+i}^{\top}  \mathbf{H} \mathbf{x}_{\tau(t)+m+i} \mathbf{x}_{\tau(t)+m+i}^{\top} \mathbf{w}_* )]\\
         = \frac{2\eta^2}{b^2} \underset{(\mathbf{x}, y) \sim \mathcal{D}}{\mathbb{E}} &[\sum_{i=0}^{b-1} (\mathbbm{1}[\mathbf{x}_{\tau(t)+m+i}^{\top} \mathbf{w}_{t-1}>0, \mathbf{x}_{\tau(t)+m+i}^{\top} \mathbf{w}_*<0]) \\
         & \cdot (\mathbf{w}_*^{\top} \mathbf{x}_{\tau(t)+m+i} \mathbf{x}_{\tau(t)+m+i}^{\top}  \mathbf{H} \mathbf{x}_{\tau(t)+m+i} \mathbf{x}_{\tau(t)+m+i}^{\top} \mathbf{w}_* )],
    \end{align*}
    where we use the following fact in the second equality
    \begin{align*}
        &(\mathbbm{1}[\mathbf{x}_{\tau(t)+m+i}^{\top} \mathbf{w}_*>0]-\mathbbm{1}[\mathbf{x}_{\tau(t)+m+i}^{\top} \mathbf{w}_{t-1}>0])^2 \\
        &= (\mathbbm{1}[\mathbf{x}_{\tau(t)+m+i}^{\top} \mathbf{w}_{t-1}>0, \mathbf{x}_{\tau(t)+m+i}^{\top} \mathbf{w}_*<0]+\mathbbm{1}[\mathbf{x}_{\tau(t)+m+i}^{\top} \mathbf{w}_{t-1}<0, \mathbf{x}_{\tau(t)+m+i}^{\top} \mathbf{w}_*>0]). 
    \end{align*}
    Now we move on to the crossing term in \cref{eq:w_t-w*|H_mb}.
    \begin{align*}
         \underset{(\mathbf{x}, y) \sim \mathcal{D}}{\mathbb{E}} & [\eta \mathbf{u}_{t-1}^{\top} \mathbf{H} (\mathbf{P}_t (\mathbf{w}_{t-1}- \mathbf{w}_*))] + \underset{(\mathbf{x}, y) \sim \mathcal{D}}{\mathbb{E}} [\eta (\mathbf{P}_t (\mathbf{w}_{t-1}- \mathbf{w}_*))^{\top} \mathbf{H} \mathbf{u}_{t-1} ]  \\
        = \eta \underset{(\mathbf{x}, y) \sim \mathcal{D}}{\mathbb{E}} &[(\frac{1}{b}\sum_{i=0}^{b-1} (\mathbbm{1}[\mathbf{x}_{\tau(t)+m+i}^{\top} \mathbf{w}_*>0]-\mathbbm{1}[\mathbf{x}_{\tau(t)+m+i}^{\top} \mathbf{w}_{t-1}>0]) \mathbf{x}_{\tau(t)+m+i} \mathbf{x}_{\tau(t)+m+i}^{\top} \mathbf{w}_*)^{\top} \mathbf{H} \\
        &\cdot ((\mathbf{I}-  \frac{\eta}{b} \sum_{i=0}^{b-1} \mathbbm{1}[\mathbf{x}_{\tau(t)+m+i}^{\top} \mathbf{w}_{t-1}>0] \mathbf{x}_{\tau(t)+m+i} \mathbf{x}_{\tau(t)+m+i}^{\top}  ) (\mathbf{w}_{t-1}- \mathbf{w}_*))\\
        &+ ((\mathbf{I}-  \frac{\eta}{b} \sum_{i=0}^{b-1} \mathbbm{1}[\mathbf{x}_{\tau(t)+m+i}^{\top} \mathbf{w}_{t-1}>0] \mathbf{x}_{\tau(t)+m+i} \mathbf{x}_{\tau(t)+m+i}^{\top}  ) (\mathbf{w}_{t-1}- \mathbf{w}_*))^{\top} \mathbf{H} \\
        &\cdot (\frac{1}{b}\sum_{i=0}^{b-1} (\mathbbm{1}[\mathbf{x}_{\tau(t)+m+i}^{\top} \mathbf{w}_*>0]-\mathbbm{1}[\mathbf{x}_{\tau(t)+m+i}^{\top} \mathbf{w}_{t-1}>0]) \mathbf{x}_{\tau(t)+m+i} \mathbf{x}_{\tau(t)+m+i}^{\top} \mathbf{w}_*)].
    \end{align*}
    Similarly, for any $i \neq j$, \cref{lem:sysm} holds that 
    \begin{align*}
        \underset{(\mathbf{x}, y) \sim \mathcal{D}}{\mathbb{E}} &[  (\mathbbm{1}[\mathbf{x}_{\tau(t)+m+i}^{\top} \mathbf{w}_*>0]-\mathbbm{1}[\mathbf{x}_{\tau(t)+m+i}^{\top} \mathbf{w}_{t-1}>0]) \mathbf{x}_{\tau(t)+m+i} \mathbf{x}_{\tau(t)+m+i}^{\top} \\
        & \cdot \mathbbm{1}[\mathbf{x}_{\tau(t)+m+j}^{\top} \mathbf{w}_{t-1}>0] \mathbf{x}_{\tau(t)+m+j} \mathbf{x}_{\tau(t)+m+j}^{\top}] = 0.
    \end{align*}
    Therefore, the crossing term can be represented as
    \begin{align*}
         = \eta \underset{(\mathbf{x}, y) \sim \mathcal{D}}{\mathbb{E}} & [(\frac{1}{b}\sum_{i=0}^{b-1} (\mathbbm{1}[\mathbf{x}_{\tau(t)+m+i}^{\top} \mathbf{w}_*>0]-\mathbbm{1}[\mathbf{x}_{\tau(t)+m+i}^{\top} \mathbf{w}_{t-1}>0]) )  \\
        & \cdot (\mathbf{w}_*^{\top} \mathbf{x}_{\tau(t)+m+i} \mathbf{x}_{\tau(t)+m+i}^{\top}  \mathbf{H} (\mathbf{w}_{t-1}- \mathbf{w}_*) + (\mathbf{w}_{t-1}- \mathbf{w}_*)^{\top} \mathbf{H} \mathbf{x}_{\tau(t)+m+i} \mathbf{x}_{\tau(t)+m+i}^{\top} \mathbf{w}_*) \\
        & - \frac{\eta}{b^2}( (\sum_{i=0}^{b-1} (\mathbbm{1}[\mathbf{x}_{\tau(t)+m+i}^{\top} \mathbf{w}_*>0]-\mathbbm{1}[\mathbf{x}_{\tau(t)+m+i}^{\top} \mathbf{w}_{t-1}>0]) \cdot ( \mathbbm{1}[\mathbf{x}_{\tau(t)+m+i}^{\top} \mathbf{w}_{t-1}>0]  ) )\\
        & \cdot 2 (\mathbf{w}_*^{\top} \mathbf{x}_{\tau(t)+m+i} \mathbf{x}_{\tau(t)+m+i}^{\top}  \mathbf{H} \mathbf{x}_{\tau(t)+m+i} \mathbf{x}_{\tau(t)+m+i}^{\top} (\mathbf{w}_{t-1}- \mathbf{w}_*) )]\\
    \end{align*}
    Notice that
    \begin{align*}
        &- (\sum_{i=0}^{b-1} (\mathbbm{1}[\mathbf{x}_{\tau(t)+m+i}^{\top} \mathbf{w}_*>0]-\mathbbm{1}[\mathbf{x}_{\tau(t)+m+i}^{\top} \mathbf{w}_{t-1}>0]) \cdot ( \mathbbm{1}[\mathbf{x}_{\tau(t)+m+i}^{\top} \mathbf{w}_{t-1}>0]  ) )\\
        &= \sum_{i=0}^{b-1} (\mathbbm{1}[\mathbf{x}_{\tau(t)+m+i}^{\top} \mathbf{w}_{t-1}>0] - \mathbbm{1}[\mathbf{x}_{\tau(t)+m+i}^{\top} \mathbf{w}_{t-1}>0]  \cdot \mathbbm{1}[\mathbf{x}_{\tau(t)+m+i}^{\top} \mathbf{w}_*>0])\\
        & =\sum_{i=0}^{b-1} \mathbbm{1}[\mathbf{x}_{\tau(t)+m+i}^{\top} \mathbf{w}_{t-1}>0,\mathbf{x}_{\tau(t)+m+i}^{\top} \mathbf{w}_*>0 ].
    \end{align*}
    Combining the \cref{lem:sysm}, the crossing term holds that
    \begin{align*}
        \text{crossing term}= \underset{(\mathbf{x}, y) \sim \mathcal{D}}{\mathbb{E}}&  [ \frac{2\eta^2}{b^2}\sum_{i=0}^{b-1}( \mathbbm{1}[\mathbf{x}_{\tau(t)+m+i}^{\top} \mathbf{w}_{t-1}>0,\mathbf{x}_{\tau(t)+m+i}^{\top} \mathbf{w}_*>0 ]) \\
        & \cdot (\mathbf{w}_*^{\top} \mathbf{x}_{\tau(t)+m+i} \mathbf{x}_{\tau(t)+m+i}^{\top}  \mathbf{H} \mathbf{x}_{\tau(t)+m+i} \mathbf{x}_{\tau(t)+m+i}^{\top} (\mathbf{w}_{t-1}- \mathbf{w}_*) )].
    \end{align*}
    Let us add $\underset{(\mathbf{x}, y) \sim \mathcal{D}}{\mathbb{E}}[\| \eta \mathbf{u}_{t-1}\|_{\mathbf{H}}^2]$ and the crossing term together
    \begin{align*}
        &\underset{(\mathbf{x}, y) \sim \mathcal{D}}{\mathbb{E}}[\| \eta \mathbf{u}_{t-1}\|_{\mathbf{H}}^2] + \text{crossing term}\\
        &= \frac{2\eta^2}{b^2} \underset{(\mathbf{x}, y) \sim \mathcal{D}}{\mathbb{E}} [\sum_{i=0}^{b-1} (\mathbbm{1}[\mathbf{x}_{\tau(t)+m+i}^{\top} \mathbf{w}_{t-1}>0, \mathbf{x}_{\tau(t)+m+i}^{\top} \mathbf{w}_*<0]) \cdot (\mathbf{w}_*^{\top} \mathbf{x}_{\tau(t)+m+i} \mathbf{x}_{\tau(t)+m+i}^{\top}  \mathbf{H} \mathbf{x}_{\tau(t)+m+i} \mathbf{x}_{\tau(t)+m+i}^{\top} \mathbf{w}_{t-1} )].
    \end{align*}
    It is clear that $\underset{(\mathbf{x}, y) \sim \mathcal{D}}{\mathbb{E}}[\| \eta \mathbf{u}_{t-1}\|_{\mathbf{H}}^2] + \text{crossing term} \leq 0$. Moreover, according to \cref{lem:p_t^2_MB} and \cref{lem:v_t^2_MB}, we have that
    \begin{align*}
        &\underset{(\mathbf{x}, y) \sim \mathcal{D}}{\mathbb{E}}[\|\mathbf{P}_t (\mathbf{w}_{t-1}- \mathbf{w}_*)\|_{\mathbf{H}}^2] + \underset{(\mathbf{x}, y) \sim \mathcal{D}}{\mathbb{E}}[\| \eta \mathbf{v}_{t-1}\|_{\mathbf{H}}^2] \\
        &\leq (1-\eta \mu) \underset{(\mathbf{x}, y) \sim \mathcal{D}}{\mathbb{E}}[\|\mathbf{w}_{t-1}-\mathbf{w}^*\|_{\mathbf{H}}^2]+\frac{\eta^2}{b} \operatorname{Tr}(\mathbf{H} \boldsymbol{\Sigma})-0+\underset{(\mathbf{x}, y) \sim \mathcal{D}}{\mathbb{E}}[\eta^2 \frac{4 s_{t-1}^2 f^2}{b^2}] \operatorname{Tr}(\mathbf{H}).
    \end{align*}
    Therefore, the update of $\underset{(\mathbf{x}, y) \sim \mathcal{D}}{\mathbb{E}}[\|\mathbf{w}_t-\mathbf{w}^*\|_{\mathbf{H}}^2]$ in \cref{eq:w_t-w*|H_mb} holds that
    \begin{align*}
        \underset{(\mathbf{x}, y) \sim \mathcal{D}}{\mathbb{E}}[\|\mathbf{w}_t-\mathbf{w}^*\|_{\mathbf{H}}^2] \leq (1-\eta \mu) \underset{(\mathbf{x}, y) \sim \mathcal{D}}{\mathbb{E}}[\|\mathbf{w}_{t-1}-\mathbf{w}^*\|_{\mathbf{H}}^2]+\frac{\eta^2}{b} \operatorname{Tr}(\mathbf{H} \boldsymbol{\Sigma})-0+\underset{(\mathbf{x}, y) \sim \mathcal{D}}{\mathbb{E}}[\eta^2 \frac{4 s_{t-1}^2 f^2}{b^2}] \operatorname{Tr}(\mathbf{H}).
    \end{align*}
    Considering the adaptive clipping algorithm, we have
    $$
    s_t \leq R_x C_2 \log ^{2 a} N(\sqrt{\|\mathbf{H}\|}\|\mathbf{w}_t-\mathbf{w}_*\|+\|\mathbf{w}_*\|+\sigma+\Delta).
    $$
    Similar to one sample case, the recursion of $\underset{(\mathbf{x}, y) \sim \mathcal{D}}{\mathbb{E}}[\|\mathbf{w}_t-\mathbf{w}^*\|_{\mathbf{H}}^2]$ will be
    \begin{align*}
    \underset{(\mathbf{x}, y) \sim \mathcal{D}}{\mathbb{E}}[\|\mathbf{w}_t-\mathbf{w}^*\|_{\mathbf{H}}^2] &\leq (1-{\eta \mu}) \underset{(\mathbf{x}, y) \sim \mathcal{D}}{\mathbb{E}}[\|\mathbf{w}_{t-1}-\mathbf{w}^*\|_{\mathbf{H}}^2]\\
    &+\frac{\eta^2\sigma^2}{b}\operatorname{Tr}(\mathbf{H}) +{16 R_x^2 C_2^2 \log ^{4 a} N \eta^2 \frac{f^2}{b^2}} \operatorname{Tr}(\mathbf{H}) (\kappa\|\mathbf{w}_t-\mathbf{w}_*\|_{\mathbf{H}}^2+\|\mathbf{w}_*\|^2+\sigma^2+\Delta^2)\\
    &= (1-({\eta \mu}-16 \eta^2 {\frac{f^2}{b^2}} C_2^2 R_x^2 \kappa \log ^{4 a} N \operatorname{Tr}(\mathbf{H}))) \underset{(\mathbf{x}, y) \sim \mathcal{D}}{\mathbb{E}}[\|\mathbf{w}_{t-1}-\mathbf{w}^*\|_{\mathbf{H}}^2] \\
    & +\frac{\eta^2\sigma^2}{b} \operatorname{Tr}(\mathbf{H} )+16 \eta^2 {\frac{f^2}{b^2}} C_2^2 R_x^2 \log ^{4 a} N(\sigma^2+\Delta^2+ \|\mathbf{w}_*\|^2) \operatorname{Tr}(\mathbf{H}).
\end{align*}
    Notice that $T \cdot (b+m) = N $, thus if we have
    $\frac{\eta\mu}{2} \geq 16 \eta^2 {\frac{f^2}{b^2}} C_2^2 R_x^2 \kappa \log ^{4 a} N \operatorname{Tr}(\mathbf{H}) $, it equals to
    \begin{align*}
        (\frac{N}{T}-m)^2 \geq \frac{32 \eta f^2 C_2^2 R_x^2 \kappa \log ^{4 a} N \operatorname{Tr}(\mathbf{H})}{\mu},
    \end{align*}
    which implies that
    \begin{align*}
    &\underset{(\mathbf{x}, y) \sim \mathcal{D}}{\mathbb{E}}[\|\mathbf{w}_t-\mathbf{w}^*\|_{\mathbf{H}}^2]\\
    &\leq (1-\frac{\eta \mu}{2}) \underset{(\mathbf{x}, y) \sim \mathcal{D}}{\mathbb{E}}[\|\mathbf{w}_{t-1}-\mathbf{w}^*\|_{\mathbf{H}}^2] \frac{\eta^2\sigma^2}{b} \operatorname{Tr}(\mathbf{H} )+16 \eta^2 {\frac{f^2}{b^2}} C_2^2 R_x^2 \log ^{4 a} N(\sigma^2+\Delta^2+ \|\mathbf{w}_*\|^2) \operatorname{Tr}(\mathbf{H})\\
    &\leq (1-\eta \mu / 2)^t\|\mathbf{w}_0-\mathbf{w}^*\|_{\mathbf{H}}^2+\frac{2}{\eta \mu}(\frac{\eta^2\sigma^2}{b} \operatorname{Tr}(\mathbf{H} )+16 \eta^2 {\frac{f^2}{b^2}} C_2^2 R_x^2 \log ^{4 a} N(\sigma^2+\Delta^2+ \|\mathbf{w}_*\|^2) \operatorname{Tr}(\mathbf{H}))\\
    &\leq e^{-\eta \mu t / 2}\|\mathbf{w}^*\|_{\mathbf{H}}^2 +\frac{2}{\eta \mu}(\frac{\eta^2\sigma^2}{b} \operatorname{Tr}(\mathbf{H} )+16 \eta^2 {\frac{f^2}{b^2}} C_2^2 R_x^2 \log ^{4 a} N(\sigma^2+\Delta^2+ \|\mathbf{w}_*\|^2) \operatorname{Tr}(\mathbf{H})).
\end{align*}
Substituting the above to the $\underset{(\mathbf{x}, y) \sim \mathcal{D}}{\mathbb{E}}[s_t^2]$, similarly, we will have the following results
\begin{align*}
\Gamma^2 & =\max  \{\operatorname{Upper-Bound} (\underset{(\mathbf{x}, y) \sim \mathcal{D}}{\mathbb{E}} [s_{0}^2 ] ), \ldots, \operatorname{Upper-Bound} (\underset{(\mathbf{x}, y) \sim \mathcal{D}}{\mathbb{E}} [s_T^2 ] ) \} \\
& =\operatorname{Upper-Bound} (\underset{(\mathbf{x}, y) \sim \mathcal{D}}{\mathbb{E}} [s_{0}^2 ] ) \\
& =4 C_2^2 R_x^2 \log ^{4 a} N(\kappa \underset{(\mathbf{x}, y) \sim \mathcal{D}}{\mathbb{E}}[\|\mathbf{w}_{0}-\mathbf{w}^*\|_{\mathbf{H}}^2]+\sigma^2+\|\mathbf{w}_*\|^2+\Delta^2).
\end{align*}

Before presenting the utility guarantee, we first need to redefine certain notations and properties.

We denote the recursion:
\begin{equation}\label{eq:recur_batch}
    \begin{aligned}
        (\mathcal{I}-\mathcal{T}(\eta, b, \mathbf{H})) \circ \mathbf{A}_{t-1} &= \mathbf{A}_{t-1}-\frac{\eta}{2}(\mathbf{H} \mathbf{A}_{t-1}+\mathbf{A}_{t-1} \mathbf{H})+\frac{\eta^2}{b} ( \frac{1}{2}\mathcal{M} +(b-1) \frac{1}{4}\mathbf{H}^2) \circ \mathbf{A}_{t-1}\\
        (\mathcal{I}- \widetilde{\mathcal{T}}(\eta, b, \mathbf{H})) \circ \mathbf{A}_{t-1} &= \mathbf{A}_{t-1}-\frac{\eta}{2}(\mathbf{H} \mathbf{A}_{t-1}+\mathbf{A}_{t-1} \mathbf{H})+\frac{\eta^2}{4} \mathbf{H}^2 \circ \mathbf{A}_{t} = (\mathbf{I}-\frac{\eta}{2} \mathbf{H})\mathbf{A}_{t}(\mathbf{I}-\frac{\eta}{2} \mathbf{H}),
    \end{aligned}
\end{equation}
where $\mathcal{I} \circ \mathbf{A}=\mathbf{A}$, $ \mathcal{M} \circ \mathbf{A}=\mathbb{E}[(\mathbf{x}^{\top} \mathbf{A} \mathbf{x}) \mathbf{x} \mathbf{x}^{\top}]$ and $\widetilde{\mathcal{M}} \circ \mathbf{A}=\mathbf{H} \mathbf{A} \mathbf{H}$ for a symmetric matrix $\mathbf{A}$. For simplicity, we will use $(\mathcal{I}-\mathcal{T})$ and $(\mathcal{I}- \widetilde{\mathcal{T}})$ in place of the complete notation.

It can be readily understood that the following properties are satisfied:
\begin{lemma}[\citet{zou2021benign}] \label{lem:properties}
    An operator $\mathcal{O}$, when defined on symmetric matrices, is termed a Positive Semi-Definite (PSD) mapping, if $\mathbf{A} \succeq 0$ implies $\mathcal{O} \circ \mathbf{A} \succeq 0$. Consequently, we have:
    \begin{itemize}
        \item [1.] $\mathcal{M}$ and $\widetilde{\mathcal{M}}$ are both PSD mappings.
        \item [2.] $\mathcal{M}-\widetilde{\mathcal{M}}$ and $\widetilde{\mathcal{T}}-\mathcal{T}$ are both PSD mappings.
        \item [3.] $\mathcal{I}-\eta\mathcal{T}$ and $\mathcal{I}-\eta\widetilde{\mathcal{T}}$ are both PSD mappings.
        \item [4.] If $0<\eta<1 / \lambda_1$, then $\widetilde{\mathcal{T}}^{-1}$ exists, and is a PSD mapping.
        \item [5.] If $0<\eta<1 /(\alpha \operatorname{tr}(\mathbf{H}))$, then $\mathcal{T}^{-1} \circ \mathbf{A}$ exists for PSD matrix $\mathbf{A}$, and $\mathcal{T}^{-1}$ is a PSD mapping.
    \end{itemize}
\end{lemma}

\begin{proof}[{\bf Proof}]
    The subsequent proofs are summarized from \cite{jain2018parallelizing,zou2021benign}, and are included herein for the sake of completeness.
    \begin{itemize}
        \item [1.] For any PSD matrix $\mathbf{A} \succeq 0$, by definition, we have
        $$
        \begin{aligned}
        & \mathcal{M} \circ \mathbf{A}=\mathbb{E}[\mathbf{x} \mathbf{x}^{\top} \mathbf{A} \mathbf{x} \mathbf{x}^{\top}] \succeq 0, \\
        & \widetilde{\mathcal{M}} \circ \mathbf{A}=\mathbf{H} \mathbf{A H} \succeq 0 .
        \end{aligned}
        $$
        \item [2.] For any PSD matrix $\mathbf{A} \succeq 0$,
        $$
        (\mathcal{M}-\widetilde{\mathcal{M}}) \circ \mathbf{A}=\mathbb{E}[\mathbf{x} \mathbf{x}^{\top} \mathbf{A} \mathbf{x} \mathbf{x}^{\top}]-\mathbf{H} \mathbf{A} \mathbf{H}=\mathbb{E}[(\mathbf{x} \mathbf{x}^{\top}-\mathbf{H}) \mathbf{A}(\mathbf{x} \mathbf{x}^{\top}-\mathbf{H})] \succeq 0.
        $$
        Also, we have $\widetilde{\mathcal{T}}-\mathcal{T}=\frac{\eta^2}{2b}\mathcal{M}-\frac{\eta^2}{4b}\widetilde{\mathcal{M}} \succeq 0$, which indicates $\mathcal{M}-\widetilde{\mathcal{M}}$ and $\widetilde{\mathcal{T}}-\mathcal{T}$ are both PSD mappings.
        \item [3.] For any $\mathrm{PSD}$ matrix $\mathbf{A} \succeq 0$, we have
        $$
        \begin{aligned}
        & (\mathcal{I}-\eta \mathcal{T}) \circ \mathbf{A}=(\mathbf{I}-\frac{\eta}{2} \mathbf{H}) \mathbf{A}(\mathbf{I}-\frac{\eta}{2} \mathbf{H}) + \frac{\eta^2}{2b}\mathcal{M}-\frac{\eta^2}{4b}\widetilde{\mathcal{M}} \succeq 0 \\
        & (\mathcal{I}-\eta \widetilde{\mathcal{T}}) \circ \mathbf{A}=(\mathbf{I}-\frac{\eta}{2} \mathbf{H}) \mathbf{A}(\mathbf{I}-\frac{\eta}{2} \mathbf{H}) \succeq 0 .
        \end{aligned}
        $$
        \item[4.] The proof adheres to Lemma B.1 in \citet{zou2021benign}.
        \item[5.] For any finite PSD matrix $\mathbf{A}$, we have:
        $$\mathcal{T}^{-1} \circ \mathbf{A}=\eta \sum_{t=0}^{\infty}(\mathcal{I}-\eta \mathcal{T})^t \circ \mathbf{A}.$$
        It is evident that if the right-hand side exists, it must be PSD, owing to the fact that $\mathcal{I}-\eta \mathcal{T}$ is a PSD mapping. Demonstrating that the trace of $\sum_{t=0}^{\infty}(\mathcal{I}-\eta \mathcal{T})^t \circ \mathbf{A}$ is finite would suffice to establish the conclusion.
        $$
        \begin{aligned}
            \operatorname{tr}(\sum_{t=0}^{\infty}(\mathcal{I}-\eta \mathcal{T})^t \circ \mathbf{A})=\sum_{t=0}^{\infty} \operatorname{tr}((\mathcal{I}-\eta \mathcal{T})^t \circ \mathbf{A})= \sum_{t=0}^{\infty} \operatorname{tr}(\mathbf{A}_t)
        \end{aligned}
        $$
        By \cref{eq:recur_batch}, we have:
        \begin{align}\notag
        \operatorname{tr}(\mathbf{A}_t) &=\operatorname{tr}(\mathbf{A}_{t-1})-\eta \operatorname{tr}(\mathbf{H} \mathbf{A}_{t-1})+\frac{\eta^2}{2b} \operatorname{tr}(\mathbb{E}[\mathbf{x} \mathbf{x}^{\top} \mathbf{A}_{t-1} \mathbf{x} \mathbf{x}^{\top}] + \frac{b-1}{2} \mathbf{H} \mathbf{A}_{t-1} \mathbf{H})\\ \notag
        & \leq \operatorname{tr}(\mathbf{A}_{t-1})-\eta \operatorname{tr}(\mathbf{H} \mathbf{A}_{t-1}) + \frac{\eta^2}{2b} \operatorname{tr}(\mathbf{A}_{t-1} \alpha \operatorname{tr}(\mathbf{H})\mathbf{H} + \frac{b-1}{2} \operatorname{tr}(\mathbf{H})\mathbf{A}_{t-1}\mathbf{H })\\ \notag
        & \leq \operatorname{tr}(\mathbf{A}_{t-1})-\eta (1-\frac{\eta\alpha\operatorname{tr}(\mathbf{H}) }{2b} -\frac{\eta (b-1)\operatorname{tr}(\mathbf{H}) }{4b})\operatorname{tr}(\mathbf{H} \mathbf{A}_{t-1}) \\ \notag
        & \leq \operatorname{tr}(\mathbf{A}_{t-1})- \frac{\eta}{2}\operatorname{tr}(\mathbf{H} \mathbf{A}_{t-1}) \\ \notag
        & \leq (1-\frac{\eta}{2} \lambda_d) \operatorname{tr}(\mathbf{A}_{t-1}), 
        \end{align}
        where we use $\eta \leq \frac{2b}{2\alpha \operatorname{tr}(\mathbf{H}) + (b-1)\operatorname{tr}(\mathbf{H})}$ in the penultimate inequality.
        
        Hence, we have $\sum_{t=0}^{\infty} \operatorname{tr}(\mathbf{A}_t) \leq \frac{2\operatorname{tr}(\mathbf{A})}{\eta \lambda_d}<\infty$, which complete the proofs.
    \end{itemize}
\end{proof}
Now we are ready to provide the evolution of $\mathbf{A}_t$.

Consider the gradient norm not exceeding the clipping norm:
\begin{align*} 
\mathbf{w}_t -\mathbf{w}_*= & \mathbf{w}_{t-1}-\frac{\eta}{b} \sum_{i=1}^{b}(\mathbf{x}_{\tau(t)+m+i}(\operatorname{ReLU}(\mathbf{x}_{\tau(t)+m+i}^{\top} \mathbf{w}_{t})-y_{t,i})) -\frac{2 \eta \Gamma {f}}{b} \mathbf{g}_t-\mathbf{w}_* \\ 
= & \mathbf{w}_{t-1}-\frac{\eta}{b} \sum_{i=1}^{b}( \mathbbm{1}[\mathbf{x}_{\tau(t)+m+i}^{\top} \mathbf{w}_{t-1}>0] \cdot \mathbf{x}_{\tau(t)+m+i} \mathbf{x}_{\tau(t)+m+i}^{\top} \mathbf{w}_{t-1} \\
- &  \mathbbm{1}[\mathbf{x}_{\tau(t)+m+i}^{\top} \mathbf{w}_*>0] \cdot \mathbf{x}_{\tau(t)+m+i} \mathbf{x}_{\tau(t)+m+i}^{\top} \mathbf{w}_* )+\frac{\eta}{b} \sum_{i=1}^{b} z_t \mathbf{x}_{\tau(t)+m+i} - \frac{2 \eta \Gamma {f}}{b} \mathbf{g}_t \\
= & (\mathbf{I}-\frac{\eta}{b} \sum_{i=1}^{b}  \mathbbm{1}[\mathbf{x}_{\tau(t)+m+i}^{\top} \mathbf{w}_{t-1}>0] \mathbf{x}_{\tau(t)+m+i} \mathbf{x}_{\tau(t)+m+i}^{\top})(\mathbf{w}_{t-1}-\mathbf{w}_*) \\ \notag
& +\frac{\eta}{b} \sum_{i=1}^{b} (\mathbbm{1}[\mathbf{x}_{\tau(t)+m+i}^{\top} \mathbf{w}_*>0]-\mathbbm{1}[\mathbf{x}_{\tau(t)+m+i}^{\top} \mathbf{w}_{t-1}>0]) \mathbf{x}_{\tau(t)+m+i} \mathbf{x}_{\tau(t)+m+i}^{\top} \mathbf{w}_*+\frac{\eta}{b} \sum_{i=1}^{b}  z_t \mathbf{x}_{\tau(t)+m+i} - \frac{2 \eta \Gamma {f}}{b} \mathbf{g}_t .
\end{align*}

Let's consider the expected outer product:
\begin{equation}\label{eq:outer_mini}
    \begin{aligned}
        \underset{(\mathbf{x}, y) \sim \mathcal{D}}{\mathbb{E}}(\mathbf{w}_t-\mathbf{w}_*)^{\otimes 2} & = [\text{(quadratic term 1)} + \text{(quadratic term 2)}+ \text{(quadratic term 3)} \\
    & +\text{(crossing term 1)}+ \text {(crossing term 2)}],
    \end{aligned}
\end{equation}
where 
\begin{equation}
    \begin{aligned}
        \text{(quadratic term 1)} &= {\underset{(\mathbf{x}, y) \sim \mathcal{D}}{\mathbb{E}}(\mathbf{I}-\frac{\eta}{b} \sum_{i=1}^{b}  \mathbbm{1}[\mathbf{x}_{\tau(t)+m+i}^{\top} \mathbf{w}_{t-1}>0] \mathbf{x}_{\tau(t)+m+i} \mathbf{x}_{\tau(t)+m+i}^{\top})^{\otimes 2} \circ(\mathbf{w}_{t-1}-\mathbf{w}_*)^{\otimes 2}}\\
        \text{(quadratic term 2)} &= (\frac{\eta}{b})^2 \underset{(\mathbf{x}, y) \sim \mathcal{D}}{\mathbb{E}}  \sum_{i=1}^{b} (\mathbbm{1}[\mathbf{x}_{\tau(t)+m+i}^{\top} \mathbf{w}_*>0]-\mathbbm{1}[\mathbf{x}_{\tau(t)+m+i}^{\top} \mathbf{w}_{t-1}>0]) \mathbf{x}_{\tau(t)+m+i} \mathbf{x}_{\tau(t)+m+i}^{\top} \mathbf{w}_* \\
        & \cdot \sum_{j=1}^{b} (\mathbbm{1}[\mathbf{x}_{\tau(t)+m+j}^{\top} \mathbf{w}_*>0]-\mathbbm{1}[\mathbf{x}_{\tau(t)+m+j}^{\top} \mathbf{w}_{t-1}>0])\mathbf{w}_*^{\top}\mathbf{x}_{\tau(t)+m+j} \mathbf{x}_{\tau(t)+m+j}^{\top} \\
        \text{(quadratic term 3)} &=\frac{\eta^2}{b^2} \underset{(\mathbf{x}, y) \sim \mathcal{D}}{\mathbb{E}}  \{(\sum_{i=1}^{b}  z_t \mathbf{x}_{\tau(t)+m+i} - {2  \Gamma {f}}\mathbf{g}_t) (\sum_{j=1}^{b}  z_t \mathbf{x}_{\tau(t)+m+i} - {2 \Gamma {f}} \mathbf{g}_t)^{\top}\\
        \text{(crossing term 1)} & = \frac{\eta}{b}  \underset{(\mathbf{x}, y) \sim \mathcal{D}}{\mathbb{E}} \sum_{i=1}^{b} (\mathbbm{1}[\mathbf{x}_{\tau(t)+m+i}^{\top} \mathbf{w}_*>0]-\mathbbm{1}[\mathbf{x}_{\tau(t)+m+i}^{\top} \mathbf{w}_{t-1}>0]) \\
        & \cdot \mathbf{x}_{\tau(t)+m+i} \mathbf{x}_{\tau(t)+m+i}^{\top} \mathbf{w}_* (\mathbf{w}_{t-1}-\mathbf{w}_*)^{\top}(\mathbf{I}-\frac{\eta}{b} \sum_{j=1}^{b}  \mathbbm{1}[\mathbf{x}_{\tau(t)+m+j}^{\top} \mathbf{w}_{t-1}>0] \mathbf{x}_{\tau(t)+m+j} \mathbf{x}_{\tau(t)+m+j}^{\top})\\
        \text{(crossing term 2)} & = \frac{\eta}{b}\underset{(\mathbf{x}, y) \sim \mathcal{D}}{\mathbb{E}} (\mathbf{I}-\frac{\eta}{b} \sum_{i=1}^{b}  \mathbbm{1}[\mathbf{x}_{\tau(t)+m+i}^{\top} \mathbf{w}_{t-1}>0] \mathbf{x}_{\tau(t)+m+i} \mathbf{x}_{\tau(t)+m+i}^{\top}) (\mathbf{w}_{t-1}-\mathbf{w}_*)^{\top}\\
        & \cdot \frac{\eta}{b}\sum_{j=1}^{b} (\mathbbm{1}[\mathbf{x}_{\tau(t)+m+j}^{\top} \mathbf{w}_*>0]-\mathbbm{1}[\mathbf{x}_{\tau(t)+m+j}^{\top} \mathbf{w}_{t-1}>0]) \mathbf{x}_{\tau(t)+m+j} \mathbf{x}_{\tau(t)+m+j}^{\top} \mathbf{w}_* .
    \end{aligned}
\end{equation}

We will consider the above separately.

According to \cref{asm:symmetric} (where each $x$ is independent and symmetric), the following conditions hold when $i \neq j$:
\begin{equation}\label{eq:mini000}
    \begin{aligned}
    &\underset{(\mathbf{x}, y) \sim \mathcal{D}}{\mathbb{E}} (\mathbbm{1}[\mathbf{x}_{\tau(t)+m+i}^{\top} \mathbf{w}_*>0]-\mathbbm{1}[\mathbf{x}_{\tau(t)+m+i}^{\top} \mathbf{w}_{t-1}>0]) \mathbf{x}_{\tau(t)+m+i} \mathbf{x}_{\tau(t)+m+i}^{\top} \\
    &\cdot  (\mathbbm{1}[\mathbf{x}_{\tau(t)+m+j}^{\top} \mathbf{w}_*>0]-\mathbbm{1}[\mathbf{x}_{\tau(t)+m+j}^{\top} \mathbf{w}_{t-1}>0]) \mathbf{x}_{\tau(t)+m+j} \mathbf{x}_{\tau(t)+m+j}^{\top} = 0\\
    &\underset{(\mathbf{x}, y) \sim \mathcal{D}}{\mathbb{E}} (\mathbbm{1}[\mathbf{x}_{\tau(t)+m+i}^{\top} \mathbf{w}_*>0]-\mathbbm{1}[\mathbf{x}_{\tau(t)+m+i}^{\top} \mathbf{w}_{t-1}>0]) \mathbf{x}_{\tau(t)+m+i} \mathbf{x}_{\tau(t)+m+i}^{\top} \\
    & \cdot \mathbbm{1}[\mathbf{x}_{\tau(t)+m+j}^{\top} \mathbf{w}_{t-1}>0] \mathbf{x}_{\tau(t)+m+j} \mathbf{x}_{\tau(t)+m+j}^{\top} = 0.
\end{aligned}
\end{equation}
Substituting the above into the quadratic term 2 and the crossing terms, we will obtain
\begin{equation}\label{eq:quadratic_term2_mini}
    \begin{aligned}
        & \underset{(\mathbf{x}, y) \sim \mathcal{D}}{\mathbb{E}} \text { (quadratic term 2) } \\ 
        & =(\frac{\eta}{b})^2 \underset{(\mathbf{x}, y) \sim \mathcal{D}}{\mathbb{E}} ( \sum_{i=1}^{b}(\mathbbm{1}[\mathbf{x}_{\tau(t)+m+i}^{\top} \mathbf{w}_*>0]-\mathbbm{1}[\mathbf{x}_{\tau(t)+m+i}^{\top} \mathbf{w}_{t-1}>0])^2 \cdot(\mathbf{x}_{\tau(t)+m+i}^{\top} \mathbf{w}_*)^2 \cdot \mathbf{x}_{\tau(t)+m+i} \mathbf{x}_{\tau(t)+m+i}^{\top}) \\ 
        & =(\frac{\eta}{b})^2 \underset{(\mathbf{x}, y) \sim \mathcal{D}}{\mathbb{E}} ( \sum_{i=1}^{b} (\mathbbm{1}[\mathbf{x}_{\tau(t)+m+i}^{\top} \mathbf{w}_{t-1}>0, \mathbf{x}_{\tau(t)+m+i}^{\top} \mathbf{w}_*<0]+\mathbbm{1}[\mathbf{x}_{\tau(t)+m+i}^{\top} \mathbf{w}_{t-1}<0, \mathbf{x}_{\tau(t)+m+i}^{\top} \mathbf{w}_*>0]) \\
        & \cdot(\mathbf{x}_{\tau(t)+m+i}^{\top} \mathbf{w}_*)^2 \cdot \mathbf{x}_{\tau(t)+m+i} \mathbf{x}_{\tau(t)+m+i}^{\top}) \\ 
        & = 2(\frac{\eta}{b})^2\cdot \mathbb{E}( \sum_{i=1}^{b}\mathbbm{1}[\mathbf{x}_{\tau(t)+m+i}^{\top} \mathbf{w}_{t-1}>0, \mathbf{x}_{\tau(t)+m+i}^{\top} \mathbf{w}_*<0] \cdot(\mathbf{x}_{\tau(t)+m+i}^{\top} \mathbf{w}_*)^2 \cdot \mathbf{x}_{\tau(t)+m+i} \mathbf{x}_{\tau(t)+m+i}^{\top}),
    \end{aligned}
\end{equation}
where the first equality comes from
$$(\mathbbm{1}[\mathbf{x}_t^{\top} \mathbf{w}_*>0]-\mathbbm{1}[\mathbf{x}_t^{\top} \mathbf{w}_{t-1}>0])^2=\mathbbm{1}[\mathbf{x}_t^{\top} \mathbf{w}_{t-1}>0, \mathbf{x}_t^{\top} \mathbf{w}_*<0]+\mathbbm{1}[\mathbf{x}_t^{\top} \mathbf{w}_{t-1}<0, \mathbf{x}_t^{\top} \mathbf{w}_*>0].$$

For the crossing terms, we know
\begin{equation}\label{eq:crossing_term_mini}
    \begin{aligned}
        &  (\text { crossing term } 1)+(\text { crossing term } 2)  \\ 
        &=  \frac{2\eta}{b} \underset{(\mathbf{x}, y) \sim \mathcal{D}}{\mathbb{E}} [ \sum_{i=1}^{b}(\mathbbm{1}[\mathbf{x}_{\tau(t)+m+i}^{\top} \mathbf{w}_*>0]-\mathbbm{1}[\mathbf{x}_{\tau(t)+m+i}^{\top} \mathbf{w}_{t-1}>0]) \cdot(\mathbf{x}_{\tau(t)+m+i} \mathbf{x}_{\tau(t)+m+i}^{\top} \mathbf{w}_*(\mathbf{w}_{t-1}-\mathbf{w}_*)^{\top} \\
        & +(\mathbf{w}_{t-1}-\mathbf{w}_*) \mathbf{w}_*^{\top} \mathbf{x}_{\tau(t)+m+i} \mathbf{x}_{\tau(t)+m+i}^{\top}) ] \\ 
        & - \frac{2\eta^2}{b^2} \underset{(\mathbf{x}, y) \sim \mathcal{D}}{\mathbb{E}}[\sum_{i=1}^{b} (\mathbbm{1}[\mathbf{x}_{\tau(t)+m+i}^{\top} \mathbf{w}_*>0]-\mathbbm{1}[\mathbf{x}_{\tau(t)+m+i}^{\top} \mathbf{w}_{t-1}>0]) \cdot \mathbbm{1}[\mathbf{x}_{\tau(t)+m+i}^{\top} \mathbf{w}_{t-1}>0] \\
        & \cdot \mathbf{x}_{\tau(t)+m+i}^{\top} \mathbf{w}_* \cdot \mathbf{x}_{\tau(t)+m+i}^{\top}(\mathbf{w}_{t-1}-\mathbf{w}_*) \cdot \mathbf{x}_{\tau(t)+m+i} \mathbf{x}_{\tau(t)+m+i}^{\top} ]\\ 
        & =  \frac{2\eta^2}{b^2} \underset{(\mathbf{x}, y) \sim \mathcal{D}}{\mathbb{E}}[ \sum_{i=1}^{b}\mathbbm{1}[\mathbf{x}_{\tau(t)+m+i}^{\top} \mathbf{w}_{t-1}>0, \mathbf{x}_{\tau(t)+m+i}^{\top} \mathbf{w}_*<0] \cdot \mathbf{x}_{\tau(t)+m+i}^{\top} \mathbf{w}_* \cdot \mathbf{x}_{\tau(t)+m+i}^{\top}(\mathbf{w}_{t-1}-\mathbf{w}_*) \cdot \mathbf{x}_{\tau(t)+m+i} \mathbf{x}_{\tau(t)+m+i}^{\top} ].
    \end{aligned}
\end{equation}

Now, we turn our attention to the first quadratic term
\begin{align*}
    {\text {(quadratic term 1) }} & = \underset{(\mathbf{x}, y) \sim \mathcal{D}}{\mathbb{E}}(\mathbf{w}_{t-1}-\mathbf{w}_*)^{\otimes 2} - \frac{\eta}{b} ( \sum_{i=1}^{b}  \mathbbm{1}[\mathbf{x}_{\tau(t)+m+i}^{\top} \mathbf{w}_{t-1}>0] \mathbf{x}_{\tau(t)+m+i} \mathbf{x}_{\tau(t)+m+i}^{\top}) \circ (\mathbf{w}_{t-1}-\mathbf{w}_*)^{\otimes 2} \\ 
    & - \frac{\eta}{b}(\mathbf{w}_{t-1}-\mathbf{w}_*)^{\otimes 2}\circ  ( \sum_{i=1}^{b}  \mathbbm{1}[\mathbf{x}_{\tau(t)+m+i}^{\top} \mathbf{w}_{t-1}>0] \mathbf{x}_{\tau(t)+m+i} \mathbf{x}_{\tau(t)+m+i}^{\top}) \\
    & + \frac{\eta^2}{b^2} ( \sum_{i=1}^{b}  \mathbbm{1}[\mathbf{x}_{\tau(t)+m+i}^{\top} \mathbf{w}_{t-1}>0] \mathbf{x}_{\tau(t)+m+i} \mathbf{x}_{\tau(t)+m+i}^{\top}) \\
    &\cdot ( \sum_{j=1}^{b}  \mathbbm{1}[\mathbf{x}_{\tau(t)+m+j}^{\top} \mathbf{w}_{t-1}>0] \mathbf{x}_{\tau(t)+m+j} \mathbf{x}_{\tau(t)+m+j}^{\top}) \circ (\mathbf{w}_{t-1}-\mathbf{w}_*)^{\otimes 2} \\ 
    & = \underset{(\mathbf{x}, y) \sim \mathcal{D}}{\mathbb{E}}(\mathbf{w}_{t-1}-\mathbf{w}_*)^{\otimes 2} - \frac{\eta}{2} \mathbf{H}\circ (\mathbf{w}_{t-1}-\mathbf{w}_*)^{\otimes 2} -\frac{\eta}{2} (\mathbf{w}_{t-1}-\mathbf{w}_*)^{\otimes 2}\circ\mathbf{H} \\ 
    & +  \frac{\eta^2}{b} (\frac{1}{2}\mathcal{M} +(b-1) \frac{1}{4}\mathbf{H}^2)(\mathbf{w}_{t-1}-\mathbf{w}_*)^{\otimes 2}.    
\end{align*}

Applying \cref{eq:crossing_term_mini}) and \cref{eq:quadratic_term2_mini} to the expected outer product, we have the following recursion
\begin{equation}\label{eq:recursion_1}
    \begin{aligned}
        \mathbf{A}_t  &= \mathbf{A}_{t-1}-\frac{\eta}{2}(\mathbf{H} \mathbf{A}_{t-1}+\mathbf{A}_{t-1} \mathbf{H})+\frac{\eta^2}{b} (\frac{1}{2} \mathcal{M} +(b-1) \frac{1}{4}\mathbf{H}^2) \circ \mathbf{A}_{t-1}+\frac{\eta^2}{b} \sigma^2 \mathbf{H} + \frac{4 \eta^2 \Gamma^2 {f}^2}{b^2} \mathbf{I}\\\notag
        & + \frac{2\eta^2}{b^2} \underset{(\mathbf{x}, y) \sim \mathcal{D}}{\mathbb{E}}[ \sum_{i=1}^{b}\mathbbm{1}[\mathbf{x}_{\tau(t)+m+i}^{\top} \mathbf{w}_{t-1}>0, \mathbf{x}_{\tau(t)+m+i}^{\top} \mathbf{w}_*<0] \cdot \mathbf{x}_{\tau(t)+m+i}^{\top} \mathbf{w}_* \cdot \mathbf{x}_{\tau(t)+m+i}^{\top}\mathbf{w}_{t-1} \cdot \mathbf{x}_{\tau(t)+m+i} \mathbf{x}_{\tau(t)+m+i}^{\top} ]\\
        & \preceq \mathbf{A}_{t-1}-\frac{\eta}{2}(\mathbf{H} \mathbf{A}_{t-1}+\mathbf{A}_{t-1} \mathbf{H})+\frac{\eta^2}{b} ( \frac{1}{2}\mathcal{M} +(b-1) \frac{1}{4}\mathbf{H}^2) \circ \mathbf{A}_{t-1}+\frac{\eta^2}{b} \sigma^2 \mathbf{H} + \frac{4 \eta^2 \Gamma^2 {f}^2}{b^2} \mathbf{I}.
    \end{aligned}
\end{equation}
The indicator function shows $
\mathbbm{1}[\mathbf{x}_t^{\top} \mathbf{w}_{t-1}>0, \mathbf{x}_t^{\top} \mathbf{w}_*<0] \cdot \mathbf{x}_t^{\top} \mathbf{w}_* \cdot \mathbf{x}_t^{\top} \mathbf{w}_{t-1} \leq 0$ in the last inequation. 

Consequently, analogous to the one-sample case, we can decompose $\mathbf{A}_t$ as follows:
\begin{equation}\label{eq:decom_recur_mini}
     \left\{\begin{array} { l } 
    { 
    \mathbf{B}_t \preceq  (\mathbf{I}-\eta \mathcal{T} (\eta, b, \mathbf{H} ) ) \circ \mathbf{B}_{t-1}; 
    } \\
    { \mathbf{B}_0 =  (\mathbf{w}_0-\mathbf{w}_* )^{\otimes 2}}
    \end{array} \quad  \left\{\begin{array}{l}
    { 
    \mathbf{C}_t \preceq  (\mathbf{I}-\eta \mathcal{T} (\eta, b, \mathbf{H} ) ) \circ \mathbf{C}_{t-1} +  \frac{\eta^2}{b} \sigma^2 \mathbf{H} + \frac{4 \eta^2 {\Gamma}^2 {f}^2}{b^2} \mathbf{I}; 
    } \\
    { \mathbf{C}_0 = 0}
    \end{array}   \right. \right.
\end{equation}

\begin{lemma}\label{lemma:excess_risk_decom_mini}
    Suppose that \cref{asm:fourth-moment} and \cref{asm:symmetric} hold. For $\overline{\mathbf{w}}_N$ defined by previously, we have that
    $$
    \begin{aligned}
    & \underset{(\mathbf{x}, y) \sim \mathcal{D}}{\mathbb{E}} \langle\mathbf{H},(\overline{\mathbf{w}}_{s+1,T}-\mathbf{w}_*)^{\otimes 2}\rangle \leq \frac{1}{ N^2} \cdot \sum_{t=s+1}^{s+T} \sum_{k=t}^{s+T}\langle(\mathbf{I}- \frac{\eta}{2} \mathbf{H})^{k-t} \mathbf{H}, \mathbf{A}_t \rangle. 
    \end{aligned}
    $$
\end{lemma}
\begin{proof}[{\bf Proof}]
    We first focus on the expectation of $\mathbf{w}_t-\mathbf{w}_*$
    \begin{equation}
        \begin{aligned}
            \underset{(\mathbf{x}, y) \sim \mathcal{D}}{\mathbb{E}}[\mathbf{w}_t-\mathbf{w}_* \mid \mathbf{w}_{t-1}]= & \notag
            \underset{(\mathbf{x}, y) \sim \mathcal{D}}{\mathbb{E}}[(\mathbf{I}-\frac{\eta}{b} \sum_{i=1}^{b}  \mathbbm{1}[\mathbf{x}_{\tau(t)+m+i}^{\top} \mathbf{w}_{t-1}>0] \mathbf{x}_{\tau(t)+m+i} \mathbf{x}_{\tau(t)+m+i}^{\top})(\mathbf{w}_{t-1}-\mathbf{w}_*) \mid \mathbf{w}_{t-1}]\\\notag
            & + \underset{(\mathbf{x}, y) \sim \mathcal{D}}{\mathbb{E}}[\frac{\eta}{b} \sum_{i=1}^{b} (\mathbbm{1}[\mathbf{x}_{\tau(t)+m+i}^{\top} \mathbf{w}_*>0]-\mathbbm{1}[\mathbf{x}_{\tau(t)+m+i}^{\top} \mathbf{w}_{t-1}>0]) \mathbf{x}_{\tau(t)+m+i} \mathbf{x}_{\tau(t)+m+i}^{\top} \mathbf{w}_*\mid \mathbf{w}_{t-1}]\\
            &+ \underset{(\mathbf{x}, y) \sim \mathcal{D}}{\mathbb{E}}[\frac{\eta}{b} \sum_{i=1}^{b}  \varepsilon_t \mathbf{x}_{\tau(t)+m+i} - \frac{2 \eta s_t {f}}{b} \mathbf{g}_t \mid \mathbf{w}_{t-1}]\\\notag
            & = \underset{(\mathbf{x}, y) \sim \mathcal{D}}{\mathbb{E}}[(\mathbf{I}-\frac{\eta}{b} \sum_{i=1}^{b}  \mathbbm{1}[\mathbf{x}_{\tau(t)+m+i}^{\top} \mathbf{w}_{t-1}>0] \mathbf{x}_{\tau(t)+m+i} \mathbf{x}_{\tau(t)+m+i}^{\top})(\mathbf{w}_{t-1}-\mathbf{w}_*) \mid \mathbf{w}_{t-1}]\\\notag
             & = (\mathbf{I}-\frac{\eta}{2} \mathbf{H})(\mathbf{w}_{t-1}-\mathbf{w}_*).
        \end{aligned}
    \end{equation}
    
    Applying the aforementioned recursively, we deduce that, for $t \geq s$, $\underset{(\mathbf{x}, y) \sim \mathcal{D}}{\mathbb{E}}[\mathbf{w}_t-\mathbf{w}_* \mid \mathbf{w}_s]=(\mathbf{I}-\frac{\eta}{2} \mathbf{H})^{t-s}(\mathbf{w}_s-\mathbf{w}_*)$,
    which also implies that
    $$
    \underset{(\mathbf{x}, y) \sim \mathcal{D}}{\mathbb{E}}[(\mathbf{w}_t-\mathbf{w}_*) \otimes(\mathbf{w}_s-\mathbf{w}_*)]=(\mathbf{I}-\frac{\eta}{2} \mathbf{H})^{t-s} \cdot \mathbb{E}(\mathbf{w}_s-\mathbf{w}_*)^{\otimes 2}=(\mathbf{I}-\frac{\eta}{2} \mathbf{H})^{t-s} \cdot \mathbf{A}_s.
    $$
    Then, we consider the tail-averaged mini-batch SGD algorithm and we denote $\overline{\mathbf{w}}_{s+1, T}- \mathbf{w}_*=\frac{1}{T} \sum_{t=s+1}^{s+T} \mathbf{w}_t- \mathbf{w}_*$:
    \begin{equation}\label{eq:excess_mini_decom}
        \begin{aligned}
            \underset{(\mathbf{x}, y) \sim \mathcal{D}}{\mathbb{E}}[(\overline{\mathbf{w}}_{s+1, T}- \mathbf{w}_*)^{\otimes 2}] & =\frac{1}{T^2} \sum_{t=s+1}^{s+T} \sum_{k=s+1}^{s+T} \underset{(\mathbf{x}, y) \sim \mathcal{D}}{\mathbb{E}}[({\mathbf{w}}_{t}- \mathbf{w}_*) \otimes ({\mathbf{w}}_{k}- \mathbf{w}_*)] \\
            & =\frac{1}{T^2} \cdot(\sum_{t \geq k} \underset{(\mathbf{x}, y) \sim \mathcal{D}}{\mathbb{E}}[({\mathbf{w}}_{t}- \mathbf{w}_*) \otimes ({\mathbf{w}}_{k}- \mathbf{w}_*)]+\sum_{t \leq k} \mathbb{E}[({\mathbf{w}}_{t}- \mathbf{w}_*) \otimes ({\mathbf{w}}_{k}- \mathbf{w}_*)]) \\
            & \preceq \frac{1}{T^2}  \cdot(\sum_{t \geq k} \underset{(\mathbf{x}, y) \sim \mathcal{D}}{\mathbb{E}}[({\mathbf{w}}_{t}- \mathbf{w}_*) \otimes ({\mathbf{w}}_{k}- \mathbf{w}_*)]+\sum_{t \leq k} \mathbb{E}[({\mathbf{w}}_{t}- \mathbf{w}_*) \otimes ({\mathbf{w}}_{k}- \mathbf{w}_*)])\\ 
            & =\frac{1}{T^2} \cdot(\sum_{t \geq k}(\mathbf{I}-\frac{\eta}{2}  \mathbf{H})^{t-k} \mathbf{A}_k+\sum_{t \leq k} \mathbf{A}_t(\mathbf{I}-\frac{\eta}{2}  \mathbf{H})^{k-t}) \\
            & =\frac{1}{T^2} \cdot \sum_{t \leq k} (\mathbf{A}_t(\mathbf{I}-\frac{\eta}{2}  \mathbf{H})^{k-t}+(\mathbf{I}-\frac{\eta}{2}  \mathbf{H})^{k-t} \mathbf{A}_t ) \\
            & =\frac{1}{T^2} \cdot \sum_{t=s+1}^{s+T} \sum_{k=t}^{s+T}(\mathbf{A}_t(\mathbf{I}-\frac{\eta}{2}  \mathbf{H})^{k-t}+(\mathbf{I}-\frac{\eta}{2}  \mathbf{H})^{k-t} \mathbf{A}_t).
        \end{aligned}
    \end{equation}

    Then we consider the excess risk of tail-averaged mini-batch:
    $$
    \begin{aligned}
    \underset{(\mathbf{x}, y) \sim \mathcal{D}}{\mathbb{E}} \langle\mathbf{H},(\overline{\mathbf{w}}_{s+1,T}-\mathbf{w}_*)^{\otimes 2}\rangle & \leq \langle\mathbf{H}, \frac{1}{T^2} \cdot \sum_{t=s+1}^{s+T} \sum_{k=t}^{s+T}(\mathbf{A}_t(\mathbf{I}-\frac{\eta}{2}  \mathbf{H})^{k-t}+(\mathbf{I}-\frac{\eta}{2}  \mathbf{H})^{k-t} \mathbf{A}_t)\rangle\\
    & = \frac{1}{T^2} \cdot \sum_{t=s+1}^{s+T} \sum_{k=t}^{s+T} \langle \mathbf{H},  \mathbf{A}_t(\mathbf{I}-\frac{\eta}{2}  \mathbf{H})^{k-t}\rangle + \frac{1}{T^2} \cdot \sum_{t=s+1}^{s+T} \sum_{k=t}^{s+T} \langle \mathbf{H},  (\mathbf{I}-\frac{\eta}{2}  \mathbf{H})^{k-t}\mathbf{A}_t\rangle \\
    & \leq \frac{1}{\eta T^2} \cdot \sum_{t=s+1}^{s+T} \langle \mathbf{I}- (\mathbf{I}- \frac{\eta}{2}\mathbf{H} )^{T}, \mathbf{A}_t\rangle.
    \end{aligned}
    $$
\end{proof}

\textbf{Variance Error}
\begin{lemma}\label{lem:ct_mini}
    Suppose Assumptions hold. Suppose $\eta< \min \{\frac{1}{{R_x^2}}, \frac{4b}{2{R_x^2}+(b-1)\|\mathbf{H}\|_2}\}$. Then for every $t$ we have
    $$
    \mathbf{C}_t \preceq \frac{4\eta \sigma^2}{4b - 2\eta {R_x^2} - \eta  (b-1 )\|\mathbf{H}\|_2} \mathbf{I} + \frac{16\eta {\Gamma}^2 {f}^2}{b (4b - 2\eta {R_x^2} - \eta  (b-1 )\|\mathbf{H}\|_2 )} \mathbf{H}^{-1}.
    $$
\end{lemma}

\begin{proof}[{\bf Proof}]
    We proceed with induction. 
    
    For $t=0$ we have $\mathbf{C}_0=0 \preceq \frac{4\eta \sigma^2}{4b - 2\eta {R_x^2} - \eta  (b-1 )\|\mathbf{H}\|_2} \mathbf{I} + \frac{16\eta {\Gamma}^2 {f}^2}{b (4b - 2\eta {R_x^2} - \eta  (b-1 )\|\mathbf{H}\|_2 )} \mathbf{H}^{-1}$. 
    
    We then assume that $\mathbf{C}_{t-1} $ holds for \cref{lem:ct_mini}, and exam $\mathbf{C}_t$ based on \cref{eq:decom_recur_mini}
    \begin{align*}
        \mathbf{C}_t & \preceq  (\mathbf{I}-\eta \mathcal{T} (\eta, b, \mathbf{H} ) ) \circ \mathbf{C}_{t-1} + \frac{\eta^2}{b} \sigma^2 \mathbf{H} + \frac{4 \eta^2 {\Gamma}^2 {f}^2}{b^2} \mathbf{I} \\ 
        & =\mathbf{C}_{t-1}-\frac{\eta}{2} (\mathbf{H} \mathbf{C}_{t-1}+\mathbf{C}_{t-1} \mathbf{H} )+\frac{\eta^2}{b}  ( \frac{1}{2}\mathcal{M} +(b-1) \frac{1}{4}\mathbf{H}^2 ) \circ \mathbf{C}_{t-1}+ \frac{\eta^2 \sigma^2 }{b}\mathbf{H} +\frac{4 \eta^2 {\Gamma}^2 {f}^2}{b^2} \mathbf{I} \\ 
        & \preceq \frac{4\eta \sigma^2}{4b - 2\eta {R_x^2} - \eta  (b-1 )\|\mathbf{H}\|_2} \mathbf{I} + \frac{16\eta {\Gamma}^2 {f}^2}{b (4b - 2\eta {R_x^2} - \eta  (b-1 )\|\mathbf{H}\|_2 )} \mathbf{H}^{-1}\\
        & -\eta (\frac{4\eta \sigma^2}{4b - 2\eta {R_x^2} - \eta  (b-1 )\|\mathbf{H}\|_2} \mathbf{H} + \frac{16\eta {\Gamma}^2 {f}^2}{b (4b - 2\eta {R_x^2} - \eta  (b-1 )\|\mathbf{H}\|_2 )} \mathbf{I} ) \\ 
        & + (\frac{\eta^2{R_x^2}}{2b} + \frac{\eta^2  (b-1 )\|\mathbf{H} \|_2}{4b} ) (\frac{4\eta \sigma^2}{4b - 2\eta {R_x^2} - \eta  (b-1 )\|\mathbf{H}\|_2} \mathbf{H} + \frac{16\eta {\Gamma}^2 {f}^2}{b (4b - 2\eta {R_x^2} - \eta  (b-1 )\|\mathbf{H}\|_2 )} \mathbf{I} )\\ \notag
        &+ \frac{\eta^2 \sigma^2 }{b}\mathbf{H}+\frac{4 \eta^2 {\Gamma}^2 {f}^2}{b^2} \mathbf{I}\\ 
        & \preceq \frac{4\eta \sigma^2}{4b - 2\eta {R_x^2} - \eta  (b-1 )\|\mathbf{H}\|_2} \mathbf{I} + \frac{16\eta {\Gamma}^2 {f}^2}{b (4b - 2\eta {R_x^2} - \eta  (b-1 )\|\mathbf{H}\|_2 )} \mathbf{H}^{-1}.
    \end{align*}
\end{proof}
For the simplicity, we define $\mathbf{\Sigma}:= \frac{\sigma^2}{b}\mathbf{H} + \frac{4{\Gamma}^2{f}^2}{b^2}\mathbf{I}$ and $\mu_b :=  (4b - 2\eta {R_x^2} - \eta  (b-1 )\|\mathbf{H}\|_2 )$. By the definitions of $\mathcal{T}$ and $\widetilde{\mathcal{T}}$, we have
\begin{align} \notag \label{eq:ct_mini_up}
\mathbf{C}_t & = (\mathbf{I}-\eta \mathcal{T} (\eta, b, \mathbf{H} ) ) \circ \mathbf{C}_{t-1}+\eta^2 \mathbf{\Sigma} \\
& =  (\mathbf{I}-\frac{\eta}{2}\mathbf{H} ) \mathbf{C}_{t-1} (\mathbf{I}-\frac{\eta}{2}\mathbf{H} ) + (\frac{\eta^2}{2b}\mathcal{M} - \frac{\eta^2}{4b} \mathbf{HH} ) \circ \mathbf{C}_{t-1}+\eta^2 \bm{\Sigma}\\ \notag
& \preceq  (\mathbf{I}-\frac{\eta}{2}\mathbf{H} ) \mathbf{C}_{t-1} (\mathbf{I}-\frac{\eta}{2}\mathbf{H} ) +\frac{\eta^2{R_x^2}}{2b}\mathbf{H}  \circ \mathbf{C}_{t-1}+\eta^2 \bm{\Sigma}.
\end{align}
Then by \cref{lem:ct_mini}, we have for all $t \geq 0$,
\begin{align*}
    \mathbf{H} \circ \mathbf{C}_t \preceq \mathbf{H} \circ  ( \frac{4\eta \sigma^2}{\mu_b} \mathbf{I} + \frac{16\eta {\Gamma}^2 {f}^2}{b \mu_b} \mathbf{H}^{-1} )\preceq \frac{4\eta \sigma^2}{\mu_b} \mathbf{H} + \frac{16\eta {\Gamma}^2 {f}^2}{b\mu_b} \mathbf{I}.
\end{align*}
Substituting the above into \cref{eq:ct_mini_up}, it holds that
\begin{align*}
    \mathbf{C}_t & \preceq  (\mathbf{I}-\frac{\eta}{2}\mathbf{H} ) \mathbf{C}_{t-1} (\mathbf{I}-\frac{\eta}{2}\mathbf{H} ) +\frac{\eta^2{R_x^2}}{2b} ( \frac{4\eta \sigma^2}{\mu_b} \mathbf{H} + \frac{16\eta {\Gamma}^2 {f}^2}{b\mu_b} \mathbf{I} )+\eta^2   (\frac{\sigma^2}{b}\mathbf{H} + \frac{4{\Gamma}^2{f}^2}{b^2}\mathbf{I}  )\\ 
    & \preceq  (\mathbf{I}-\frac{\eta}{2}\mathbf{H} ) \mathbf{C}_{t-1} (\mathbf{I}-\frac{\eta}{2}\mathbf{H} ) + \frac{\eta^2\sigma^2 (\mu_b -2\eta {R_x^2} )}{b\mu_b} \mathbf{H} + \frac{4\eta^2 {\Gamma}^2 {f}^2  (\mu_b -2\eta {R_x^2} )}{b^2\mu_b} \mathbf{I},
\end{align*}
which implies that
\begin{align*}
    \mathbf{C}_t& \preceq \frac{\eta^2  (\mu_b -2\eta {R_x^2} ) }{b\mu_b}  \cdot \sum_{k=0}^{t-1}(\mathbf{I}-\frac{\eta}{2}\mathbf{H})^k  ( {\sigma^2}\mathbf{H} + \frac{4{\Gamma}^2{f}^2}{b} \mathbf{I} )(\mathbf{I}-\frac{\eta}{2}\mathbf{H})^k  \\ 
    & \preceq \frac{\eta^2  (\mu_b -2\eta {R_x^2} ) }{b\mu_b}   \cdot \sum_{k=0}^{t-1}(\mathbf{I}-\frac{\eta}{2}\mathbf{H})^k  ( {\sigma^2} \mathbf{H} + \frac{4{\Gamma}^2{f}^2}{b} \mathbf{I} ) \\ 
    & =\frac{\eta \sigma^2  (\mu_b -2\eta {R_x^2} ) }{b\mu_b}  \cdot (\mathbf{I}-(\mathbf{I}-\frac{\eta}{2}\mathbf{H})^t ) + \frac{4\eta {\Gamma}^2 {f}^2   (\mu_b -2\eta {R_x^2} ) }{b^2\mu_b}   \cdot (\mathbf{I}-(\mathbf{I}-\frac{\eta}{2}\mathbf{H})^t ) \cdot \mathbf{H}^{-1}.
\end{align*}

Consequently, the variance error can be represented as follows, in accordance with \cref{lemma:excess_risk_decom_mini}:
\begin{align} \notag
    \text{variance error} &= \frac{1}{\eta T^2} \cdot \sum_{t=s+1}^{s+T}  \langle \mathbf{I}-  (\mathbf{I}- \frac{\eta}{2}\mathbf{H}  )^{T}, \mathbf{C}_t \rangle \\ \notag
    & \leq \frac{1}{\eta T^2} \cdot \frac{\eta \sigma^2  (\mu_b -2\eta {R_x^2} ) }{b\mu_b}  \cdot \sum_{t=s+1}^{s+T}  \langle \mathbf{I}-  (\mathbf{I}- \frac{\eta}{2}\mathbf{H}  )^{T},  (\mathbf{I}-(\mathbf{I}-\frac{\eta}{2}\mathbf{H})^t ) \rangle \\ 
    & + \frac{1}{\eta T^2} \cdot  \frac{4\eta {\Gamma}^2 {f}^2   (\mu_b -2\eta {R_x^2} ) }{b^2\mu_b}   \cdot \sum_{t=s+1}^{s+T}  \langle \mathbf{I}-  (\mathbf{I}- \frac{\eta}{2}\mathbf{H}  )^{T},  (\mathbf{I}-(\mathbf{I}-\frac{\eta}{2}\mathbf{H})^t ) \mathbf{H}^{-1} \rangle \\ \notag
    & \lesssim \frac{\sigma^2 d}{Tb} + \frac{\Gamma^2 f^2 \operatorname{tr}(\mathbf{H})}{T b^2} \lesssim \frac{\sigma^2 d}{N} + \frac{d^2 \log N  \log (1/\delta)}{N^2 \varepsilon^2} \cdot C_2^2 \kappa^2(\sigma^2+\left\|\mathbf{w}_*\right\|_{\mathbf{H}}^2+\Delta^2) .
\end{align}

\textbf{Bias Error}
According to \cref{lemma:excess_risk_decom_mini}, the bias error of tail average iterate follows that
\begin{align*}
    \text{bias error} \leq \frac{1}{\eta T^2} \cdot \sum_{t=s+1}^{s+T} \langle \mathbf{I}- (\mathbf{I}- \frac{\eta}{2}\mathbf{H} )^{T}, \mathbf{B}_t\rangle \leq \sum_{t=s+1}^{s+T} \frac{1}{\eta T^2} \operatorname{tr}(\mathbf{B}_t).
\end{align*}

Considering the recursion of $\mathbf{B}_t$, we have 
\begin{align*}
    \mathbf{B}_t &\preceq \mathbf{B}_{t-1}-\frac{\eta}{2}(\mathbf{H} \mathbf{B}_{t-1}+\mathbf{B}_{t-1} \mathbf{H})+\frac{\eta^2}{b}(\frac{1}{2} \mathcal{M}+(b-1) \frac{1}{4} \mathbf{H}^2) \circ \mathbf{B}_{t-1}\\
    & \preceq \mathbf{B}_{t-1}-{\eta}\mathbf{H} \mathbf{B}_{t-1} + \frac{\eta^2 }{b} (R_x^2 + (b-1)\|\mathbf{H}\|_2) \mathbf{B}_{t-1} \\
    & \preceq (\mathbf{I}- \frac{\eta}{2} \mathbf{H}) \mathbf{B}_{t-1}.
\end{align*}
The last inequality derives from the choice of step size.
Consequently, the bias error will be
\begin{align*}
    \text{bias error} \leq \sum_{t=0}^N\frac{1}{\eta N^2} \operatorname{tr}(\mathbf{B}_t) \leq \sum_{t=0}^N\frac{1}{\eta N^2} (1- \frac{\eta\mu}{2})^t\operatorname{tr}(\mathbf{B}_{0}) \lesssim \frac{1}{\eta N}\|\mathbf{w}_*\|_{\mathbf{H}}^2 \lesssim \frac{d\log^2(N/\delta)}{N^2 \varepsilon^2} \cdot C_2^2 \kappa^2 \|\mathbf{w}_*\|_{\mathbf{H}}^2.
\end{align*}

Combining the previous variance error, we complete the proof.

\section{Lower Bound}

\begin{proof}[\bf Proof of \cref{thm:low_1}]
    We will denote $T_i= \mathcal{T}_\mathbf{w}((\mathbf{x}_i,y_i), M(D))$ and $T_i^{\prime}=\mathcal{T}_\mathbf{w}((\mathbf{x}_i,y_i), M(D^{\prime}_i))$. Since we have $y-\text{ReLU}(\mathbf{w}^\top\mathbf{x})=z$ and $\mathbf{x} \cdot \mathbbm{1}(\mathbf{w}^\top \mathbf{x}>0)$, and $M(D_i^{\prime})-\mathbf{w}$ are independent, we have 
    \begin{equation}
        \mathbb{E}[T_i^{\prime}]=\mathbb{E} (y-\text{ReLU}(\mathbf{w}^\top \mathbf{x})) \mathbb{E} \langle M(D_i^{\prime})-\mathbf{w}, \mathbf{x}\cdot \mathbbm{1}(\mathbf{w}^\top \mathbf{x}>0)\rangle=0.
    \end{equation}
Moreover, we have 
\begin{equation}
    \mathbb{E}[T_i^{\prime}]\leq \sqrt{\mathbb{E}[{T_i^{\prime}}^2]}\leq \sigma \sqrt{\mathbb{E}\|M(D_i^{\prime})-\mathbf{w}\|^2_{\Sigma_\mathbf{x}}}= \sigma \sqrt{\mathbb{E}\|M(D)-\mathbf{w}\|^2_{\Sigma_\mathbf{x}}}. 
\end{equation}
For the second part, we have 
\begin{equation*}
    \sum_{i\in [n]}\mathbb{E}[T_i]=\sum_{j=1}^d \mathbb{E}M(D)_j\sum_{i=1}^N (y_i-\text{ReLU}(\mathbf{w}^\top \mathbf{x}_i))\mathbbm{1}(\mathbf{w}^\top \mathbf{x}>0)\mathbf{x}_{i,j}
\end{equation*}
For each $j$ we have 
\begin{multline*}
    \mathbb{E}M(D)_j\sum_{i=1}^N (y_i-\text{ReLU}(\mathbf{w}^\top \mathbf{x}_i))\mathbbm{1}(\mathbf{w}^\top \mathbf{x}>0)\mathbf{x}_{i,j} =\sigma^2 \mathbb{E}M(D)_j\frac{\partial \log f_\mathbf{w}(Y|X)}{\partial \mathbf{w}_j}=\sigma^2 \frac{\partial}{\partial \mathbf{w}_j} \mathbb{E}_{Y, X|\mathbf{w}} M(D)_j. 
\end{multline*}
Thus we have 
\begin{equation*}
    \sum_{i\in [n]}\mathbb{E}[T_i]=\sum_{j=1}^d\sigma^2 \frac{\partial}{\partial \mathbf{w}_j} \mathbb{E}_{Y, X|\mathbf{w}} M(D)_j. 
\end{equation*}
Recall the following Stein's lemma: 
\begin{lemma}
    Let $Z$ be distributed according to some density $p(z)$ that is continuously differentiable w.r.t. $z$ and let $h$ be a differentiable function such that $\mathbb{E}|h'(Z)|<\infty$. We have 
    \begin{equation*}
        \mathbb{E}[h'(Z)]=\mathbb{E}[-\frac{h(Z)p'(Z)}{p(Z)}]. 
    \end{equation*}
\end{lemma}
Consider the following prior distribution $\pi$ for $\mathbf{w}$: let $v_1,\cdots, v_d$ be i.i.d. sampled from the truncated $\mathcal{N}(0, 1)$ with truncation at $-1$ and $1$, and let $\mathbf{w}_j=\frac{v_j}{\sqrt{d}}$ thus $\mathbf{w}\in \mathcal{W}$. Denote $g_j(\mathbf{w})= \mathbb{E}_{Y, X|\mathbf{w}} M(D)_j$. For each $j\in [d]$ by using the above lemma we have 
\begin{align*}
    \mathbb{E}_\pi \frac{\partial}{\partial \mathbf{w}_j} g_j(\mathbf{w}) &=\mathbb{E}_\pi \frac{\partial}{\partial \mathbf{w}_j} \mathbb{E} (g_j(\mathbf{w})|\mathbf{w}_j)
    \geq  \mathbb{E}_\pi (\frac{-\mathbf{w}_j \pi^{\prime}_j(\mathbf{w}_j) }{\pi_j(\mathbf{w}_j)}-\mathbb{E}(|g_j(\mathbf{w})-\mathbf{w}_j| |\frac{\pi'_j(\mathbf{w}_j)}{\pi_j(\mathbf{w}_j)}|). 
\end{align*}
Since $\pi_j$ is a truncated normal distribution, we can easily get $\frac{\pi^{\prime}_j(\mathbf{w}_j)}{\pi_j(\mathbf{w}_j)}=-d \mathbf{w}_j.$
Therefore, it holds that
\begin{align*}
    &\mathbb{E}_\pi \sum_{j=1}^d (\frac{-\mathbf{w}_j \pi^{\prime}_j(\mathbf{w}_j) }{\pi_j(\mathbf{w}_j)}-\mathbb{E}(|g_j(\mathbf{w})-\mathbf{w}_j| |\frac{\pi^{\prime}_j(\mathbf{w}_j)}{\pi_j(\mathbf{w}_j)}|) 
    = \mathbb{E}_\pi [d \sum_{j=1}^d \mathbf{w}_j^2]-\sum_{j=1}^d \mathbb{E}_\pi(|g_j(\mathbf{w})-\mathbf{w}_j| d |\mathbf{w}_j|)\\
    &\geq  d\{\mathbb{E}_\pi [ \sum_{j=1}^d \beta_j^2]-\sqrt{\mathbb{E}_\pi \mathbb{E}_{Y, X|\mathbf{w}} \|M(D)-\mathbf{w}\|_2^2 }\sqrt{ \mathbb{E}_\pi \sum_{j=1}^d \mathbf{w}_j^2}\}. 
\end{align*}
As $\mathbb{E}_\pi [ \sum_{j=1}^d \mathbf{w}_j^2]=\mathcal{W}(1)$, in total we have 
\begin{equation*}
     \sum_{i\in [n]}\mathbb{E}[T_i]\geq O(\sigma^2 d \{1
     -\sqrt{\mathbb{E}_\pi \mathbb{E}_{Y, X|\mathbf{w}} \|M(D)-\mathbf{w}\|_2^2 } \}
\end{equation*}
We have the proof under the assumption that $\mathbb{E}_\pi \mathbb{E}_{Y, X|\mathbf{w}} 
 \|M(D)-\mathbf{w}\|_2^2 =o(1)$. 
\end{proof}

\begin{proof}
    We first prove the following lemma, whose proof is the same as the proof of  Lemma B.2 in \cite{cai2021cost}. 
    \begin{lemma}\label{lemma:aux1}
        For all $i\in [n]$, if $M$ is $(\varepsilon,\delta)$-DP then for every $T>0$
        \begin{equation}
            \mathbb{E}[T_i]\leq \mathbb{E}[T^{\prime}_i]+2\varepsilon \mathbb{E}[|T^{\prime}_i|]+2\delta T+\int_{T}^\infty \mathbb{P}(|A_i|\geq t).
        \end{equation}
    \end{lemma}
    By the above lemma, we have 
    \begin{equation}
        \mathbb{E}_{Y, X|\mathbf{w}} \sum_{i=1}^N T_i\leq 2n\varepsilon \sigma \sqrt{  \mathbb{E}_{Y, X|\mathbf{w}}\|M(D)-\mathbf{w}\|_{\Sigma_x}^2}+2n\delta T + n \int_{T}^\infty \mathbb{P}(|T_i|\geq t). 
    \end{equation}
    For the last term, we have 
    \begin{align*}
        \mathbb{P}(|T_i|\geq t)&= \mathbb{P}(|(y_i-\text{ReLU}(\mathbf{w}^\top \mathbf{x}_i))| |\langle M(D)-\mathbf{w}, x_i\mathbbm{1}(w^\top x_i>0)\rangle|>t)\\
        &\leq \mathbb{P}(|(y_i-\text{ReLU}(\mathbf{w}^\top \mathbf{x}_i))| |\langle M(D)-\mathbf{w}, \mathbf{x}_i \rangle|>t)\\
        &\leq  \mathbb{P}(|(y_i-\text{ReLU}(\mathbf{w}^\top \mathbf{x}_i))| \sqrt{d}\rangle|>t)\\
        &\leq 2\exp(-\frac{-t^2}{2d\sigma^2}). 
    \end{align*}
    Choosing $T=\sqrt{2}\sigma \sqrt{d\log(1/\delta)}$ we have
    \begin{align*}
        O(\sigma^2 d) &\leq   \mathbb{E}_{Y, X|\mathbf{w}} \sum_{i=1}^N T_i\\
        &\leq  2n\varepsilon \sigma \sqrt{  \mathbb{E}_{Y, X|\mathbf{w}}\|M(D)-\mathbf{w}\|_{\Sigma_x}^2}+O(\sigma n\delta \sqrt{d\log(1/\delta)}).
    \end{align*}
    Thus we have the result when $\delta\leq N^{-(1+u)}$ for large enough $u$. 

    Next we will show that  $\mathcal{L}(M(D))-\mathcal{L}(\mathbf{w}_*)\geq  \frac{1}{4} \|M(D)-\mathbf{w}_*\|_{\Sigma_\mathbf{x}}^2$. Specifically, we will show for any $\mathbf{w}$, $\mathcal{L}({\mathbf{w}})-\mathcal{L}(\mathbf{w}_*)\geq \frac{1}{4} \|{\mathbf{w}}-\mathbf{w}_*\|_{\Sigma_\mathbf{x}}^2$. Under the well-specified condition, we can easily see that 
    \begin{align*}
        \mathcal{L}({\mathbf{w}})-\mathcal{L}(\mathbf{w}_*)&=\mathbb{E}[\text{ReLU}(\mathbf{x}^\top {\mathbf{w}})-\text{ReLU}(\mathbf{x}^\top \mathbf{w}_*)]^2\\
        & =\mathbb{E}(\mathbf{x}^{\top} \mathbf{w} \cdot \mathbbm{1}[\mathbf{x}^{\top} \mathbf{w}>0]-\mathbf{x}^{\top} \mathbf{w}_* \cdot \mathbbm{1}[\mathbf{x}^{\top} \mathbf{w}_*>0])^2 \\
        &=  \mathbb{E}[\mathbf{w}^{\top} \mathbf{x} \mathbf{x}^{\top} \mathbf{w} \cdot \mathbbm{1}[\mathbf{x}^{\top} \mathbf{w}>0]]+\mathbb{E}[\mathbf{w}_*^{\top} \mathbf{x} \mathbf{x}^{\top} \mathbf{w}_* \cdot \mathbbm{1}[\mathbf{x}^{\top} \mathbf{w}_*>0]] \\
        & -2 \mathbb{E}[\mathbf{w}^{\top} \mathbf{x x}^{\top} \mathbf{w}_* \cdot \mathbbm{1}[\mathbf{x}^{\top} \mathbf{w}>0, \mathbf{x}^{\top} \mathbf{w}_*>0]] .
    \end{align*}
    
    According to \cref{asm:sym_restated}, it further implied that 
    \begin{align*}
    & \mathbb{E}(\operatorname{ReLU}(\mathbf{x}^{\top} \mathbf{w})-\operatorname{ReLU}(\mathbf{x}^{\top} \mathbf{w}_*))^2 \\
    &= \mathbb{E}[\mathbf{w}^{\top} \mathbf{x} \mathbf{x}^{\top} \mathbf{w} \cdot \mathbbm{1}[\mathbf{x}^{\top} \mathbf{w}<0]]+\mathbb{E}[\mathbf{w}_*^{\top} \mathbf{x x}^{\top} \mathbf{w}_* \cdot \mathbbm{1}[\mathbf{x}^{\top} \mathbf{w}_*<0]] \\
    &-2 \mathbb{E}[\mathbf{w}^{\top} \mathbf{x} \mathbf{x}^{\top} \mathbf{w}_* \cdot \mathbbm{1}[\mathbf{x}^{\top} \mathbf{w}<0, \mathbf{x}^{\top} \mathbf{w}_*<0]] \\
    &= \mathbb{E}(\operatorname{ReLU}(-\mathbf{x}^{\top} \mathbf{w})-\operatorname{ReLU}(-\mathbf{x}^{\top} \mathbf{w}_*))^2 .
    \end{align*}
    Moreover, we have:
    $$
    \begin{aligned}
    (\mathbf{x}^{\top} \mathbf{w}-\mathbf{x}^{\top} \mathbf{w}_*)^2  & =(\mathbf{x}^{\top} \mathbf{w} \mathbbm{1}[\mathbf{x}^{\top} \mathbf{w}>0]-\mathbf{x}^{\top} \mathbf{w}_* \mathbbm{1}[\mathbf{x}^{\top} \mathbf{w}_*>0]+\mathbf{x}^{\top} \mathbf{w} \mathbbm{1}[\mathbf{x}^{\top} \mathbf{w}<0]-\mathbf{x}^{\top} \mathbf{w}_* \mathbbm{1}[\mathbf{x}^{\top} \mathbf{w}_*<0])^2 \\
    & \leq 2(\mathbf{x}^{\top} \mathbf{w} \mathbbm{1}[\mathbf{x}^{\top} \mathbf{w}>0]-\mathbf{x}^{\top} \mathbf{w}_* \mathbbm{1}[\mathbf{x}^{\top} \mathbf{w}_*>0])^2+2(\mathbf{x}^{\top} \mathbf{w} \mathbbm{1}[\mathbf{x}^{\top} \mathbf{w}<0]-\mathbf{x}^{\top} \mathbf{w}_* \mathbbm{1}[\mathbf{x}^{\top} \mathbf{w}_*<0])^2 \\
    & =2(\operatorname{ReLU}(\mathbf{x}^{\top} \mathbf{w})-\operatorname{ReLU}(\mathbf{x}^{\top} \mathbf{w}_*))^2+2(\operatorname{ReLU}(-\mathbf{x}^{\top} \mathbf{w})-\operatorname{ReLU}(-\mathbf{x}^{\top} \mathbf{w}_*))^2.
    \end{aligned}
    $$
    Then taking an expectation on both sides we obtain that
    $$
    \begin{aligned}
    \mathbb{E}(\mathbf{x}^{\top} \mathbf{w}-\mathbf{x}^{\top} \mathbf{w}_*)^2 & \leq 2 \mathbb{E}(\operatorname{ReLU}(\mathbf{x}^{\top} \mathbf{w})-\operatorname{ReLU}(\mathbf{x}^{\top} \mathbf{w}_*))^2+2 \mathbb{E}(\operatorname{ReLU}(-\mathbf{x}^{\top} \mathbf{w})-\operatorname{ReLU}(-\mathbf{x}^{\top} \mathbf{w}_*))^2 \\
    & =4 \mathbb{E}(\operatorname{ReLU}(\mathbf{x}^{\top} \mathbf{w})-\operatorname{ReLU}(\mathbf{x}^{\top} \mathbf{w}_*))^2.
    \end{aligned}
    $$
    The proof is completed.

\end{proof}

\end{document}